\documentclass[12pt,letterpaper]{article}

\usepackage{amsmath,amssymb,amsfonts,amsthm,bbm}
\usepackage{mathtools}   % for \norm

\usepackage{enumitem}   % for customized enumerate

\usepackage{authblk} % for multiple author

\usepackage[titletoc,title]{appendix}  % Appendix A

\usepackage{comment}

\usepackage{kpfonts}
\usepackage[T1]{fontenc}

\usepackage[table,xcdraw,dvipsnames]{xcolor}   % for \textcolor, \colorlet, ...
\usepackage{hyperref}
\newcommand\myshade{85}
\colorlet{mylinkcolor}{YellowOrange}
\colorlet{mycitecolor}{Aquamarine}
\colorlet{myurlcolor}{violet}

\hypersetup{
  linkcolor  = mylinkcolor!\myshade!black,
  citecolor  = mycitecolor!\myshade!black,
  urlcolor   = myurlcolor!\myshade!black,
  colorlinks = true,
}

\usepackage{caption}
\usepackage{subcaption}

\usepackage{arydshln}
\usepackage{multirow}

\usepackage{siunitx} % control number display
\usepackage[linesnumbered,ruled,vlined]{algorithm2e}

\SetCommentSty{mycommfont}

\SetKwInput{KwInput}{Input}                % Set the Input
\SetKwInput{KwOutput}{Output}              % set the Output

\usepackage{listings}             % Include the listings-package
\renewcommand{\hat}{\widehat}
\renewcommand{\tilde}{\widetilde}
\renewcommand{\bar}{\overline}

\newcommand{\bfm}[1]{\ensuremath{\boldsymbol{#1}}} % bm

\def\bzero{\bfm 0}
\def\bone{\bfm 1}
\def\bbone{\mathbbm{1}} % package bm

\def\bb{\bfm b}   \def\bB{\bfm B}

     \def\EE{\mathbb{E}}
    
   \def\bG{\bfm G}  
   \def\bH{\bfm H}  
   \def\bI{\bfm I}

     \def\RR{\mathbb{R}}

\def\bu{\bfm u}   \def\bU{\bfm U}  
     
\def\bw{\bfm w}   \def\bW{\bfm W}  
\def\bx{\bfm x}   \def\bX{\bfm X}  
   \def\bY{\bfm Y}  
\def\bz{\bfm z}   \def\bZ{\bfm Z}  

\def\calA{{\cal  A}}

\def\calD{{\cal  D}} 
 
\def\calF{{\cal  F}} 
\def\calG{{\cal  G}} 
\def\calH{{\cal  H}}

\def\calL{{\cal  L}} 
\def\calM{{\cal  M}} 
\def\calN{{\cal  N}} 
 
\def\calP{{\cal  P}} 
 
\def\calR{{\cal  R}} 
\def\calS{{\cal  S}} 
\def\calT{{\cal  T}}

\def\calX{{\cal  X}}

\newcommand{\bfsym}[1]{\ensuremath{\boldsymbol{#1}}}

 \def\balpha{\bfsym \alpha}
 \def\bbeta{\bfsym \beta}
              
 \def\bdelta{\bfsym {\delta}}

 \def\bnu{\bfsym {\nu}}
 \def\btheta{\bfsym {\theta}}           \def\bTheta {\bfsym {\Theta}}
 \def\beps{\bfsym \varepsilon}          
 
              \def\bSigma{\bfsym \Sigma}
               
         \def\bLambda {\bfsym {\Lambda}}
           \def\bOmega {\bfsym {\Omega}}

 \def\bzeta{\bfsym {\zeta}}
           
\def\bphi{\bfsym {\phi}}

\providecommand{\abs}[1]{\left\lvert#1\right\rvert}
\providecommand{\norm}[1]{\left\lVert#1\right\rVert}

\providecommand{\paran}[1]{\left( #1 \right)}
\providecommand{\brackets}[1]{\left[ #1 \right]}
\providecommand{\braces}[1]{\left\{ #1 \right\}}

\DeclarePairedDelimiter\floor{\lfloor}{\rfloor}

\providecommand{\convprob}{\stackrel{\calP}{\longrightarrow}}
\providecommand{\convdist}{\stackrel{\calD}{\longrightarrow}}
\providecommand{\defeq}{\triangleq}

\usepackage{mathtools}
\DeclarePairedDelimiterX{\infdivx}[2]{(}{)}{%
  #1 \; \delimsize\| \; #2%
}

\DeclareMathOperator{\diag}{diag}
\newcommand{\E}[1]{{\mathbb{E}} \left[ #1 \right]}

\DeclareMathOperator{\Tr}{Tr}
\newcommand{\vect}[1]{{\textsc{vec}} \left( #1 \right)}
\newcommand{\mat}[1]{{\textsc{mat}} \left( #1 \right)}

\newtheorem{definition}{Definition}
\newtheorem{assumption}{Assumption}
\newtheorem{lemma}{Lemma}
\newtheorem{proposition}{Proposition}
\newtheorem{theorem}{Theorem}
\newtheorem{corollary}{Corollary}

\newtheorem{remark}{Remark}

\newcommand{\bigO}[1]{ \mathcal{O} \left( #1 \right) }

\newcommand{\smlo}[1]{{\rm o} \left( #1 \right) }

\newcommand{\Op}[1]{{\mathcal{O}_p} \left( #1 \right) }
\newcommand{\op}[1]{{\rm o_p} \left( #1 \right) }

\definecolor{royalpurple}{rgb}{0.47, 0.32, 0.66}

\def\eps{\varepsilon}
\def\lam {\lambda}

\def\bLam {\bLambda}
\def\bdiag{{\rm bdiag}}

\usepackage[framemethod=TikZ]{mdframed} % box environment & flexible box designs

\usepackage{fancybox}

\makeatletter
\def\widebreve{\mathpalette\wide@breve}
\def\wide@breve#1#2{\sbox\z@{$#1#2$}%
    \mathop{\vbox{\m@th\ialign{##\crcr
                \kern0.08em\brevefill#1{0.8\wd\z@}\crcr\noalign{\nointerlineskip}%
                $\hss#1#2\hss$\crcr}}}\limits}
\def\brevefill#1#2{$\m@th\sbox\tw@{$#1($}%
    \hss\resizebox{#2}{\wd\tw@}{\rotatebox[origin=c]{90}{\upshape(}}\hss$}
\makeatletter

\usepackage[top=1in, bottom=1in, left=1in, right=1in]{geometry}

\usepackage{setspace}
\usepackage{sectsty}
\allsectionsfont{\singlespacing}

\usepackage[compact]{titlesec}
\titlespacing{\section}{0pt}{*0}{*0}
\titlespacing{\subsection}{0pt}{*0}{*0}
\titlespacing{\subsubsection}{0pt}{*0}{*0}

\usepackage{caption}
\usepackage{etoolbox}
\BeforeBeginEnvironment{equation*}{\begin{singlespace}\vspace*{-\baselineskip}}
   \AfterEndEnvironment{equation*}{\end{singlespace}\noindent\ignorespaces}
\BeforeBeginEnvironment{equation}{\begin{singlespace}\vspace*{-\baselineskip}}
    \AfterEndEnvironment{equation}{\end{singlespace}\noindent\ignorespaces}

\BeforeBeginEnvironment{eqnarray}{\begin{singlespace}\vspace*{-\baselineskip}}
    \AfterEndEnvironment{eqnarray}{\end{singlespace}\noindent\ignorespaces}
\BeforeBeginEnvironment{eqnarray*}{\begin{singlespace}\vspace*{-\baselineskip}}
    \AfterEndEnvironment{eqnarray*}{\end{singlespace}\noindent\ignorespaces}

\BeforeBeginEnvironment{align}{\begin{singlespace}\vspace*{-\baselineskip}}
    \AfterEndEnvironment{align}{\end{singlespace}\noindent\ignorespaces}
\BeforeBeginEnvironment{align*}{\begin{singlespace}\vspace*{-\baselineskip}}
    \AfterEndEnvironment{align*}{\end{singlespace}\noindent\ignorespaces}

\BeforeBeginEnvironment{gather}{\begin{singlespace}\vspace*{-\baselineskip}}
    \AfterEndEnvironment{gather}{\end{singlespace}\noindent\ignorespaces}
\BeforeBeginEnvironment{gather*}{\begin{singlespace}\vspace*{-\baselineskip}}
    \AfterEndEnvironment{gather*}{\end{singlespace}\noindent\ignorespaces}

\usepackage[authoryear]{natbib}  %[numbers]
\newcommand{\mybibsty}{chicago}%{hsiamplain}%{agsm}
\newcommand{\mybib}{main}

\usepackage{graphicx}
\graphicspath{{figs/}}

\def\PIFULL{Auto-Clustered Policy Iteration}
\def\PI{ACPI}
\def\PEFULL{Auto-Clustered Policy Evaluation}
\def\PE{ACPE}
\def\MODELFULL{$K$-Heterogeneous Markov Decision Process}
\def\MODEL{$K$-Hetero MDP}

\begin{document}
%-------
%
%  title
%
%%
%% title page
%%

\newcommand{\blind}{0}

\if0\blind
{
	\title{\bf Reinforcement Learning with Heterogeneous Data: Estimation and Inference}
	\author[1]{ Elynn Y. Chen} % \thanks{Email: elynn.chen@stern.nyu.edu}} % {Alphabetical}
	\author[2]{Rui Song} % \thanks{Supported in part by NSF grant DMS-1555244, 2113637 and NCI grant P01 CA142538. Email: rsong@ncsu.edu}}
    \author[3]{ Michael I. Jordan} % \thanks{Email: jordan@cs.berkeley.edu}}
	\affil[1]{Stern School of Business, New York University}
    \affil[2]{Department of Statistics, North Carolina State University}
    \affil[3]{Department of Computer Science, University of California, Berkeley}
	%\affil[2]{...}
	%\affil[3]{...}
	%\date{\vspace{-5ex}}
	%\date{\today}
    \date{January 18, 2022}
	\maketitle
} \fi

\if1\blind
{
	\bigskip
	\bigskip
	\bigskip
	\title{\bf ...}
	\date{\vspace{-5ex}}
	\maketitle
	\medskip
} \fi

\begin{abstract}
\begin{singlespace}
    Reinforcement Learning (RL) has the promise of providing data-driven support for decision-making in a wide range of problems in healthcare, education, business, and other domains.
    Classical RL methods focus on the mean of the total return and, thus, may provide misleading results in the setting of the heterogeneous populations that commonly underlie large-scale datasets. We introduce the {\MODELFULL} ({\MODEL}) to address sequential decision problems with population heterogeneity.
    We propose the {\PEFULL} ({\PE}) for estimating the value of a given policy, and the {\PIFULL} ({\PI}) for estimating the optimal policy in a given policy class.
    Our auto-clustered algorithms can automatically detect and identify homogeneous sub-populations, while estimating the $Q$ function and the optimal policy for each sub-population.
    We establish convergence rates and construct confidence intervals for the estimators obtained by the {\PE} and {\PI}.
    We present simulations to support our theoretical findings, and we conduct an empirical study on the standard MIMIC-III dataset. 
    The latter analysis show evidence of value heterogeneity and confirms the advantages of our new method.
\end{singlespace}
\end{abstract}

%-------
%
%  main text
%
%\begin{singlespace}
%\tableofcontents
%\end{singlespace}

%!TEX root = 0-main.tex

\section{Introduction}  \label{sec:intro}

Many real world problems involve making decisions sequentially based on data that is influenced by previous decisions.
For example, in clinical practice, physicians make a series of treatment decisions over the course of a patient's disease based on his/her baseline and evolving characteristics \citep{schulte2014q}.
In education, human and automated tutors attempt to choose pedagogical activities that will maximize a student's learning, informed by their estimates of the student's current knowledge \citep{rafferty2016faster,reddy2017accelerating}.
The framework of reinforcement learning (RL), powered by large-scale datasets, has promise to provide valuable decision-making support in a wide range of domains, from healthcare, education, and business to scientific research~\citep{schulte2014q,mandel2017better,zhou2017optimizing}.

Classical RL methods estimate the \emph{value function}, which is defined as the expectation of accumulated future rewards conditioned on the state, or a state-action pair, at the time when a decision is made. An \emph{optimal policy} is obtained by maximizing the value function at each decision point.
Such mean-value-based RL methods have been the subject of a large literature (See \cite{sutton2018reinforcement} and references therein for a complete review).
However, naive applications of the mean-value-based RL methods to large-scale datasets may generate misleading results because real-world datasets are often created via aggregating many heterogeneous data sources.
Each data source may contain a different sub-population with different associations between state variable, action and outcomes (such as state transition, immediate rewards, and return).
Consider, for example, the MIMIC-III dataset \citep{johnson2016mimic,komorowski2018artificial}.
In Figure \ref{fig:hetero-mimic-III} (a), we plot the results of learning a single linear $Q$ function (i.e., a conditional expectation of future return on state-action pairs) with the entire sepsis cohort data.  The plot shows the estimation residual.
It is obvious that the dataset consists of two sub-populations with different mean values.
The mean-value-based $Q$ learning procedure fails to capture the heterogeneity in the dataset.
In general, results based on integrating over the entire super-population can be misleading, sometimes even dangerous in high-risk applications such as health care, finance and autonomous driving.
\begin{figure}[ht!]
    \centering
    \begin{subfigure}[b]{0.45\textwidth}
        \includegraphics[width=1\textwidth]{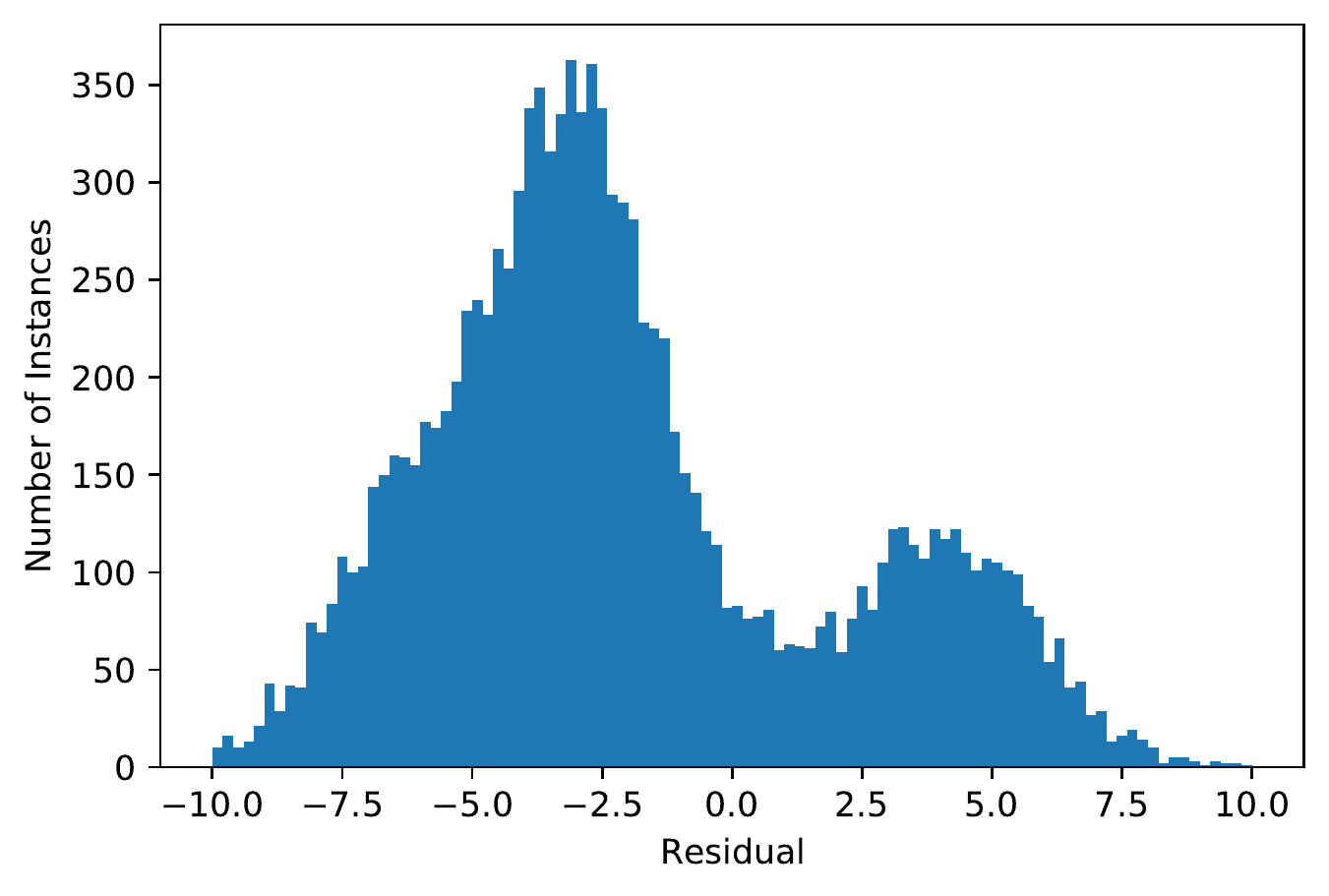}
        \caption{Residual distribution of simple linear regression is clearly bi-modal, with two peaks at around -4 and 4.}
    \end{subfigure}
    \begin{subfigure}[b]{0.45\textwidth}
        \includegraphics[width=1\textwidth]{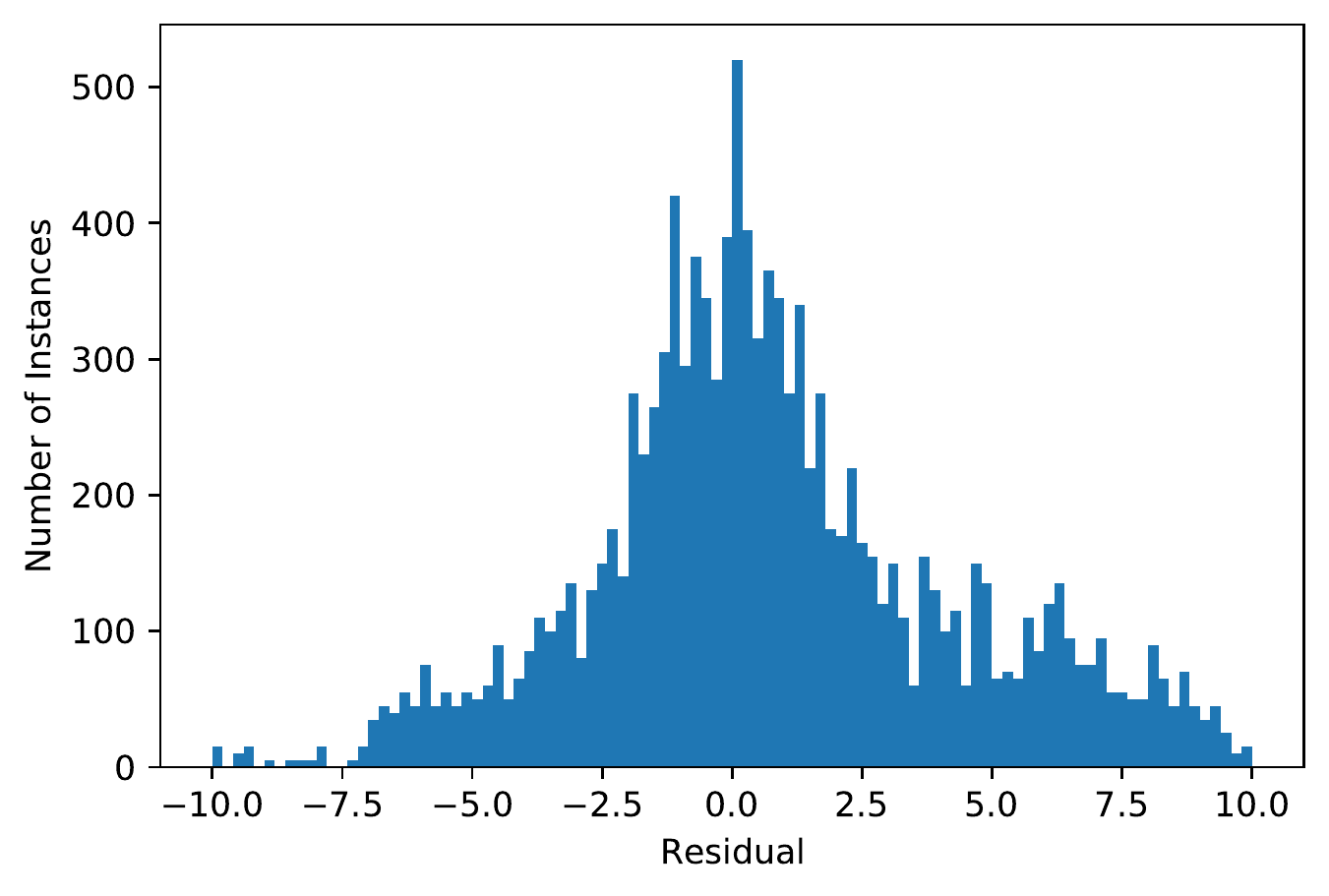}
        \caption{The residual distribution for our algorithm is bell-shaped, with one peak at around zero. It also performs better in terms of prediction error.}
    \end{subfigure}
    \caption{Heterogeneity in the MIMIC-III dataset and the improvement achieved by {\PI}.}
    \label{fig:hetero-mimic-III}
\end{figure}

In this paper, we introduce a framework that we refer to as \emph{$K$-heterogeneous MDP} ($K$-hetero MDP) to model dynamic decision making problems with heterogeneous populations.
We propose an algorithm, {\PIFULL} ({\PI}), that can automatically detect and identify homogeneous sub-populations from a dataset while learning the $Q$ function and the optimal policy for each sub-population.
We establish statistical guarantees for the estimators produced by the {\PI}.
Specifically, we obtain convergence rates and we construct confidence intervals (CIs) for the estimated model parameters and value functions.
These theoretical results are especially important in high-risk applications.
We apply the proposed method on the standard MIMIC-III dataset.
Figure \ref{fig:hetero-mimic-III} (b) presents the residuals after applying the {\PI} to the same MIMIC-III dataset.
It is clear that the {\PI} captures the heterogeneity in the population.
We present additional results in the experiments section that further confirm the favorable performance of our new method.

%\subsection{Contribution}
The contributions of our work are summarized as follows.
Methodologically, we introduce a formal definition of $K$-hetero MDP and propose the new {\PE} and {\PI} algorithms to deal with this heterogeneity automatically.
The efficacy of the proposed model and method is validated through empirical evidences from a synthetic and a real large-scale dataset.
While this paper focuses on batch-mode estimation, the framework of $K$-hetero MDP is generally applicable and the {\PE} and {\PI} can be extended to deal with online estimation.
Theoretically, we establish convergence rates and construct confidence intervals (CIs) under a bidirectional asymptotic framework that allows either $N$ or $T$ to go to infinity.
Our CIs are valid as long as either the number of trajectories $N$ or the length of trajectory $T$ diverge to infinity.
Our approach covers a wide variety of real applications and is especially useful when the number of trajectories in a sub-population is small.
We also offer theoretical guidance to practitioners on the choice of tuning parameters in {\PI}.
Technically, in order to study the properties of the estimators by the {\PE} and {\PI}, we study the temporal difference error and derive its convergence in the infinity norm.
This result is of independent importance, especially in studying RL algorithms with penalization \citep{song2015penalized}.

\subsection{Related work}

This paper is closely related to three different strands of research  in machine learning and statistics, namely distributional RL, dynamic treatment regime and mixture models.

\textbf{Distributional RL.}
In contrast to classical mean-value-based RL, \cite{bellemare2017distributional} introduce {\em distributional RL (DRL)} that focuses on the full distribution of the random return $Y_t$.
\cite{rowland2019statistics} present an example of multimodality in value distributions.
The main approaches to DRL include learning discrete {\em categorical} distribution \citep{bellemare2017distributional,rowland2018analysis,barth2018distributed}, learning distribution {\em quantiles} \citep{dabney2018distributional,yang2019fully}, and learning distribution {\em expectiles} \citep{rowland2019statistics}.
This work differs from ours in that it mainly focuses on computational issues and empirical performance.
No statistical properties of the estimators are studied, partly because of the fundamental difficulty of estimating a full distribution.

A useful concept introduced in the  DRL literature is the \emph{Bellman closedness} of a statistic.
Lemma 4.2 and Theorem 4.3 in \cite{rowland2019statistics} show that collections of moments are effectively the only finite set of statistics that are Bellman closed.
This inspires us to estimate the sub-population means instead of a full distribution of the super-population: the means are Bellman closed but the discretized distribution or the quantile is not.
In the middle of the spectrum from the mean-value-based RL to the distributional RL, our method is shown to capture heterogeneity while avoiding the sample inefficiency inherent in estimating a full distribution.

\textbf{Dynamic treatment regime.}
Researchers in dynamic treatment regime (DTR) have used the RL framework and MDP to derive a set of sequential treatment decision rules, or treatment regimes, to optimize patients' clinical outcomes over a fixed period of time.
See \cite{murphy2003optimal,robins2004optimal,schulte2014q,luedtke2016statistical,ertefaie2018constructing,zhu2019proper,shi2020statistical} and references therein.
Most methods in DTR are designed for finite horizon $T$ and are implemented through a backward recursive fitting procedure that is related to dynamic programming \citep{schulte2014q}.
For infinite horizon $T$, \cite{ertefaie2018constructing} propose a variant of greedy gradient $Q$ learning (GGQ) to estimate optimal dynamic treatment regimes.
This requires modeling a non-smooth function of the data, which creates complications.
\cite{luckett2019estimating} propose a $V$-learning algorithm that models the state value function and directly searches for the optimal policy among a pre-specified class of policies.
They show that the estimators of $V$-learning converge at the rate of the order $N^{-1/2}$, where $N$ is the number of decision trajectories in a dataset.
Although their setting allows for infinite horizon, the information accumulated over horizon $T$ does not enter explicitly in their results.
In contrast, we show that the convergence rate of our estimators are of the order $\paran{NT}^{-1/2}$ under similar assumptions.
\cite{shi2020statistical} obtain a convergence rate of the order $\paran{NT}^{-1/2}$ and derive confidence intervals for policy evaluation. But they do not consider finding the optimal policy or data heterogeneity.

While the personalized treatment regime in DTR aims to offer individualized treatment to each patient $i$ according to different state variable $X_{i,t}$,
the model coefficients are the same for all the patients.
In contrast, the heterogeneity in this paper stems from different model coefficients.
Our problem is more challenging since we need to estimate the model coefficients and, at the same time, to group them into clusters.

\textbf{Mixture models.}
In a linear regression setting, researchers have studied the problem of supervised clustering that explores homogeneous effects of covariates along coordinates and across samples.
The mixture-model-based approach needs to specify an underlying distribution for data, and also requires specification of the number of mixture components in the population.
An alternative approach to model-based clustering  employs grouping penalization on pairwise differences of the data points.
For example,  assuming parameter homogeneity over individuals, \cite{tibshirani2005sparsity}, \cite{bondell2008simultaneous}, \cite{shen2010grouping} and \cite{ke2013homogeneity} employ different penalty functions for each pair of coordinates of the coefficient vector to group similar-effect covariates along coordinates.
Assuming parameter heterogeneity across individuals, \cite{hocking2011clusterpath}, \cite{pan2013cluster}, and \cite{ma2017concave} adopt a fusion-type penalty with either an $L_p$-shrinkage or a non-convex penalty function to formulate clusters  across samples with the same regression coefficients.

This model-based approach is difficult to deploy in the RL setting because the distribution of the value function is generally not likely to be characterized by a Gaussian or mixture of Gaussian distribution \citep{sobel1982variance}.
Our {\PE} and {\PI} algorithms therefore are developed based on pairwise-fusion-type penalties.
However, our problem in RL is more challenging than that in linear regression.
First, while it is common to assume the noise in regression is i.i.d.\ sub-Gaussian, the observations over $T$ in RL are not uncorrelated. This raises significant challenges for our analysis.
Second, we consider the semi-parametric approximation with a diverging number of basis functions instead of the linear models.
Lastly, our ultimate goal is to learn an optimal policy, while the mixture regression framework focuses entirely on parameter estimation.

\subsection{Notation and organization}

We use $x$, $\bx$, $\bX$ to represent scalars, vectors and matrices, respectively.
Capital letter $X$ represents a random variable.
For any matrix $\bX$,  we use $\bx_{i \cdot}$, $\bx_{\cdot j}$, and $x_{ij}$ to refer to its $i$-th row, $j$-th column, and $ij$-th entry, respectively.
All vectors are column vectors and row vectors are written as $\bx^\top$ for any vector $\bx$.
%We write $\sigma_i$ be the $i$-th largest square root of eigenvalues of $\bX^\top \bX$, which are called the singular values of $\bX$.
%Then $\sigma_1$ is the largest in absolute value among $\sigma_i$.
We denote the matrix $\ell_2$-norm as $\norm{\bX}_2$ and max norm as $\norm{\bX}_{\max} \defeq \underset{i,j}{\max} \abs{x_{ij}}$.
When $\bX$ is a square matrix, we denote by $\Tr \paran{\bX}$, $\lambda_{max} \paran{\bX}$, and $\lambda_{min} \paran{\bX}$ the trace, maximum and minimum eigenvalues of $\bX$, respectively.
We denote $\mat{\bx}$ as the matricization of a vector $\bx$ and $\vect{\bX}$ as the vectorization of a matrix $\bX$.
%Notation $\bbone$ represents a vector or a matrix of all ones of proper size.
%We write by $a \asymp b$ when $a = \bigO{b}$ and $b = \bigO{a}$.
We let $C, c, C_0, c_0, \ldots$ denote generic constants whose actual values may vary.
%As a convention, we denote the true, oracle and estimated values as $\mathring\beta$, $\tilde\beta$, and $\hat\beta$, respectively.

The rest of this paper is organized as follows.
Section \ref{sec:model} introduces the problem setting and a few important concepts.
Section \ref{sec:estimation} proposes the {\PI}.
Section \ref{sec:computation} describes the computation procedure.
Section \ref{sec:theory} establishes the statistical properties of the estimators.
Section \ref{sec:simul} and Section \ref{sec:appl}  present empirical results with synthetic and real datasets.
Section \ref{sec:summ} concludes.
All proofs and technical lemmas are relegated to the supplementary material.

%!TEX root = 0-main.tex

\section{Statistical Framework} \label{sec:model}

\subsection{Markov decision process}

The problem of dynamic decision making can be  formalized mathematically as a \emph{Markov Decision Process} (MDP), $\calM = \braces{\calX, \calA, \Pr, r, \gamma}$.
Let $t \in [T]$ index the times in the decision process at which an action is taken, and let $i \in [N]$ index the set of individuals interacting with the environment.
The $i$-th individual may correspond to the $i$-th round of game, the $i$-th student, and the $i$-th patient in the setting of game playing, education, and health care, respectively.
At each time point $t$, the current status of the $i$-th individual, $i \in [N]$, is characterized by covariates $X_{i,t} \in \calX$ and a {\em finite} set of all possible actions $\calA$ is available.
An action $A_{i,t+1} \in \calA$ is selected and the state of the $i$-th individual changes to $X_{i,t+1}$ according to the transition probability $\Pr(X_{i,t+1} \mid X_{i,t}, A_{i,t})$.
An immediate random scalar reward $R_{i,t}$ is observed.
Over time, these comprise a sample trajectory for the $i$-th individual and we have in total $N$ sample trajectories, $\braces{\paran{X_{i,t}, A_{i,t}, R_{i,t}}}_{0\le t \le T}$ for $i\in[N]$ .

We assume that the data-generating model is a \textit{Time-Homogeneous Markov process (TH Markov)} and we also assume that the observed rewards satisfy a \textit{Conditional Mean Independence (CMI)} assumption.
\begin{assumption}[Data generating MDP.]   \label{assum:generating-MDP}
    The sample trajectories $\braces{\braces{\paran{X_{i,t}, A_{i,t}, R_{i,t}}}_{t\ge 0}}_{i\in[N]}$ are generated from (possibly $N$ different) Markov decision processes satisfying
    \begin{enumerate}[label=(\alph*)]
        \item (Markovianity)  For any $i\in[N]$ and $t\in [T]$, it holds that $X_{i,t}$ is independent of $\braces{\paran{X_{i,s}, A_{i,s}, Y_{i,s}}}$, for any $0\le s\le t-1$, conditional on $\paran{X_{i,t-1}, A_{i,t-1}}$.
        \item (Time-Homogeneity) The conditional density $\Pr(X_{i,t+1} \mid X_{i,t}, A_{i,t})$ is the same over $t$ for any $i\in[N]$.
         \item (Conditional Mean Independence (CMI)) For any $i\in[N]$ and $t\in[T]$, it holds that
        \begin{equation*}
                \E{R_{i,t} | X_{i,t} = \bx, A_{i,t}=a, \braces{X_{i,s}, A_{i,s}, R_{i,s}}_{0\le s < t}}
               = \E{R_{i,t} | X_{i,t} = \bx, A_{i,t}=a} = r_i\paran{\bx, a},
        \end{equation*}
        for some bounded reward function $r_i\paran{\bx,a}$.
    \end{enumerate}
\end{assumption}

\begin{remark}
    The assumptions of the TH Markov and CMI are common in the literature.
    However, Assumption \ref{assum:generating-MDP} (c) is different
    because $r_i\paran{\bx,a}$ can be different across $N$ trajectories.
\end{remark}

\begin{remark}
    Assumption \ref{assum:generating-MDP} implies that $R_{i,t} = r_i\paran{X_{i,t}, A_{i,t}} + \eta_{i,t}$
    where $\E{\eta_{i,t} \mid X_{i,t}, A_{i,t}} = 0$.
    The $\ell_2$ convergence and the asymptotic normality of the oracle estimators do not need any additional assumption on the noise term $\eta_{i,t}$.
    However, the uniform convergence that is needed for the feasible estimators requires $\eta_{i,t}$ to be a sub-Gaussian random variable.
\end{remark}

%\subsection{Policy and value functions}
A policy $\pi\paran{\bx} = \brackets{ \pi(1 \mid \bx), \;\cdots\;, \pi(M \mid \bx) }$ is a function that maps the state space $\calX$ to probability mass functions on the action space $\calA := \{1, \cdots, M\}$.
Under a policy $\pi$, an agent chooses action $A_{i,t}$ at state $X_{i,t}$ with probability $\pi\paran{A_{i,t} \mid X_{i,t}}$ and receives an immediate reward $R_{i,t}$.
The overall return at step $t$ is the accumulated discounted reward defined as
\begin{equation}  \label{eqn:return}
    Y_{i,t}= \sum_{s=t}^T \gamma^{t-s} R_{i,s},
\end{equation}
where the discount rate $\gamma \in [0,1)$ reflects a trade-off between immediate and future rewards.
Classic RL measures the goodness of a policy $\pi$ by value functions defined as follows.
The {\em state-value function} is the conditional expectation of the return conditioned on the starting state $\bx$:
\begin{equation} \label{eqn:v-func-0}
    V\paran{\pi, \bx} = \E{ \sum_{t=0}^T \gamma^t R_{i,t}
        \;\bigg|\; X_{i,0} = \bx,\;
        A_{i,t} \sim \pi( \cdot | X_{i,t}),\;
        X_{i,t+1} \sim \Pr(\cdot | X_{i,t}, A_{i,t})},
\end{equation}
where the expectation is taken under the trajectory distribution generated by policy $\pi$.
The {\em action-value function} or {\em $Q$ function} is the conditional expectation of the return conditioned on the starting state-action pair $(\bx, a)$:
\begin{equation} \label{eqn:q-func-0}
Q\paran{\pi, \bx, a} = \E{ \sum_{t=0}^T \gamma^t R_{i,t}
    \;\bigg|\; X_{i,0} = \bx,\; A_{i,0} = a,\;
    X_{i,t+1} \sim \Pr(\cdot | X_{i,t}, A_{i,t}), \;
    A_{i,t+1} \sim \pi( \cdot | X_{i,t+1})},
\end{equation}
where the expectation is taken by assuming that the dynamic system follows the given policy $\pi$ afterwards.
Let $\calR$ be a reference distribution of the covariate space $\calX$. The {\em integrated state-value} function of a given policy $\pi$ is defined as
\begin{equation} \label{eqn:inte-value}
    V_\calR(\pi)
    = \int_{\bx\in\calX} V\paran{\pi, \bx} \calR\paran{d\bx}
    = \int \sum_{a\in\calA} Q\paran{\pi, \bx, a} \pi\paran{a\mid\bx} \calR\paran{d\bx}.
\end{equation}
The reference distribution can be viewed as a distribution of initial states or stationary states and can be estimated from the dataset.

For the purpose of finding the optimal policy, we focus on a finite action set $\calA= \braces{1, \cdots, M}$ and a parametric class of policies $\Pi$ defined by
\begin{equation} \label{eqn:policy-param}
    \begin{aligned}
        & \pi\paran{\bx; \balpha} := \brackets{ \pi\paran{1, \bx; \balpha}, \cdots,  \pi\paran{M, \bx; \balpha}}, \quad\text{where} \\
        & \pi\paran{j, \bx; \balpha} =
        \begin{cases}
            \frac{\exp(\bx^\top\balpha_j)}{1 + \sum_{j=1}^{M-1} \exp(\bx^\top\balpha_j)}, \quad\text{for}\quad  k = 1, \cdots, M-1, \\
            \frac{1}{1 + \sum_{j=1}^{M-1} \exp(\bx^\top\balpha_M)}, \quad\text{otherwise,}
        \end{cases}
    \end{aligned}
\end{equation}
where we use $\balpha_j$ to denote a vector of parameters for the $j$-th action and $\balpha = \brackets{\balpha_1^\top, \cdots, \balpha_M^\top}^\top$.
The optimal policy $\pi^*_\calR\in\Pi$ satisfies $V_{\calR}\paran{\pi_\calR^*} \ge V_{\calR}\paran{\pi}$ for all $\pi \in \Pi$.
That is,
\begin{equation} \label{eqn:policy-opt}
    \pi_{\calR}^* = \underset{\pi\in\Pi}{\arg\max} \; V_{\calR}\paran{\pi}.
\end{equation}
The goal is to estimate $\pi_{\calR}^*$ using data collected from $N$ trajectories, each of which is assumed to have length $T_i = T$ for simplicity of presentation, and where $T$ may be finite or diverge to infinity.
Throughout, we assume that this reference distribution is fixed and the same for each homogeneous sub-population.

\subsection{Heterogeneous MDP and Bellman consistent equation}

Since the learning objective is optimizing the total return, we define the heterogeneity in terms of the value functions---the expectation of return \eqref{eqn:return} in homogeneous sub-populations.
Different from the classic RL, we allow different values, and thus different optimal policies, across a heterogeneous super-population.

\begin{definition}[$K$-Hetero MDP] \label{def:hetero-MDP}
    A $K$-heterogeneous MDP is defined as a dynamic system with a latent variable $\omega \in [K]$
     such that
    \begin{enumerate}[label=(\roman*)]
        \item For $\forall k \in [K]$, $\Pr\paran{\omega=k} = P^{(k)}$ and it satisfies that $P^{(k)}\in (0,1)$ and $\sum_{k=1}^K P^{(k)} = 1$.
        \item The value functions for each sub-population are conditioned additionally on $w_i = k$:
        \begin{align*}
           V^{(k)}\paran{\pi, \bx} &:= \E{ \sum_{t=0}^T \gamma^t R_{i,t}
               \;\bigg|\; X_{i,0} = \bx,\;
               A_{i,t} \sim \pi( \cdot | X_{i,t}),\;
               X_{i,t+1} \sim \Pr(\cdot | X_{i,t}, A_{i,t}),
               w_i = k}, \\
           Q^{(k)}\paran{\pi, \bx, a} &:= \EE \left[ \sum_{t=0}^T \gamma^t R_{i,t}
               \;\bigg|\; X_{i,0} = \bx,\; A_{i,0} = a,\;
               X_{i,t+1} \sim \Pr(\cdot | X_{i,t}, A_{i,t}), \; \right. \\
               & \hspace{8em} \bigg.
               A_{i,t+1} \sim \pi( \cdot | X_{i,t+1}), w_i = k \bigg].
        \end{align*}
       \item For any $k \ne k'$, the value functions of a given policy $\pi$ are different, that is, $V^{(k)}\paran{\pi, \bx} \ne V^{(k')}\paran{\pi, \bx}$ and $Q^{(k)}\paran{\pi, \bx, a} \ne Q^{(k')}\paran{\pi, \bx, a}$ for any $\pi\in\Pi$, $\bx\in\calX$, and $a\in\calA$.
    \end{enumerate}
\end{definition}

\begin{remark}
    Heterogeneity may also be defined in terms of MDP tuple $\calM^{(k)}$,
    or in terms of trajectory distribution $\calH^{(k)}(\pi)$.
    However, Definition \ref{def:hetero-MDP} provides a clearer framework to work with for the purpose of estimating an optimal policy that maximizes the expected return.
    Since MDPs with different $\calM^{(k)}$ or $\calH^{(k)}(\pi)$ may have the same value functions, heterogeneity defined by $\calM^{(k)}$ or $\calH^{(k)}(\pi)$ unnecessarily complicates the problem.
\end{remark}

Now we allow for heterogeneity over $i\in[N]$ on  trajectories $\braces{ X_{i,t}, A_{i,t}, R_{i,t} }$, $t\in[T]$, $i\in[N]$.
By the $K$-Hetero MDP assumption in Definition \ref{def:hetero-MDP},  the latent variable $w_i$ partitions $N$ trajectories into $K$ groups: $\calG = \paran{\calG^{(1)}, \cdots, \calG^{(K)}}$.
Without knowing the true groups, it is safe for now to use different action-value function $Q^{i}\paran{\pi, \bx, a}$  for different trajectory $i\in[N]$.
We have $Q^{i} \paran{\pi, \bx, a} = Q^{(k)}\paran{\pi,\bx, a}$ for all $i\in\calG^{(k)}$.
Our goal is to estimate $K$ and identify sub-populations that respond differently to a policy $\pi$.
The following Bellman consistency equation is a first-moment condition  that we will use to construct an auto-clustered estimation of the action-value function.
\begin{lemma} \label{lemma:q-bellman-cons-2}
    Under the assumptions of TH Markov and CMI,
    we have the first-moment condition of the action-value function $Q^{i}$ for any $i\in[N]$,
    \begin{gather}
        \E{ R_{i,t} + \gamma \sum_{a'\in\calA} Q^{i}\paran{\pi,X_{i,t+1}, a'} \pi\paran{a'|X_{i,t+1}} - Q^{i}\paran{\pi,X_{i,t}, A_{i,t}} \bigg | X_{i,t}, A_{i,t}} = 0,  \label{eqn:q-bellman-cons-2}  \\
        \E{ \paran{R_{i,t} + \gamma \sum_{a'\in\calA} Q^{i}\paran{\pi, X_{i,t+1}, a'} \pi(a'|X_{i,t+1}) - Q^{i}\paran{\pi, X_{i,t}, A_{i,t}}} \psi\paran{X_{i,t}, A_{i,t}} } = 0,
    \end{gather}
   where $\psi: \calX \times \calA \mapsto \RR$ is any real function supported on $\calX \times \calA$.
\end{lemma}

\section{Estimation} \label{sec:estimation}

The Auto-Clustered Policy Evaluation (ACPE)
and
the Auto-Clustered Policy Iteration (ACPI) are summarized in Algorithm \ref{alg:ACPE} and \ref{alg:ACPI}, respectively. 
Algorithm \ref{alg:ACPE} deals with value evaluation of a given policy on the $K$-Hetero MDP. 
Algorithm \ref{alg:ACPI} obtains optimal policies for different groups by using Algorithm \ref{alg:ACPE} and group-wise policy improvement.

\subsection{Auto-Clustered Policy Evaluation} \label{sec:heteroPE}

Let $Q\paran{\bx, a; \bbeta^{\pi,i}}$ denote a model for $Q^{i}(\pi,\bx, a)$ indexed by $\bbeta^{\pi,i}= \brackets{(\bbeta_{1}^{\pi,i})^\top, \cdots, (\bbeta_{M}^{\pi,i})^\top}^\top$, $i \in [N]$,
where $\bbeta_{m}^{\pi,i} \in \RR^J$ is the coefficient for the $m$-th action in the action set $\calA$.
We assume that the map $\bbeta^{\pi,i} \mapsto Q\paran{\bx, a; \bbeta^{\pi,i}}$ is differentiable everywhere for each fixed $\bx$, $a$ and $\pi$.
We stack all individual coefficients in a long vector $\bbeta^\pi = \brackets{ (\bbeta^{\pi,1})^{\top} \cdots (\bbeta^{\pi,N})^{\top}}^\top$  of length $JMN$ and define a trajectory indicator
$\bLam^i = [ \bzero  \cdots \bzero \; \bI \; \bzero \cdots \bzero ]$
such that $\bbeta^{\pi,i} = \bLam^i \bbeta^\pi$.
Let $\nabla\, Q\paran{\bx, a; \bLam^i \bbeta^\pi}$ denote the gradient of $Q\paran{\bx, a; \bLam^i \bbeta^\pi}$ with respect to $\bbeta^\pi$.
With observed trajectories $\braces{ X_{i,t}, A_{i,t}, R_{i,t} }$,
we define a sample-version quantity:
\begin{align}   \label{eqn:G}
    \bG\paran{\pi, \bbeta^\pi}
    & = \frac{1}{NTJ} \sum_{i=1}^{N} \sum_{t=1}^{T} \bigg( R_{i,t}
    + \gamma \sum_{a\in\calA} Q\paran{X_{i,t}, a; \bLam^i\bbeta^\pi } \pi\paran{a|X_{i,t+1}} \nonumber \\
    & \hspace{4em} - Q\paran{X_{i,t}, A_{i,t}; \bLam^i\bbeta^\pi } \bigg) \cdot \nabla\, Q\paran{X_{i,t}, A_{i,t}; \bLam^i \bbeta^\pi}.
\end{align}

By Lemma \ref{lemma:q-bellman-cons-2}, under certain mild conditions, there exist some $\mathring\bbeta^\pi$ that satisfy $\E{\bG^\pi\paran{\mathring\bbeta^\pi} } = \bzero$.
Thus, an estimator of $\mathring\bbeta^\pi$ can be obtained by minimizing $\bG\paran{\pi, \bbeta^\pi}^\top\bG^\pi\paran{\pi, \bbeta^\pi}$.
At the same time, the true model coefficients should satisfy $\mathring\bbeta^{\pi,i}= \mathring\bbeta^{\pi,j}$ if $i,j \in \calG^{(k)}$, for any $k \in [K]$, under the $K$-hetero MDP assumption.
When the grouping $\calG^{(k)}$ is unknown, it is crucial and beneficial to encourage grouping individuals with the same model coefficients together to achieve efficient estimation.
Given a penalty function $p:\RR \mapsto \RR$, the {\PE} algorithm estimates $\mathring\bbeta^\pi$ for a given policy $\pi$ as:
\begin{equation} \label{eqn:Q-ls-beta}
    \hat\bbeta^\pi
    = \underset{\bbeta=[(\bbeta^1)^\top\cdots(\bbeta^N)^\top]^\top}{\arg\min} \;
    \brackets{ \bG\paran{\pi,\bbeta}^\top \bG\paran{\pi,\bbeta} + \frac{1}{N^2}\sum_{1\le i<j \le N} p\paran{(JM)^{-1/2}\norm{\bbeta^i - \bbeta^j}_2, \lam } }.
\end{equation}
The choices of the approximation model for $Q$ function and the penalty function $p$  will be presented in the sequel.
The fusion-type penalty shrinks some of the pairs $\hat\bbeta^{\pi,i}-\hat\bbeta^{\pi,j}$ to zero, based on which we can partition the sample into subgroups simultaneously when we estimate the coefficients.
The choice of tuning parameter $\lam$ is given in Theorem \ref{thm:feasible-in-prob}.

The {\PE} requires a class of models for the $Q$ function indexed by parameter $\bbeta^\pi$.
We use a basis function approximation.
Let $\bphi(\cdot) = \brackets{\phi_1(\cdot),\cdots,\phi_J(\cdot)}^\top$ be a vector of pre-specified basis functions such as Gaussian basis function or splines.
We allow $J$ to grow with the sample size to reduce the bias of the resulting estimates.
Given a policy $\pi$, the $Q$ function becomes
\begin{equation}  \label{eqn:Q-hetero-glm}
    Q\paran{X_{i,t}, A_{i,t}; \bLam^i \bbeta^\pi} \approx   Z_{i,t}^\top \bLam^i \bbeta^\pi,
\end{equation}
where $Z_{i,t} := \bz\paran{X_{i,t}, A_{i,t}}$ and
\begin{equation} \label{eqn:z}
     \bz\paran{\bx, a}
    := \brackets{ \bphi(\bx)^\top \bbone(a=1), \cdots, \bphi(\bx)^\top \bbone(a=M) }^\top.
\end{equation}
Plugging \eqref{eqn:Q-hetero-glm} in \eqref{eqn:G}, we have
\begin{equation}   \label{eqn:G-sieve}
    \bG\paran{\pi,\bbeta^\pi}
    = \frac{1}{NTJ} \sum_{i=1}^{N} \sum_{t=1}^{T} (\bLam^i)^\top Z_{i,t} \paran{ R_{i,t} - \paran{  Z_{i,t} - \gamma U^\pi_{i,t+1} }^\top \bLam^i  \bbeta^\pi },
\end{equation}
where $U^\pi_{i,t+1} := \bu\paran{\pi, X_{i,t+1}}$ and
\begin{equation}  \label{eqn:u}
    \bu\paran{\pi, \bx}
    := \brackets{ \bphi(\bx)^\top \pi(1|\bx), \cdots, \bphi(\bx)^\top  \pi(M|\bx) }^\top.
\end{equation}
% By Lemma \ref{lemma:q-bellman-cons-2}, under certain mild conditions, there exist some $\mathring\bbeta  $ that satisfy $\E{\bG^\pi\paran{\mathring\bbeta} } = \bzero$.

The {\PE } also requires a penalty function $p\paran{\cdot}$.
We consider concave penalties satisfying Assumption \ref{assum:penalty} that can produce unbiased estimates.
Examples includes SCAD \citep{fan2001variable} and MCP \citep{zhang2010nearly}, which are defined respectively as
\begin{align}
    p_{\eta}^{MCP}(t,\lambda) & = \lambda \int_0^t \paran{1-x/(\eta\lambda)}_+ dx, \quad \eta > 1,  \label{eqn:pen-MCP} \\
    p_{\eta}^{SCAD}(t,\lambda) & = \lambda \int_0^t \min\braces{1, \paran{\gamma-x/\lambda}_+/(\gamma-1)}dx, \quad \eta > 2, \label{eqn:pen-SCAD}
\end{align}
where $\eta$ is a parameter that controls the concavity of the penalty function.
In particular, both penalties converge to the $L_1$ penalty as $\eta \rightarrow \infty$.

The procedure of ACPE is summarized in Algorithm \ref{alg:ACPE}.
Under a semi-parametric $Q$ function approximation, we obtain $\hat \bbeta^\pi  $ by solving \eqref{eqn:Q-ls-beta} with instantiated $\bG\paran{\pi,\bbeta}$ given in \eqref{eqn:G-sieve} and a concave penalty $p\paran{\cdot}$ satisfying Assumption \ref{assum:penalty}.
Then, the individual-wise value functions can be estimated by
\begin{equation} \label{eqn:Q-V-indiv-est}
    \hat Q^{i} \paran{\pi, \bx, a} = \bz\paran{\bx, a}^\top \hat \bbeta^{\pi,i}, \quad\text{and}\quad
    \hat V^{i} \paran{\pi, \bx, a} = \bu\paran{\pi, \bx}^\top \hat \bbeta^{\pi,i},  \quad\text{for}\quad i \in [N],
\end{equation}
where $\bz\paran{\bx, a}$ and $\bu\paran{\pi, \bx}$ are defined in \eqref{eqn:z} and \eqref{eqn:u}, respectively.

\begin{algorithm}[ht!]
    \DontPrintSemicolon
    \KwInput{Sample trajectories $\braces{X_{i,t}, A_{i,t}, R_{i,t}, X_{i, t+1}}$, $t\in[T], i\in[N]$, a policy $\pi\paran{a \mid \bx}$ to be evaluated, and a tuning parameter $\lam$. }
    \KwOutput{Individual $Q$ coefficients $\hat\bbeta^{\pi,i}$ for $i\in[N]$, number of groups $K$, a partition of $N$ trajectories, $\calG = \braces{\calG^{(1)}, \cdots, \calG^{(K)}}$ and group coefficients $\hat\btheta^{\pi,(k)}$, $k\in[K]$.}

    Estimate $\hat\bbeta^{\pi}$ by solving %\eqref{eqn:Q-ls-beta} 
    $$\hat\bbeta^\pi
    = \underset{\bbeta=[(\bbeta^1)^\top\cdots(\bbeta^N)^\top]^\top}{\arg\min} \;
    \brackets{ \bG\paran{\pi,\bbeta}^\top \bG\paran{\pi,\bbeta} + \frac{1}{N^2}\sum_{1\le i<j \le N} p\paran{(JM)^{-1/2}\norm{\bbeta^i - \bbeta^j}_2, \lam } }, $$
    with instantiated $\bG\paran{\pi,\bbeta}$ given in \eqref{eqn:G-sieve} and a concave penalty $p\paran{\cdot}$ satisfying Assumption \ref{assum:penalty}. Two candidates for $p\paran{\cdot}$ are given in  \eqref{eqn:pen-MCP} and \eqref{eqn:pen-SCAD}.
    
    Let $\mat{\hat\bbeta^{\pi} } = \brackets{\hat\bbeta^{\pi,1}, \cdots, \hat\bbeta^{\pi,N}}$ be the matricization of $\hat\bbeta^{\pi}$.

    Apply a chosen clustering algorithm to cluster $N$ trajectories to $K$ groups based on $\braces{\hat\bbeta^{\pi,i}}_{i\in[N]}$ and obtain the $N\times K$ group membership matrix $\hat \bW$.

   Calculate the group-wise coefficients by 
   $$\hat\btheta^{\pi,(k)} =  \mat{\hat\bbeta^\pi}\,\hat\bw_{\cdot k} / \hat N^{(k)}, $$
   where $\hat N^{(k)} = \hat\bw_{\cdot k}^\top\hat\bw_{\cdot k}$ estimates the number of trajectories in $\calG^{(k)}$.

    \caption{Auto-Clustered Policy Evaluation (ACPE)}
    \label{alg:ACPE}
\end{algorithm}

Given estimated $\hat \bbeta^{\pi,i}$, we can estimate the number of groups $K$ and the $N\times K$ group membership matrix
$\hat \bW = \brackets{w_{i,k}}_{1\le i\le N, 1\le k \le K}$ by using any chosen clustering algorithm.
The $k$-th group coefficient is estimated as
\begin{equation} \label{eqn:theta-hat}
    \hat\btheta^{\pi,(k)}
    = \mat{\hat\bbeta^\pi}\hat\bw_{\cdot k} / \hat N^{(k)},
\end{equation}
where $\hat N^{(k)} = \hat\bw_{\cdot k}^\top\hat\bw_{\cdot k}$ estimates the number of trajectories in $\calG^{(k)}$ and $\mat{\hat\bbeta^\pi} \in \RR^{MJ\times N}$ is the matricization of $\hat\bbeta^\pi$.
When $N$ or $T$ are sufficiently large, by the oracle property in Theorem \ref{thm:feasible-in-prob}, we have $\Pr\paran{ \hat\bbeta^{\pi,i} = \hat\bbeta^{\pi,j}} \rightarrow 1$ if $i$ and $j$ belong to the same group.
The group coefficients $\hat\btheta^{\pi, (k)}$, $1\le k \le K$, can be chosen as the distinct values of $\hat\bbeta^{\pi,i}$, $i \in [N]$.
Accordingly, we have $\Pr\paran{ \hat K = K} \rightarrow 1$.
The value functions for $\calG^{(k)}$ can be estimated by
\begin{equation} \label{eqn:Q-V-group-est}
    \hat Q^{(k)} \paran{\pi, \bx, a}
    = \bz\paran{\bx, a}^\top \hat\btheta^{\pi, (k)}, \quad\text{and}\quad
    \hat V^{(k)}\paran{\pi,\bx}
    = \bu\paran{\pi, \bx}^\top \hat\btheta^{\pi, (k)}, \quad\text{for}\quad 1\le k \le K,
\end{equation}
where $\bz\paran{\bx, a}$ and $\bu\paran{\pi, \bx}$ are defined in \eqref{eqn:z} and \eqref{eqn:u}, respectively.
Let $\calR$ be a reference distribution on the covariate space $\calX_k$ of the $k$-th group.
The integrated value function for group $k$ is estimated by
\begin{equation}  \label{eqn:V-group-int-est}
    \hat V^{(k)}_\calR \paran{\pi}
    = \int \hat V^{(k)}\paran{\pi, \bx} \calR \paran{d\bx}.
    % = \int \sum_{ a\in\calA} \hat Q^{\pi,(k)}\paran{\bx, a} \pi\paran{a \mid \bx; \balpha} \calR \paran{d\bx}.
\end{equation}

\subsection{Auto-Clustered Policy Improvement}

For the purpose of finding the optimal policy, we focus on the policy class defined in \eqref{eqn:policy-param}.
Because of the heterogeneity across different sub-populations, we expect their respect optimal polices are different.
For each group $k$, its policy is denoted as $\pi\paran{\balpha^{(k)}}$, where the group-wise policy is indexed by parameter $\balpha^{(k)}$, $k\in[K]$.

The idea of Auto-Clustered Policy Iteration (ACPI) is to use ACPE (Algorithm \ref{alg:ACPE}) to estimate the value of any policy, obtain a $K$-partition of $N$ trajectories, and maximize the integrated state-value function over the class of policies \eqref{eqn:policy-param} for each group $k\in [K]$.
%For each group $k$, group-wise policies $\pi\paran{\balpha^{(k)}}$ are improved to maximize integrated value \eqref{eqn:V-group-int-est}, according to line 7 - 9 in Algorithm \ref{alg:ACPI}.
The procedure is summarized in Algorithm \ref{alg:ACPI}.

\begin{algorithm}[ht!]
    \DontPrintSemicolon
    \KwInput{Sample trajectories $\braces{\paran{X_{i,t}, A_{i,t}, R_{i,t}, X_{i,t+1}}}$, $t\in[T], i\in[N]$, and tuning parameter $\lam$. }
    \KwOutput{A partition of $N$ trajectories into $K$ groups, and an optimal policy for each group $k\in [K]$ in parametric family \eqref{eqn:policy-param}, with parameter $\balpha^{(k)}$.}

    At step $s = 1$, we initialize the number of groups $K=1$ (the value of $K$ will be updated immediately in the ACPE step)
    and policy coefficients $\alpha_s^{{(k)}}$, $1\le k \le K$.

    \While{Not converged}
    {
        \For{each group $1 \le k \le K$}{
            Let $\pi_s^{(k)} = \pi\paran{ \alpha_s^{{(k)}} }$, we apply ACPE (Algorithm \ref{alg:ACPE}) to obtain $\hat\bbeta^{\pi_s^{(k)},1}, \cdots, \hat\bbeta^{\pi_s^{(k)},N}$.
        }

        Apply a chosen clustering algorithm to cluster $N$ trajectories to $K$ groups based on $\braces{\hat\bbeta^{\pi_s^{(k)}, i}}_{i\in[N]}$ and obtain the $N\times K$ group membership matrix $\hat \bW$.

        Calculate the group-wise coefficients
        $$
            \hat\btheta^{\pi_s^{(k)},(k)} =  \mat{\hat\bbeta^{\pi_s^{(k)}}}\,\hat\bw_{\cdot k} / \hat N^{(k)},
        $$
        where $N^{(k)} = \hat\bw_{\cdot k}^\top\hat\bw_{\cdot k}$ estimates the number of trajectories in $\calG^{(k)}$.

        \For{each group $1 \le k \le K$}{
            Let $\pi_s^{(k)} = \pi\paran{\balpha^{(k)}_s}$ be the fixed index, we update
            $$
                \balpha^{(k)}_{s+1} \leftarrow \arg\underset{\balpha}{\max}\; \hat V_{\calR}\paran{\balpha}
            $$
            where the integrated value function $\hat V_{\calR}\paran{\balpha}$ is defined as
            $$
            \hat V_{\calR}\paran{\balpha}
                = \int \bu\paran{\pi\paran{\balpha}, \bx}^\top\hat\btheta^{\pi_s^{(k)},(k)}  d\calR(\bx), 
            $$
            and $\pi\paran{\balpha}$ and $\bu\paran{\pi\paran{\balpha}, \bx}$ are defined in \eqref{eqn:u} and \eqref{eqn:policy-param}, respectively.

            %Set $\balpha_k^{s+1} = \balpha_k^s + \delta^s \nabla_{\balpha} \hat V_{\calR}\paran{\pi\paran{\balpha}}$ where $\nabla_{\balpha} \hat V_{\calR}\paran{\pi\paran{\balpha}}$.% is given in \eqref{eqn:v-diff-explicit}.
            %Update $$
        }
        $s \leftarrow s + 1$.
    }
    \caption{Auto-Clustered Policy Iteration (ACPI)}
    \label{alg:ACPI}
\end{algorithm}

%!TEX root = 0-main.tex

\section{Theory} \label{sec:theory}

In this section, we lay out the theoretical framework for ACPE and ACPI in a double-divergence structure, which allows either sample size $N$ or decision horizon $T$ to go to infinity.
We establish asymptotic properties of the offline estimation of the coefficients, value functions and optimal policy.
Let $\bW=\brackets{\bw_{1 \cdot} \cdots \bw_{N \cdot}}^\top \in \RR^{N\times K}$ be the group membership matrix where $w_{ik} = 1$ for $i\in \calG^{(k)}$ and $w_{ik}=0$ otherwise.
We first present theoretical results for the {\em oracle estimator} of ACPE when the true $\bW$ is {\em known a priori}.
Section \ref{sec:oracle} derives the oracle properties, including $\ell_2$ and $\ell_\infty$ convergence rates, when the true sub-population information $\braces{\calG^{(k)}, 1\le k \le K}$ is known.
Then, in Section \ref{sec:feasible}, we establish the $\ell_2$ convergence rate and the asymptotic normality of the parameters estimated from \eqref{eqn:Q-ls-beta} when $\bW$  is {\em not known}.
Lastly, in Section \ref{sec:theory-policy}, we present that the estimated optimal policies for each group $\calG_k$ convergences to their respective true optimal policies.
As a convention, we denote the true, oracle and estimated values as $\mathring\beta$, $\tilde\beta$, and $\hat\beta$, respectively.
The proofs of all theorems are provided in the supplementary materials.

\subsection{Properties of the oracle estimators} \label{sec:oracle}

We denote $\btheta^{\pi} = \brackets{ {\btheta^{\pi, (1)}}^{\top} \cdots {\btheta^{\pi,(K)}}^\top}^{\top}$ as the group coefficient matrix where $\btheta^{\pi,(k)}$ is the coefficient of the $Q$ function for $\calG^{(k)}$.
When the $\bW$ is known, the oracle estimator refers to the estimator that minimizes the objective function in \eqref{eqn:Q-ls-beta} with respect to $\btheta^{\pi}$ without penalty, that is
\begin{equation} \label{eqn:Q-beta-ora}
    \tilde\btheta^{\pi} = \underset{\btheta}{\arg\min} \; \tilde\bG\paran{\pi, \btheta}^\top \tilde\bG\paran{\pi, \btheta},
\end{equation}
where $\tilde\bG\paran{\pi, \btheta}
= \bG\paran{\pi,(\bW\otimes\bI_{JM})\btheta}$, or equivalently,
\begin{equation}
    \tilde\bG\paran{\pi, \btheta}
    = \frac{1}{NTJ} \sum_{i=1}^{N} \sum_{t=1}^{T} \paran{(\tilde\bLam^{i})^\top Z_{i,t} R_{i,t} - (\tilde\bLam^{i})^\top  \paran{Z_{i,t} Z_{i,t}^\top - \gamma Z_{i,t} (U_{i,t+1}^{\pi})^\top} \tilde\bLam^{i} \btheta},
\end{equation}
and $\tilde\bLam^{i} = \bLam^{i} (\bW \otimes \bI_{JM})$,
$\tilde\bbeta^\pi = (\bW\otimes\bI_{JM}) \tilde\btheta^\pi$,
$Z_{i,t}$ and $U^{\pi}_{i,t}$ are defined in \eqref{eqn:z} and \eqref{eqn:u}, respectively.

We first detail the assumptions that are necessary for the convergence and asymptotic normality of estimates obtained by the ACPE.

\begin{definition}[$\kappa$-Smooth functions]
    Let $f\paran{\cdot}$ be an arbitrary function on $\calX\in\RR^{p}$.
    For a $p$-tuple $\balpha = (\alpha_1, \cdots, \alpha_p)$ fo non-negative integers, let $D^{\alpha}$ denote the differential operator:
    \begin{equation*}
        D^{\balpha} f(\bx) = \frac{{\partial}^{\norm{\balpha}_1} f(\bx)}{\partial x_1^{\alpha_1}\cdots \partial x_p^{\alpha_p}},
    \end{equation*}
    where $\bx=(x_1, \cdots, x_p)^\top$.
    The class of $\kappa$-smooth functions is defined as
    \begin{equation*}
        \calH(\kappa, c) = \braces{ f: \underset{\norm{\balpha}_1\le \floor{\kappa}}{\sup} \; \underset{\bx\in\calX}{\sup} \; \abs{D^{\balpha} f(\bx)} \le c; \text{ and } \underset{\norm{\balpha}_1\le \floor{\kappa}}{\sup}\; \underset{\bx_1, \bx_2\in\calX, \bx_1 \ne \bx_2}{\sup} \; \frac{\abs{D^{\balpha} f(\bx_1) - D^{\balpha} f(\bx_2)}}{\norm{\bx_1-\bx_2}_2^{\kappa - \floor{\kappa}}}}.
    \end{equation*}
\end{definition}

\begin{assumption} \label{assum:r-q-func}
    There exists some $\kappa$, $c>0$ such that $r(\bx, a)$, $\Pr(\bx' | \bx, a)$ belong to the class of $\kappa$-smooth function of $\bx$ for any $a\in\calA$ and $\bx'\in\calX$.
\end{assumption}

Lemma 1 in \cite{shi2020statistical} shows that under Assumption \ref{assum:r-q-func}, there exists some constant $c' > 0$ such that $Q\paran{\pi, \bx, a}$ belongs to the class of $\kappa$-smooth function of $\bx$ for any policy $\pi$ and any action $a\in\calA$.
This implies that the $Q$ function has bounded derivatives up to order $\floor{\kappa}$.
To approximate the $Q$ function, we  restrict our attention to two particular type of Sieve basis functions. 
Let $BSpline(J, r)$ denote a tensor-product B-spline basis of dimension $J$ and of degree $r$ on $[0,1]^p$ and
$Wav(J, r)$ denote a tensor-product Wavelet basis of regularity $r$ and dimension $J$ on $[0,1]^p$.
The sieve $\bphi_J$ we use is either $BSpline(J, r)$ or $Wav(J, r)$ with $r > \max\paran{\kappa, 1}$.
This, together with Assumption \ref{assum:r-q-func}, implies that there exists a set of vectors $\{\mathring\btheta^{\pi}_{a}\}_{a\in\calA}$ that satisfies $\underset{\bx\in\calX, a\in\calA}{\sup}\abs{Q(\pi, \bx, a) - \Phi_J^\top(\bx) \mathring\btheta^{\pi}_{a}} = \bigO{J^{-\kappa/p}}$. 

%\begin{assumption} \label{assum:basis}
%    Let $BSpline(J, r)$ denote a tensor-product B-spline basis of dimension $J$ and of degree $r$ on $[0,1]^p$ and
%    $Wav(J, r)$ denote a tensor-product Wavelet basis of regularity $r$ and dimension $J$ on $[0,1]^p$.
%    The sieve $\bphi_J$ is either $BSpline(J, r)$ or $Wav(J, r)$ with $r > \max\paran{\kappa, 1}$.
%\end{assumption}

%\begin{remark}
%    \elynn{Remark on basis functions. Gaussian basis function offers the highest degree of flexibility.}
%\end{remark}

\begin{assumption} \label{assum:density}
    Let $\mu$ be the limiting density function of $X_{i,t}$ and $\nu_0$ be the initial state density. 
    The density function $\mu$ and $\nu_0$ are uniformly bounded away from zero and infinity on $\calX$.
\end{assumption}

\begin{assumption}  \label{assum:min-eigen-of-E-Sigma} %\label{assum:MC-T-bound}
    There exists some constant $C_1 > 0$ such that
    \begin{equation*}
        \underset{\pi\in\Pi}{\inf} \;
        \lam_{\min}\paran{T^{-1}\sum_{t=0}^{T-1} \E{ Z_{i,t}Z_{i,t}^\top - \gamma^2 \bar\bu(\pi, X_{i,t}, A_{i,t}) \bar\bu(\pi, X_{i,t}, A_{i,t})^\top} } \ge C_1,
    \end{equation*}
    where
    $\bar\bu(\pi, \bx, a) = \E{ \bu\paran{\pi, X_{i,t+1}} \,\mid\, X_{i,t}=\bx, A_{i,t}=a }$
    and
    $\lam_{\min}(\cdot)$ denotes the minimal eigenvalue of a matrix.
\end{assumption}

\begin{assumption}   \label{assum:geo-ergodic} %\label{assum:MC-T-infty}
Under the setting that $T\rightarrow\infty$, the Markov chain $\braces{X_{i,t}}_{t\ge 0}$ is geometrically ergodic, that is, there exists some function $f(\bx)$ on $\calX$ and some constant $c \le 1$ such that $\int_{\bx\in\calX} f(\bx)\mu(\bx) d\bx < + \infty$ and
\begin{equation*}
    \norm{P_t(\cdot | \bx) - \mu(\bx)}_{TV} \le f(\bx) \rho^t, \qquad \forall t \ge 0,
\end{equation*}
where $\mu(\cdot)$ is the limiting density function and $\norm{\cdot}_{TV}$ denotes the total variation norm.
\end{assumption}

Assumption \ref{assum:density} is weak, it does not require the limiting density function $\mu$ to be equal to the limit state density $\nu_0$. 
When $T$ is finite, Assumption \ref{assum:min-eigen-of-E-Sigma} guarantees that $\tilde\bSigma$ defined in \eqref{eqn:tilde-sigma} is invertible with probability approaching one when $N \rightarrow \infty$.
When $T$ is infinite, Assumption \ref{assum:geo-ergodic} (i) guarantees that the matrix $\E{\tilde\bSigma}$ is invertible and (ii) enables us to derive matrix concentration inequalities for $\tilde\bSigma$.
Together, they imply that $\tilde\bSigma$ is invertible with probability approaching one.
Assumption \ref{assum:min-eigen-of-E-Sigma} and \ref{assum:geo-ergodic} are needed to show the existence of a unique $\{ \mathring{\btheta}^\pi_a \}_{a\in\calA}$ uniformly over $\Pi$.
They are mild and can be verified empirically by checking that certain data-dependent matrices are invertible.
More discussion on similar assumptions can be found in \cite{luckett2019estimating} and \cite{shi2020statistical}.

Let $\mathring\btheta^{\pi} = \brackets{ \mathring{\btheta}^{\pi, (1) \top} \;\cdots\; \mathring{\btheta}^{\pi,(K) \top} }^{\top}$ where $\mathring\btheta^{\pi,(k)} = \brackets{\mathring\btheta_1^{\pi,(k) \top} \;\cdots\; \mathring\btheta_M^{\pi,(k) \top}}^\top$ for $k\in[K]$. 
Proposition \ref{thm:beta-l2-convg} and Theorem \ref{eqn:ora-asymp-normal} in the supplemental material establish the $\ell_2$ convergence and the asymptotic normality of the oracle estimator $\tilde\btheta^{\pi}$.
Corollary \ref{thm:int-val-asymp-normal}  in the supplemental material shows the the asymptotic normality of the integrated value of a given policy for each sub-population. 
They imply that the oracle group-wise estimator converges at a rate of $\paran{T N P_k}^{-1/2}$ for $\forall k\in [K]$.
The estimators in \cite{luckett2019estimating,jiang2016doubly,thomas2015high} typically converge at a rate of $\paran{N P_k}^{-1/2}$ and are not suitable for settings when one sub-population has only few trajectories.
Our estimation procedure also aggregates information along horizon $T$ and achieves same rates as in \cite{shi2020statistical} when group membership $\bW$ are known a priori.

To finally establish the large sample theory for the {\PE} when group membership are unknown, we need a stronger {\em uniform consistency} result for the oracle estimators when either $N \rightarrow \infty$ or $T \rightarrow \infty$.

\begin{theorem}  [Oracle estimator uniform convergence]   \label{eqn:ora-uniform-conv}
    Suppose Assumption \ref{assum:r-q-func} -- \ref{assum:geo-ergodic} hold.
    Let $N_{min} = \underset{1\le k \le K}{\min} N_k$ and $N_{max} = \underset{1\le k \le K}{\max} N_k$.
    If $K=\smlo{N_{\min} T}$,
   $J \ll \sqrt{N_{\min} T} / \log(N_{\min} T)$,
   $J^{-\kappa / p} \ll 1/\sqrt{N_{\max} T}$,
    %and $\sqrt{J \log(N_{\min}T)/N_{\min}T}. \rightarrow 0$,
    we have with probability at least $1 - 2JMK (N_{\min} T)^{-2} - \bigO{(N_{\min}T)^{-2}}$ that
    \begin{equation*}
       \underset{\pi \in \Pi}{\sup}\norm{ \tilde\btheta^{\pi} - \mathring\btheta^{\pi} }_{\infty}  \le \phi_{NT},
    \end{equation*}
    where
    \begin{equation} \label{eqn:psi-NT}
        \phi_{NT} = 6 c C^{-1} \frac{N_{\max}}{N_{\min}} \sqrt{2 J \frac{\log(N_{\min}T)}{N_{\min}T} }.
    \end{equation}
\end{theorem}

\begin{remark}
    By Definition \ref{def:hetero-MDP}, $N_{\max} / N_{\min} = P_{\max} / P_{\min}$ are bounded away from infinity.
    The bound \eqref{eqn:psi-NT} can be simplified to $\Op{\sqrt{ J \frac{\log(N T)}{N T} }}$.
\end{remark}

\subsection{Properties of the feasible estimator}    \label{sec:feasible}

We now study the theoretical properties of the feasible estimator when the true group membership $\bW$ is not known.
We introduce the following assumptions on the penalty function and the minimum signal difference between groups.
\begin{assumption} \label{assum:penalty}
    The penalty function $p\paran{x, \lam}$ is a symmetric function of $x$, and it is non-decreasing and concave in $x$ for $x\in [0, +\infty)$.
    Let $\rho(x) =  \lam^{-1} p\paran{x, \lam}$, there exists a constant $0 < c < \infty$ such that $\rho(0) = 0$ and $ \rho(x)$ is a constant for all $x \ge c \lam$.
    Its derivative $\rho'(x)$ exists and is continuous except for a finite number of $x$ and $\rho'(0+)=1$.
\end{assumption}

\begin{assumption} \label{assum:signal-difference}
    For $K > 2$, define the minimal difference of the common values between any pair of groups as
    \begin{equation*}
        d_{NT} = (JM)^{-1/2} \; \underset{k\ne l}{\min}\; \norm{\mathring{\btheta}_k - \mathring{\btheta}_l}_2.
    \end{equation*}
   We assume that $d_{NT} \gg \phi_{NT}$ where $\phi_{NT}$ is given in \eqref{eqn:psi-NT}.
\end{assumption}

Assumption \ref{assum:penalty} is commonly assumed in high-dimensional settings.
It is satisfied by concave penalties such as MCP \eqref{eqn:pen-MCP} and SCAD \eqref{eqn:pen-SCAD}.
Assumption \ref{assum:signal-difference} is the separability condition
on the minimum signal difference between groups that is needed to recover the true groups.

\begin{theorem}[Feasible estimator] \label{thm:feasible-in-prob}
    Suppose the conditions in Theorem \ref{eqn:ora-uniform-conv} and Assumption \ref{assum:penalty} and \ref{assum:signal-difference} hold and $K \ge 2$.
    Let $N_{min} = \underset{1\le k \le K}{\min} N_k$ and $N_{max} = \underset{1\le k \le K}{\max} N_k$.
    If $d_{NT} \ge C \lam$ and $\lam \gg \phi_{NT}$ %$\lam \gg \max\paran{N_{\min}^{-1} \sqrt{J}, \phi_{NT}}$, 
    where $C$, $d_{NT}$, and $\phi_{NT}$ are defined in Assumption \ref{assum:penalty}, \ref{assum:signal-difference}, and Equation \eqref{eqn:psi-NT}.
    Then, uniformly over $\Pi$, there exists a local minimizer $\hat\bbeta^\pi$ of the objective function $\calL_{NT}$ given in \eqref{eqn:Q-ls-beta} satisfying
    \begin{equation*}
        \Pr\paran{ \hat\bbeta^\pi = \tilde\bbeta^\pi } \rightarrow 1.
    \end{equation*}
\end{theorem}

Recall that we define $\hat\bW$ as an estimator of $\bW$ that is obtained by applying any clustering method on the column vector of $\mat{\bbeta^\pi}$.
A direct conclusion from Theorem \ref{thm:feasible-in-prob} is that $ \Pr\paran{ \hat\bW = \bW } \rightarrow 1$ since the oracle estimator $\tilde\bbeta^\pi_i = \tilde\bbeta^\pi_j$ for any $i,j \in \calG^{(k)}$.

The oracle property in Theorem \ref{thm:feasible-in-prob} together with Corollary \ref{thm:int-val-asymp-normal} in Appendix \ref{appen:oracle} directly leads to the asymptotic distribution of $\hat\beta^\pi$, which is presented in the following theorem.

\begin{theorem}[Asymptotic normality of integrated value estimator] \label{thm:CI-int-value}
    Suppose the conditions in Theorem \ref{eqn:ora-uniform-conv} and \ref{thm:feasible-in-prob} hold.
    If $ J^{-\kappa / p} \ll \paran{N_{\max} T \paran{1 + \norm{\int\bphi_J(\bx)\calR(d\bx)}_2^{-2}}}^{-1/2}$,
    as either $N_{\min}\rightarrow \infty$ or $T\rightarrow \infty$, we have for any $\bx\in\calX$, % and uniformly over $\pi\in\Pi$,
    \begin{equation*}
        \begin{aligned}
            \sqrt{N T} \cdot \hat\sigma_{\calR}^{(k)}(\pi)^{-1}
            \paran{ \hat V_{\calR}^{(k)}\paran{\pi} -  V_{\calR}^{(k)}\paran{\pi} }
            & \convdist \calN\paran{0, 1} \, ,
        \end{aligned}
    \end{equation*}
    for any $k\in[K]$, where
    $$\hat\sigma_{\calR}^{(k)}\paran{\pi}^2
    = \paran{\int\bu\paran{\pi, \bx}\calR(d\bx)}^\top
    \hat\bSigma^{\pi,(k)} \hat \bOmega^{\pi,(k)} (\hat\bSigma^{\pi,(k) \top})^{-1} \paran{\int\bu\paran{\pi, \bx}\calR(d\bx)},$$
    where $\hat\bSigma^{\pi,(k)}$ and $\hat\bOmega^{\pi,(k)}$ are defined as
        \begin{align}
        \hat\bSigma^{\pi,(k)}
        & = \frac{1}{NT} \sum_{i=1}^{N} \sum_{t=0}^{T-1}
        (\hat\bLam^{i,(k)})^\top
        Z_{i,t}\paran{ Z_{i,t} - \gamma U_{i,t+1}^\pi}^\top
        \hat \bLam^{i,(k)},  \label{eqn:hat-sigma-k}\\
        \hat\bOmega^{\pi,(k)}
        & = \frac{1}{NT}  \sum_{i=1}^{N} \sum_{t=0}^{T-1} (\hat\bLam^{i,(k)})^\top Z_{i,t} Z_{i,t}^\top \hat\bLam^{i,(k)} \paran{ R_{i,t} - \paran{Z_{i,t} - \gamma U_{i,t+1}^\pi}^\top \hat\bLam^{i,(k)} \hat\btheta^{\pi} }^2, \label{eqn:hat-omega-k}
    \end{align}
    and $\hat\bLam^{i,(k)} = \bLam^i(\hat\bw_{\cdot k} \otimes \bI_{JM})$,
    $Z_{i,t}$ and $U^{\pi}_{i,t}$ are defined in \eqref{eqn:z} and \eqref{eqn:u}, respectively.
\end{theorem}

Theorem \ref{thm:CI-int-value} imply that the group-wise estimator converges at a rate of $\paran{N T P_k}^{-1/2}$, $k\in [K]$.
The estimators in \cite{luckett2019estimating,jiang2016doubly,thomas2015high} typically converge at a rate of $\paran{N P_k}^{-1/2}$ and are not suitable for settings when one sub-population has only a few trajectories.
Our estimation procedure also aggregates information along horizon $T$.
The theoretical property along $T$ is obtained by treating finite and infinite $T$ separately, using the matrix concentration and martingale center limit theorem.

The dimension of the unknown parameters $\bbeta$ in  \eqref{eqn:Q-ls-beta} is $NJM$, which will diverge as sample size $N$ increases.
Without the $K$-hetero MDP assumption and the penalty term, the $\ell_2$ convergence of the parameters can be shown to be $1/\sqrt{T}$ under the TH Markov and CMI assumption, employing arguments similar to those in  \cite{luckett2019estimating} and \cite{shi2020statistical}.
The information accumulated along $N$ does not help.
In contrast, the {\PE} obtain $1/\sqrt{TNP_{\min}}$ convergence rate where $P_{\min} = \min\braces{P^{(1)}, \cdots, P^{(k)}}$.
The estimation in \cite{luckett2019estimating} is based on the state value $V(\cdot)$ and the calculation necessitates a correct specification of the behavior policy for the importance ratio that shows up in the Bellman equation of $V(\cdot)$.
In contrast, our method estimates the action value function $Q(\cdot)$ and derive the corresponding state value estimators.
As a result, we do not need to specify the behavior policy, nor do we need to estimate it from the observed dataset.
Our method can be viewed as an implicit importance weighting.

%\begin{theorem}  \label{thm:feasible-beta-asym}
%    Under the conditions in Theorem \ref{eqn:ora-asymp-normal} and \ref{thm:feasible-in-prob},
%    for any $\bnu \in \RR^{JMN}$ satisfying $J^{\kappa / p} \gg \sqrt{N_{\max} T \paran{1 + \norm{\bnu}_2^{-2}}}$, we have as either $N_{\min}\rightarrow\infty$ or $T\rightarrow \infty$,
%    \begin{equation*}
%        \sqrt{NT} \; \hat\sigma_{\beta_i}^{-1} \bnu^\top \paran{ \hat\bbeta^{\pi}_i- \mathring\bbeta^{\pi}_i }
%        \convdist \calN\paran{0, 1},
%    \end{equation*}
%    where
%    \begin{align*}
%        \hat \sigma_{\beta_i}
%        & = \bnu^\top (\hat\bw_{i\cdot}^\top \otimes\bI_{JM})  \hat\bSigma^{-1} \hat\bOmega (\hat\bSigma^\top)^{-1} (\hat\bw_{i\cdot}\otimes\bI_{JM})  \bnu \\
%        \hat\bSigma
%        & = \frac{1}{NT} \sum_{i=1}^{N} \sum_{t=0}^{T-1} (\hat\bLam^{i})^\top  Z_{i,t}\paran{ Z_{i,t} - \gamma U_{i,t+1}^\pi}^\top \hat \bLam^{i},\\
%        \hat\bOmega
%        & = \frac{1}{NT} \sum_{i=1}^{N} \sum_{t=0}^{T-1} (\hat\bLam^{i})^\top Z_{i,t} Z_{i,t}^\top \hat\bLam^{i} \paran{ R_{i,t} - \paran{Z_{i,t} - \gamma U_{i,t+1}^\pi}^\top \hat\bLam^{i} \hat\btheta^{\pi} }^2,
%    \end{align*}
%    and $\hat\bLam^{i} = \bLam^{i} (\hat\bW \otimes \bI_{JM})$,
%    $Z_{i,t}$ and $U^{\pi}_{i,t}$ are defined in \eqref{eqn:z} and \eqref{eqn:u}, respectively.
%\end{theorem}
%Asymptotic normality of the value functions similar to Corollary \ref{thm:val-asymp-normal} can be established by applying Corollary \ref{thm:feasible-beta-asym} and thus is omit here.

\subsection{Optimal policy} \label{sec:theory-policy}

In this section, we establish a convergence result for the estimated optimal policy for each homogeneous sub-population assuming that the parametric class $\Pi$ satisfies the following properties.
\begin{assumption} \label{assum:policy-to-value}
    The map $\balpha \rightarrow V_{\calR}\paran{\pi\paran{\balpha}}$ has a unique and well-separated maximum $\balpha^*$ in the interior of the support of $\balpha$, where $V_{\calR}\paran{\pi\paran{\balpha}}$ is defined in \eqref{eqn:inte-value}.
\end{assumption}

\begin{assumption}  \label{assum:policy-smooth}
    For any $\bx\in\calX$ and $a \in\calA$, we have as $\delta \downarrow 0$,
    \begin{equation*}
        \underset{\norm{\balpha_1-\balpha_2}_2 \le \delta}{\sup}\; \EE\norm{\pi\paran{a, \bx; \balpha_1} - \pi\paran{a, \bx; \balpha_1}} \rightarrow 0.
    \end{equation*}
\end{assumption}

Assumption \ref{assum:policy-to-value} requires that the true optimal decision in each state is unique (see also the Assumption A.8 of \cite{ertefaie2018constructing} and Assumption 6 of \cite{luckett2019estimating}) and is a standard assumption in M-estimation.
Assumption \ref{assum:policy-smooth} requires smoothness on the class of the polices.
The following lemma shows that the parametric class of policies $\Pi$ defined in \eqref{eqn:policy-param} satisfies Assumption \ref{assum:policy-to-value} and \ref{assum:policy-smooth}.

\begin{lemma}  \label{thm:policy-prop}
    The parametric class of policies $\Pi$ defined in \eqref{eqn:policy-param} satisfies Assumption \ref{assum:policy-to-value} and \ref{assum:policy-smooth}.
\end{lemma}

Theorem \ref{thm:opt-policy-conv} establishes that the estimated optimal policy for each homogeneous sub-population converges in probability to the true optimal policy over $\pi$
and that the estimated value of the estimated optimal policy converges to the true value of the estimated optimal policy.

\begin{theorem} \label{thm:opt-policy-conv}
    Suppose the conditions in Theorem  \ref{eqn:ora-uniform-conv} and \ref{thm:feasible-in-prob} and Assumption \ref{assum:policy-to-value}, \ref{assum:policy-smooth} hold.
    Let $\hat\balpha^{(k)\,*} = \arg\max_{\balpha} \hat V_{\calR}^{(k)}\paran{\pi(\alpha)}$ and
    $\balpha^{(k)\,*} = \arg\max_{\balpha} V_{\calR}^{(k)}\paran{\pi(\alpha)}$, we have as either $N_{\min}\rightarrow\infty$ or $T\rightarrow \infty$,
    \begin{enumerate}[label=(\roman*)]
        \item $\norm{\hat\balpha^{(k)\,*} - \balpha^{(k)\,*}}_2 \convprob 0$.
        \item $\abs{V_\calR\paran{\pi(\hat\balpha^{(k)\,*})} - V_\calR\paran{\pi(\balpha^{(k)\,*})}} \convprob 0$.
    \end{enumerate}
\end{theorem}

%!TEX root = 0-main.tex

\section{Computation} \label{sec:computation}

%\subsection{Optimization algorithm}

%\cite{chi2015splitting}

The optimization problem of \eqref{eqn:Q-ls-beta} is challenging because of the coupling of $\bbeta_i$ and $\bbeta_j$ in the penalty term.
This problem is similar to that in \cite{chi2015splitting} and can be solved by the alternating direction method of multipliers (ADMM) \citep{boyd2011distributed},
the alternating minimization algorithm (AMA) \citep{tseng1991applications},
and the general iterative shrinkage and thresholding method \citep{gong2013iterative}.
All three approaches employ variable splitting to handle the shrinkage penalties in \eqref{eqn:Q-ls-beta}.
In this section, we summarize the ADMM procedure to solve \eqref{eqn:Q-ls-beta}; readers are referred to \cite{boyd2011distributed,chi2015splitting,gong2013iterative} and references therein for more information on alternative optimization algorithms.
For brevity, we suppress the superscript $\pi$ in $\bbeta^\pi$ in this section.

We first recast \eqref{eqn:Q-ls-beta} as an equivalent constrained problem:
\begin{eqnarray}
\min_{\bbeta, \bdelta} & & \calL\paran{\bbeta, \bdelta}
             :=  \bG\paran{\pi, \bbeta}^\top \bG\paran{\pi, \bbeta}
             + \frac{1}{N^2}\sum_{1\le i<j \le N} p\paran{(JM)^{-1/2}\norm{\bdelta_{ij}}_2, \lam }
             \label{eqn:constrain-opt-1} \\
s.t. & & \bdelta_{ij} = \bbeta_i - \bbeta_j, \quad \forall\{i, j\}: 1 \leq i < j \leq N, \nonumber
\end{eqnarray}
where $\bdelta = \brackets{\bdelta_{ij}^\top, i < j}^\top$.
Applying the augmented Lagrangian method (ALM), the solution of \eqref{eqn:constrain-opt-1} can be obtained by minimizing
\begin{equation} \label{eqn:alm}
\calL_{AL}\paran{\bbeta, \bdelta, \bnu}
:=
\calL\paran{\bbeta, \bdelta}
+ \frac{1}{N^2} \sum_{i < j} \langle \bnu_{ij}, \bbeta_i - \bbeta_j - \bdelta_{ij} \rangle
+ \frac{\rho}{2JMN^2} \sum_{i < j} \norm{  \bbeta_i - \bbeta_j -  \bdelta_{ij} }_2^2,
\end{equation}
where $\bnu = \brackets{\bnu_{ij}^\top, i<j}^\top$ is a vector of Lagrangian multipliers and $\rho$ is a non-negative tuning parameter.
ADMM minimizes the augmented Lagrangian \eqref{eqn:alm} by iteratively optimizing over blocks of variables $\bbeta$, $\bdelta$, and the dual parameter $\bnu$.
Specifically, we start at initial values $\paran{\bbeta^0, \bdelta^0, \bnu^0}$ and update $\paran{\bbeta^s,\bdelta^s,\bnu^s}$ at the $s$-th iteration as follows.
\begin{enumerate}[leftmargin=*,label={\sc Step }\arabic*.]
\item Given $\paran{\bdelta^s,\bnu^s}$, update $\bbeta^{s+1}$ by solving $\bbeta$ from $\nabla_{\bbeta}\;\calL_{AL}\paran{\bbeta, \bdelta^s, \bnu^s} = 0$.
\item Given $\bbeta^{s+1}$, update $\bdelta^{s+1}$ using the analytical forms, for MCP, with
\begin{equation*}
    \bdelta_{ij}^{s+1} = \begin{cases}
        \frac{\calS\paran{\bbeta_i^{s+1} - \bbeta_j^{s+1} + \rho^{-1} \bnu_{ij}^s, \lambda/\rho}}{1 - 1/(\gamma \rho)} &\text{if } \|\bbeta_i^{s+1} - \bbeta_j^{s+1} + \rho^{-1} \bnu_{ij}^s\| \leq \gamma \lambda \\
        \bbeta_i^{s+1} - \bbeta_j^{s+1} + \rho^{-1} \bnu_{ij}^s &\text{if } \|\bbeta_i^{s+1} - \bbeta_j^{s+1} + \rho^{-1} \bnu_{ij}^s\| > \gamma \lambda
    \end{cases}
\end{equation*}
and, for SCAD, with
\begin{equation*}
    \bdelta_{ij}^{s+1} = \begin{cases}

        \calS\paran{\bbeta_i^{s+1} - \bbeta_j^{s+1} + \rho^{-1} \bnu_{ij}^s, \lambda/\rho} &\text{if } \|\bbeta_i^{s+1} - \bbeta_j^{s+1} + \rho^{-1} \bnu_{ij}^s\| \leq \lambda + \lambda / \rho \\

        \frac{\calS\paran{\bbeta_i^{s+1} - \beta_j^{s+1} + \rho^{-1} \bnu_{ij}^s,\gamma \lambda / ((\gamma - 1)\rho)}}{1 = 1 / ((\gamma - 1)/\rho)} &\text{if } \lambda + \lambda / \rho < \|\bbeta_i^{s+1} - \bbeta_j^{s+1} + \rho^{-1} \bnu_{ij}^s\| \leq \gamma \lambda \\

        \bbeta_i^{s+1} - \bbeta_j^{s+1} + \rho^{-1} \bnu_{ij}^s &\text{if } \|\bbeta_i^{s+1} - \bbeta_j^{s+1} + \rho^{-1} \bnu_{ij}^s)\| > \gamma \lambda

    \end{cases},
\end{equation*}
where $\calS\paran{x, c} = {\rm sign}(x) \paran{\abs{x} - c}_+$ is the soft thresholding rule and $\paran{x}_+ = x$ if $x>0$ and $0$ otherwise.
\item Update the dual parameter  $\bnu_{ij}^{s+1} = \bnu_{ij}^s + \rho (\bbeta_i^{s+1} - \bbeta_j^{s+1} - \bdelta_{ij}^{s+1})$.
\end{enumerate}
The iteration terminates when the norm of the primal residual is smaller than some pre-specified small tolerance $\epsilon$, that is, when $\|\bbeta_i - \bbeta_j - \bdelta_{ij}\| < \epsilon$.
The algorithmic convergence of ADMM has been established in \cite{boyd2011distributed} and \cite{chi2015splitting}.

%\begin{remark}
%    A practical note is that each $\bnu_{ij}$ is a vector, which will cause trouble in the first step when it has a closed-form solution. However, if all $\nu_{ij}$ are concatenated into a big vector, this closed-form solution can be written. In this case we can simply implement a ``difference" matrix, e.g. $\Delta = \{(\be_i - \be_j)^s\}_{1 \leq i \leq j \leq N}$ where $\be_i$ is an N-dimensional vector whose $i$ th coordinate is 1 and all others zero. And our objective function will be then in terms of matrix - vector product form instead of summation form. However, when the dimension is very large, such concatenation will become infeasible as the dimension of $\Delta$ will explode. In such case, we need to do iterative methods for continuously differentiable functions (i.e. gradient descent) to find an optimal $\bbeta$ or an accurate approximation for such $\bbeta$ if resource-constrained.
%\end{remark}

%\subsection{Tuning parameter $\lam$ and $K$ selection}

%!TEX root = 0-main.tex

\section{Simulations} \label{sec:simul}

In this section, we compare the performance of {\PE} and {\PI} with their mean-value-based counterparts (i.e., policy evaluation and iteration) on simulated data.
We generate the initial state variable $X_{i,1}$ from a standard normal distribution $\calN\paran{\bzero_2, \bI_2}$.
The available action set is $\calA=\braces{0, 1}$.
The system evolves according to
\begin{equation*}
    X_{i, t+1} =
    \begin{bmatrix}
        0.75 (2 A_{i,t} - 1)  & 0 \\
        0 & 0.75 (1 - 2 A_{i,t})
    \end{bmatrix}
    X_{i,t} + \beps_{i,t},
\end{equation*}
where $\beps_{i,t} \overset{i.i.d}{\sim} \calN\paran{\bzero, \bI_2/4}$.
The data-generating actions are i.i.d.\ Bernoulli random variables with expectation $0.5$ and are independent of $\bx_{i,t}$ for any $t\ge 1$.
%The behavior action $A_{i, t}$ depends on policies; in particular, we consider three types of policies: random policy, fixed policy and optimal policy.
We consider $K=2$ homogeneous groups.
The immediate reward $R_{i,t}$ is defined by
\begin{equation*}
    R_{i,t} = X_{i,t}^\top \bb_k  - 0.25 (2 A_{i,t} - 1), \quad\text{for}\quad \forall i \in \calG_k,\; k \in 1,2,
\end{equation*}
where $\bb_1 = [2, -1]$ and $\bb_2 = [-2, 1]$.
% $\bb_1 \ne \bb_2 \in \RR^2$ % are real coefficient vectors. % and $\eta_{i,t} \overset{i.i.d}{\sim} \calN\paran{0, 0.05}$.
Note that $\bb_k$ is not the coefficient of Q function.
Unless otherwise specified, we use $\gamma = 0.6$ for discount factor, the MCP \eqref{eqn:pen-MCP} with $\eta=1.5$ for penalization throughout all experiments.
We present results for the consistency and asymptotic distribution in policy estimation and solving for optimal policies.

\subsection{Coefficients of the value function of a given policy \texorpdfstring{$\pi$}{pi} } \label{sec:simul-PE1}

The target policy $\pi$ to be evaluated is specified as %  \attn{may need to be changed}
\begin{equation*}
    \pi\paran{a | \bx} =
    \begin{cases}
        0, & x_1 > 0 \text{ and } x_2 > 0; \\
        1, & \text{otherwise},
    \end{cases}
\end{equation*}
where $x_i$ denotes the $i$-th element of a vector $\bx$.
% For each setting,   we try different values of tuning parameters $\lam$.
For each group, we simulate $N_k = 100$ trajectories of length $T=10$.
Thus, we have in total $N=200$ trajectories.
Recall that we use $Q$ function model \eqref{eqn:Q-hetero-glm} with features arranged in \eqref{eqn:z}. 
The parameter  $\bbeta^i$ contains $M$ blocks of coefficients and each block corresponds to one action from $\calA=\{a_1, \cdots, a_M\}$. 
In this simulation setting, we have $\bbeta^i = [\bbeta_1^{i \top}, \bbeta_2^{i \top}]^\top \in \RR^{4}$ where $\bbeta_1^{i}$ and $\bbeta_2^{i} \in \RR^2$ are coefficients associated with action $a_1 = 0$ and $a_2=1$, respectively. 
Figure~\ref{fig:random_a} plots the $\hat\bbeta_1^i$ and $\hat\bbeta_2^i$ for all $200$ trajectories.
We further run K-means algorithms to cluster $\{\bbeta^i\}_{i\in[N]}$ into two groups and calculate the cluster centroids $\hat\btheta_1^{(k)} = [\hat\btheta_1^{(k) \top}, \hat\btheta_2^{(k) \top}]^\top$ for $k=1, 2$. 
The centroids $\hat\btheta_1^{(k)}$ and $\hat\btheta_2^{(k)}$ for $k=1,2$ are also plotted in the figure as the red dots. 
The left and right columns of Figure~\ref{fig:random_a} corresponds to using $\lambda=0.1$ or $0.05$, respectively, in the MCP \eqref{eqn:pen-MCP} penalization. 

As shown in Figure~\ref{fig:random_a}, the proposed algorithm is able to recover mostly correct coefficients with an appropriate choice of $\lam$ even when the group centroids are relatively close to each other.
The tuning parameter $\lam$ controls the ``focus of the group.''
Smaller $\lam$ encourages heterogeneity, that is, different values for different trajectories, while large $\lam$ enforce homogeneity.
When $\lambda$ is too small, we lose efficiency by grouping trajectories together.
When $\lambda$ is too big, the problem reduces to homogeneous policy evaluation with the same coefficients and hence loses the heterogeneity that we are seeking.

\begin{figure}[ht!]
  \centering
  \begin{subfigure}[b]{0.44\textwidth}
  \includegraphics[width=1\textwidth]{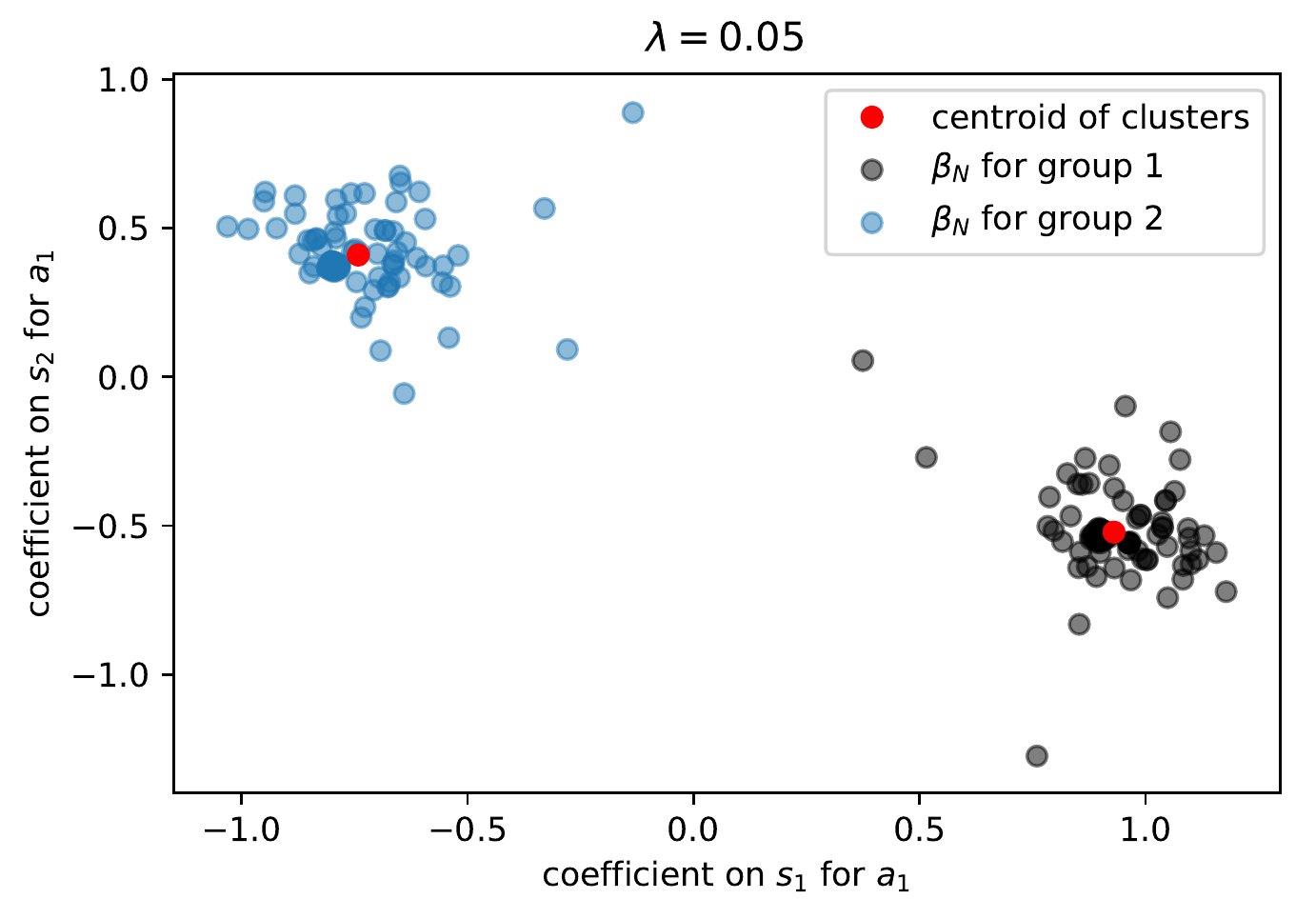}
  \caption{$\{ \hat\bbeta_1^i \}_{i\in [N]}$ for action $a_1=0$ with $\lambda = 0.05$}
  \label{fig:far1}
  \end{subfigure}
  \hspace{4ex}
  \begin{subfigure}[b]{0.44\textwidth}
  \includegraphics[width=1\textwidth]{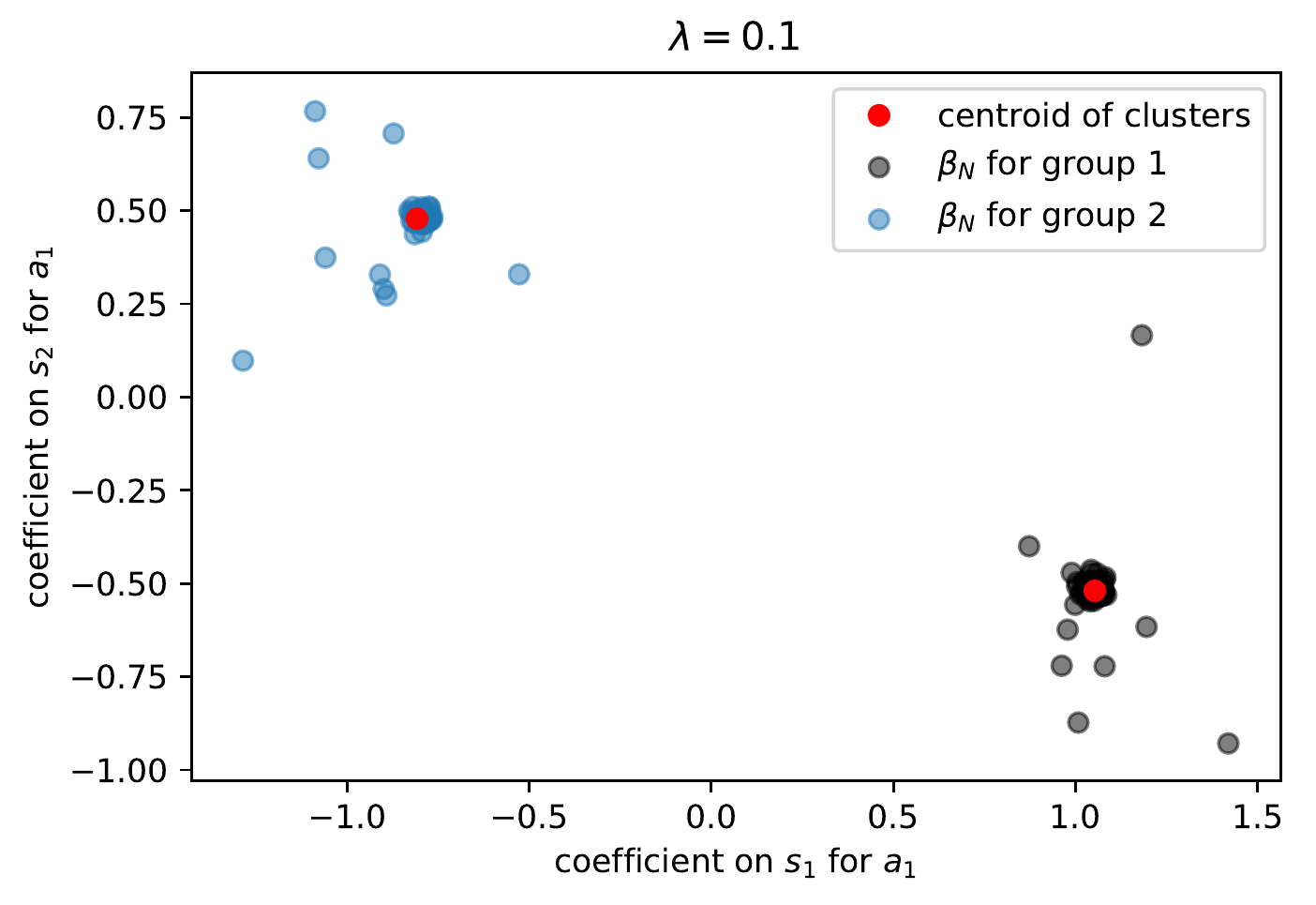}
  \caption{$\{ \hat\bbeta_1^i \}_{i\in [N]}$ for action $a_1=0$ with $\lambda = 0.1$}
  \label{fig:far2}
  \end{subfigure}
  \begin{subfigure}[b]{0.44\textwidth}
  \includegraphics[width=1\textwidth]{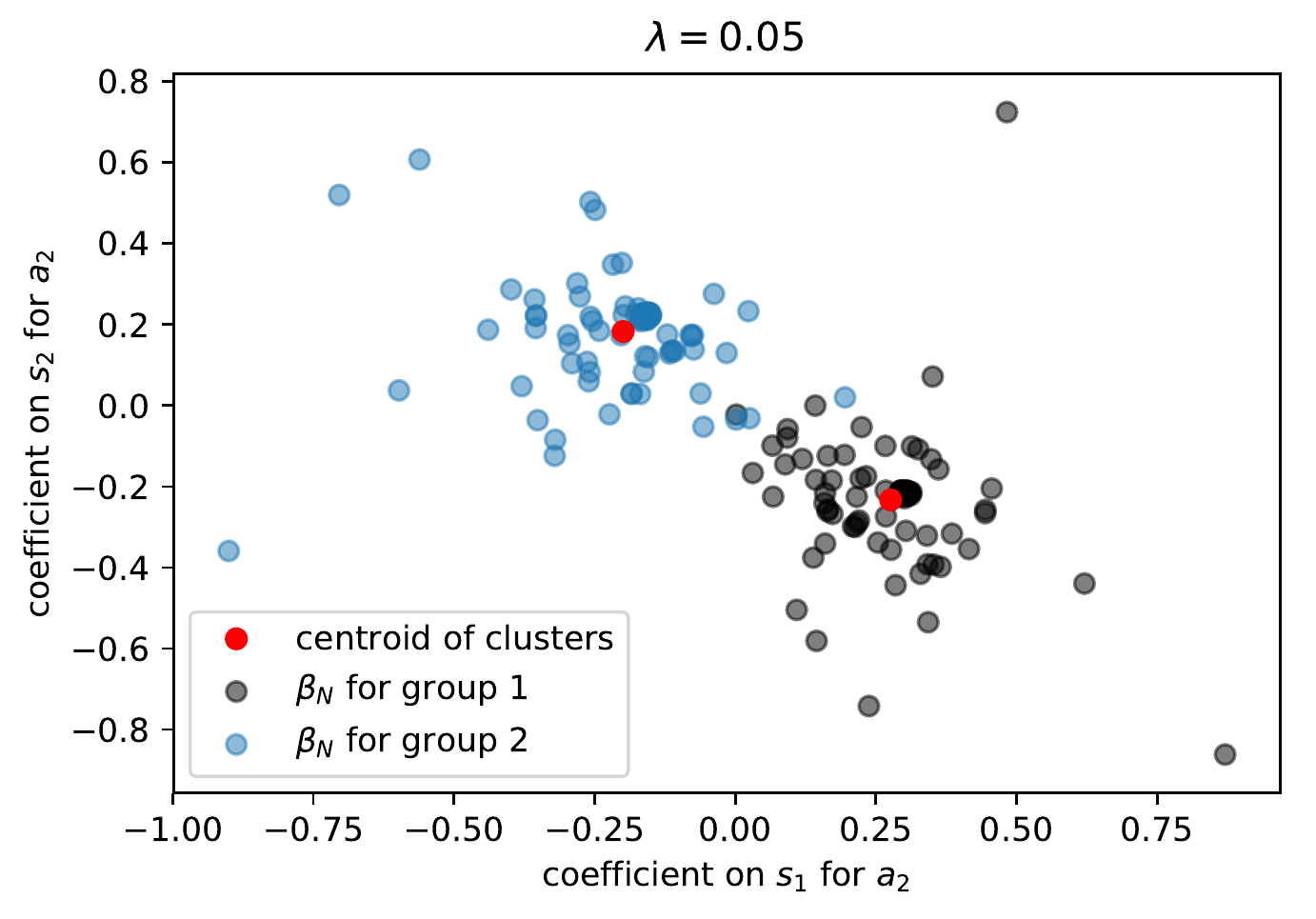}
  \caption{$\{ \hat\bbeta_2^i \}_{i\in [N]}$ for action  $a_2=1$, $\lambda = 0.05$}
  \label{fig:close1}
  \end{subfigure}
  \hspace{4ex}
  \begin{subfigure}[b]{0.44\textwidth}
  \includegraphics[width=1\textwidth]{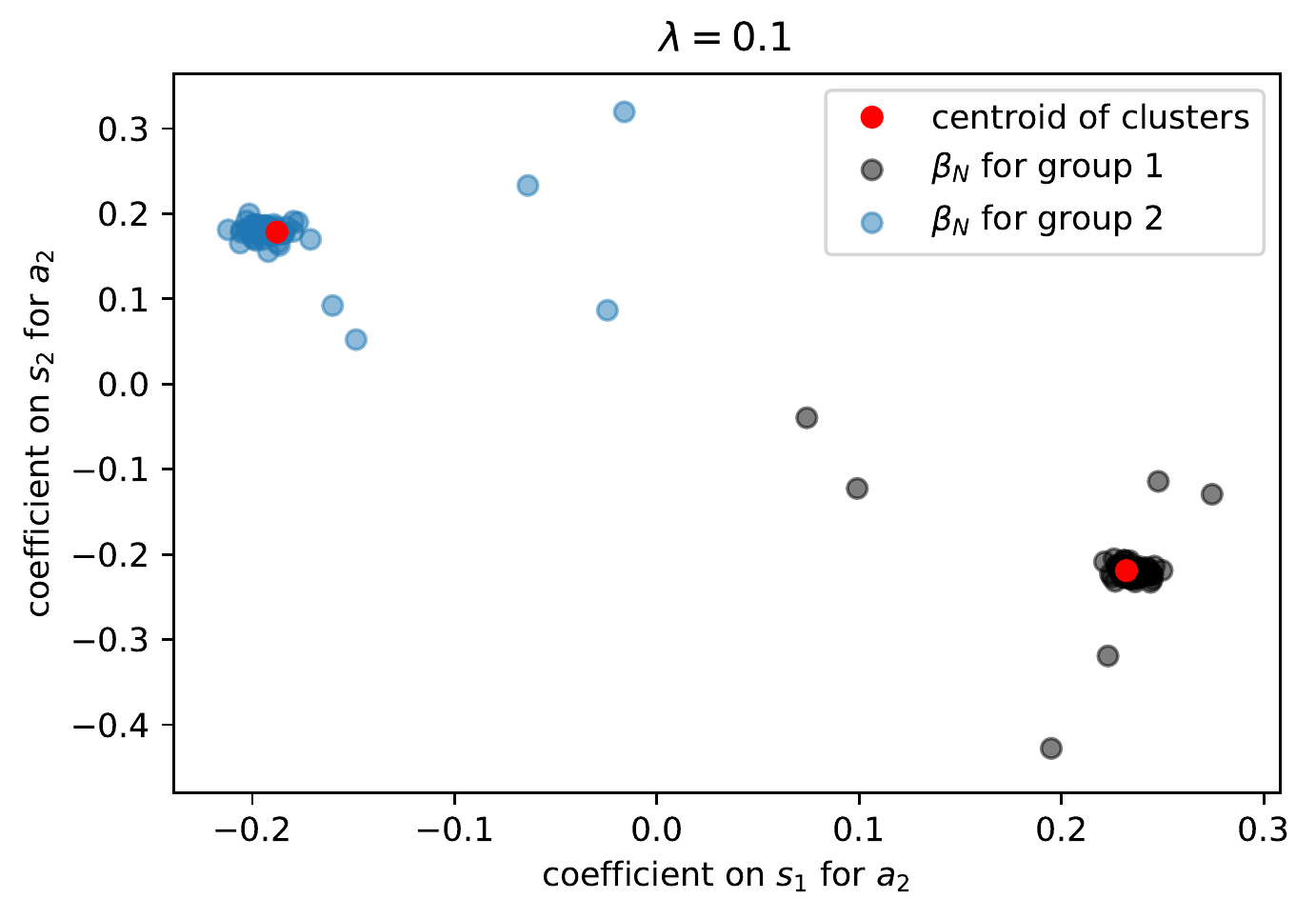}
  \caption{$\{ \hat\bbeta_2^i \}_{i\in [N]}$ for action $a_2=1$, $\lambda = 0.1$}
  \label{fig:close2}
  \end{subfigure}
  \caption{Coefficients of $Q$ functions and the centroids estimated by {\PE}.
      We use blue and gray colors to denote the true membership of each point.
      We use MCP \eqref{eqn:pen-MCP} penalization with $\eta = 1.5$. 
      Plots (a) and (c) on the left correspond to the results with $\lambda = 0.05$; plots (b) and (d) on the right correspond to the results with $\lambda = 0.1$. 
      We can see that with an appropriate choice of the tuning parameter $\lambda$, the majority of fitted coefficients  fall in a very small neighborhood of the centroids and can be hence correctly classified by simple clustering algorithms. }
  \label{fig:random_a}
\end{figure}

\subsection{Confidence intervals of the integrated value of a given policy $\pi$} \label{sec:simul-PE2}

We compare the  integrated value estimated by ACPE and classical mean-value policy evaluation (MVPE).
The true integrated value $\mathring V^\pi_\calR$ is computed by Monte Carlo approximation.
Specifically, we simulate $N = 5,000$ independent trajectories (2,500 trajectories for each group) with initial states distributed according to $\calR$.
For each trajectory, we simulate $T$ steps to obtain $\braces{Y_{i,t}^{(k)}}$ for $1 \le i \le N$, $1 \le t \le T$, and $k = 1, 2$.
For each sub-group, the true value is approximated by $\mathring{V}_{\calR}^{\pi, (k)} \approx N_k^{-1} \sum_{i=1}^{N_k} \sum_{t=1}^T \gamma^t Y_{i,t}^{(k)}$.
For the MVPE, the estimated $\hat V^{MV}_\calR(\pi)$ and the corresponding CI are obtained using equation (3.9) and (3.10) in \cite{shi2020statistical}.
For the ACPE, we use MCP \eqref{eqn:pen-MCP} with $\eta=1.5$ and $\lambda=0.1$. 
The estimated $\hat V^{AC}_\calR(\pi)$ and the corresponding CI are obtained using results in Theorem \ref{thm:CI-int-value}.
We compare the true value with our estimated values with $n = 20, 50, 100$ and $T = 10, 30, 40$. 
The coverage probability is plotted in Figure~\ref{fig:ecp} for trajectories from group 1 and 2 respectively in subplot (a) and (b). 
We can see that heterogeneous estimation performs consistently better in terms of coverage probability. 
%Due to the stationarity of this Markov chain, homogeneous estimation starts to perform better when $N$ and $T$ are both large as the estimated V between two groups are close, but is still behind heterogeneous estimation.

\begin{figure}[htpb!]
    \centering
    \begin{subfigure}[b]{0.44\textwidth}
        \includegraphics[width=1\textwidth]{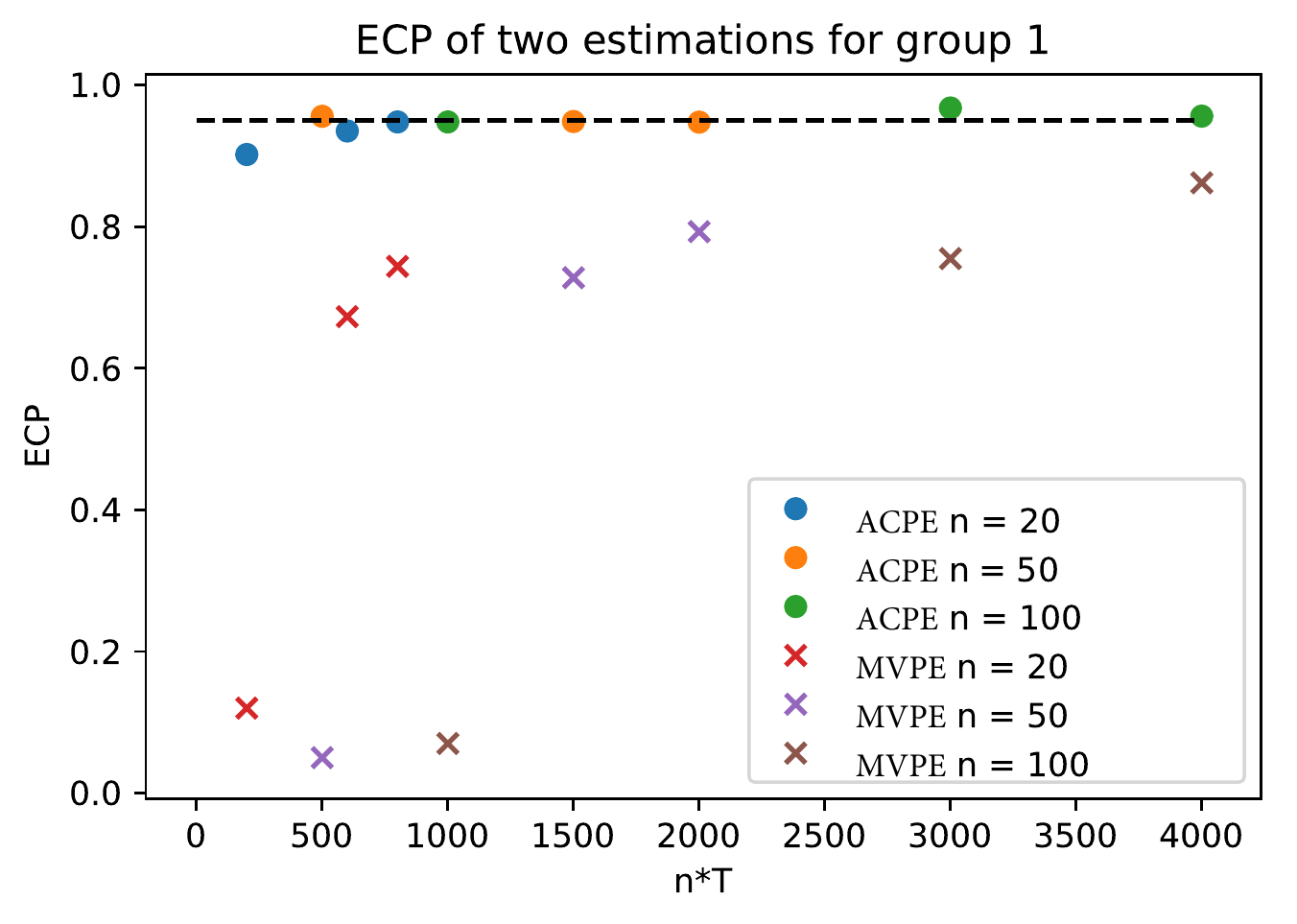}
        \caption{For trajectories from group 1}
        \label{fig:ecp1}
    \end{subfigure}
    \hspace{4ex}
    \begin{subfigure}[b]{0.44\textwidth}
        \includegraphics[width=1\textwidth]{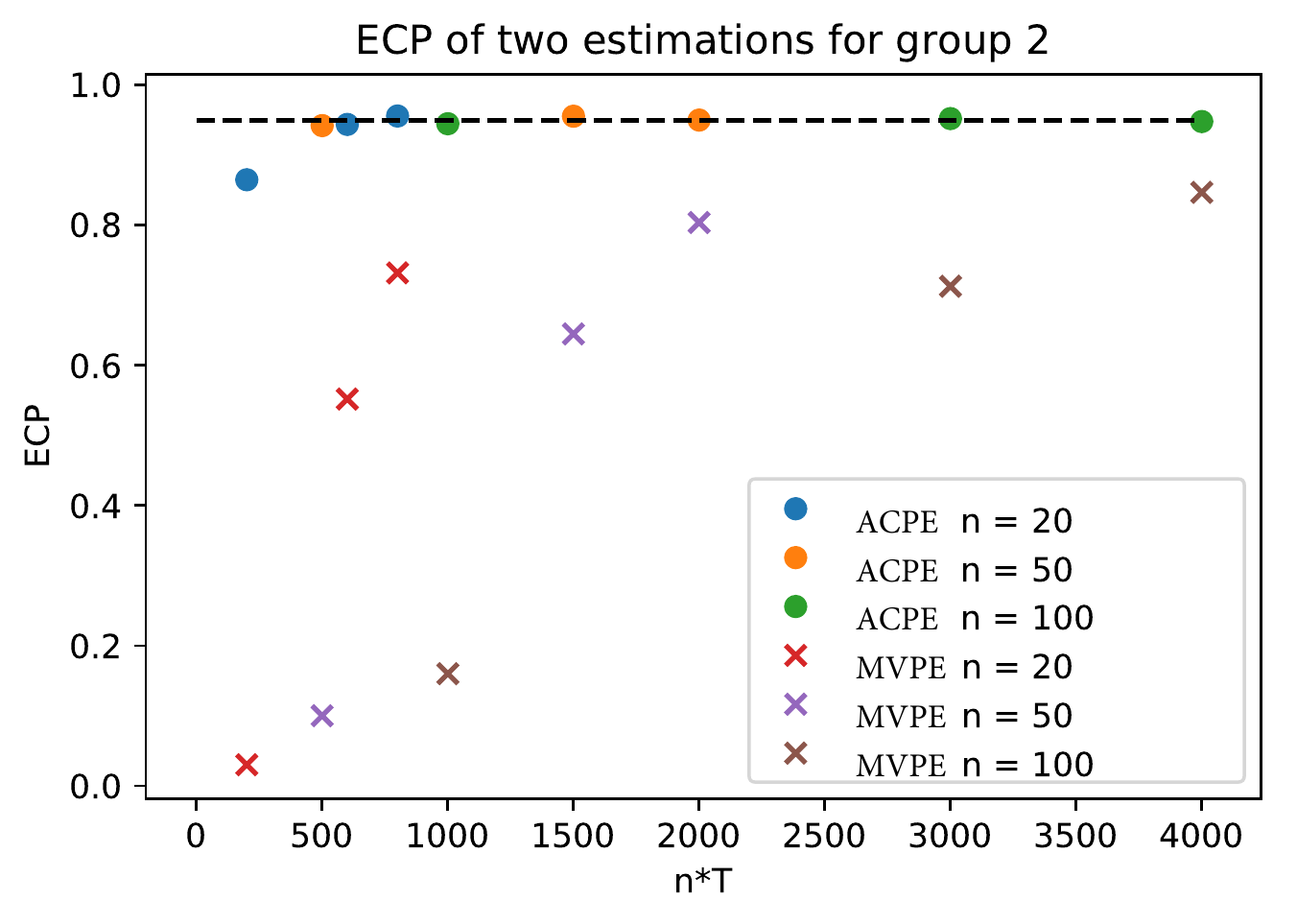}
        \caption{For trajectories from group 2}
        \label{fig:ecp2}
    \end{subfigure}
    \caption{Empirical covering probabilities (ECP) of the 95\% confidence interval. 
        ACPE represents our auto-clustered policy evaluation and its corresponding CI in Theorem \ref{thm:CI-int-value}. 
        MVPE represent the classic mean-value-based policy evaluation and its corresponding CI are constructed using equation (3.9) and (3.10) in \cite{shi2020statistical}. 
        Subplot (a) and (b) correspond to the ECP's for trajectories from two different groups $k=1$ and $2$, respectively. 
    }
    \label{fig:ecp}
\end{figure}

\subsection{Parametric optimal policy} \label{sec:simul-opt}

We train 100 steps of Algorithm \ref{alg:ACPI} (\PI) to obtain two policies, which we denote
$\pi(\hat \balpha^{(1) *}_{ACPI})$ and $\pi(\hat \balpha^{(2) *}_{ACPI})$, respectively for estimated group 1 and 2.  
%$\pi(\balpha_{G1})$ and $\pi(\balpha_{G2})$. 
%% Training curve to show convergence. 
%We present the training curve (the value of $\hat{V}$ by every 5 iterations) in Figure \ref{fig:training_curve}. 
We follow the same protocol to obtain a policy $\pi(\hat\balpha^{*}_{MVPI})$ via standard mean-value-based policy iteration (MVPI).

Firstly, we show how different are the actions given by $\pi(\hat \balpha^{(1) *})$, $\pi(\hat \balpha^{(2) *})$, and $\pi(\hat\balpha^{*}_{MVPI})$ on a specific sample. 
We take a $j$-th test sample with $X_{j,1} =  \begin{bmatrix}  1.2506\\  0.77477 \end{bmatrix}$,  $X_{j, 2} = \begin{bmatrix}1.0277\\-0.52615\end{bmatrix}$, $A_{j, 1} = 1$, $Y_{j, 1} = 1.6559$.
The optimal actions at $t=2$, i.e. $A_{j, 2}$, suggested by {\PI} for group 1 and 2 and by MVPI are given in Table \ref{tab:estimated_policy}.
This trajectory is in fact in group 1, which is correctly clustered by  {\PE}. 
Thus, the optimal action should be chosen by following $\pi\paran{\hat \balpha^{(1) *}_{ACPI}}$ which prefers action $A_{j, 2}=1$. 
If we ignore heterogeneity and naively adopt MVPI, we will wrongly weight both actions evenly, yielding a lower expected outcome.

\begin{table}[ht!]
\centering
\begin{tabular}{|c|c|c|}
\hline
 Estimated optimal policy & $\Pr(A = 0)$ & $\Pr(A = 1)$ \\
 \hline
 $\pi\paran{\hat \balpha^{(1) *}_{ACPI}}$ & 0.14 & 0.86 \\
\hdashline
$\pi\paran{\hat \balpha^{(2) *}_{ACPI}}$  & 0.75 & 0.25 \\
 \hline
 $\pi\paran{\hat\balpha_{MVPI}}$ & 0.58 & 0.42 \\
 \hline
\end{tabular}
\caption{The optimal action given by {\PI} for group 1 and 2 and by MVPI on a random trajectory. 
We use $\eta = 1.5$ and $\lambda=0.1$ in {\PI}. }
\label{tab:estimated_policy}
\end{table}

%% Convergence of ACPI
%\begin{figure}[htpb!]
%\centering
%\includegraphics[width=0.5\textwidth]{figs/value_iteration.pdf}
%\caption{The value of $\hat{V}$ reported every 5 iterations.}
%\label{fig:training_curve}
%\end{figure}

We further generate $N=500$ samples distributed according to $\calN\paran{\bzero, \bI_2}$ (250 samples for each group). 
We estimate the true value of all three policies using Monte Carlo simulation $\mathring{V}_{\calR}\paran{\pi} \approx N^{-1} \sum_{i=1}^{N} \sum_{t=0}^{T-1} \gamma^t \bY_{i,t}$ at $T = 50$. 
We repeat 10 times and obtain the average values of all three policies on different groups. 
The results are presented in Table \ref{tab:estimated_policy_value}. 
%$\mathring{V}(\pi(\balpha_{G})) = 0.0024256$ with variance $0.0095241$, $\mathring{V} (\pi(\balpha_{G_1})) = 0.11926$ with variance $0.087801$, and $\mathring{V}(\pi(\balpha_{G_2})) = 0.10398$ with variance $0.09141$. 
It shows that policies produced by {\PI} are indeed better in terms of expected discounted total sum of rewards. 

\begin{table}[ht!]
    \centering
    \begin{tabular}{|c|c|c|}
        \hline
        \multirow{2}{*}{Estimated optimal policy} & \multicolumn{2}{c}{Integrated value $\mathring{V}_\calR(\pi)$} \\ \cline{2-3}
        & True group 1 & True group 2 \\
        \hline
        $\pi\paran{\hat \balpha^{(1) *}_{ACPI}}$ & 0.119 & - \\
        \hdashline
        $\pi\paran{\hat \balpha^{(2) *}_{ACPI}}$  & - & 0.104 \\
        \hline
        $\pi\paran{\hat\balpha_{MVPI}}$ & 0.002  & 0.002 \\
        \hline
    \end{tabular}
    \caption{The integrated values of the optimal policies given by {\PI} for group 1 and 2 and by MVPI. 
    We use $\eta = 1.5$ and $\lambda=0.1$ in {\PI}.}
    \label{tab:estimated_policy_value}
\end{table}

%!TEX root = 0-main.tex

\section{Real data application: MIMIC-III} \label{sec:appl}

In this section, we illustrate the advantages of the {\PI} on the Medical Information Mart for Intensive Care version III (MIMIC-III) Database \citep{johnson2016mimic}, which is a freely available source of de-identified critical care data from 2001--2012 in six ICUs at a Boston teaching hospital.

We consider a cohort of sepsis patients, following the same data processing procedure as in \cite{komorowski2018artificial}.
Each patient in the cohort is characterized by a set of 47 variables, including demographics, Elixhauser premorbid status, vital signs, and laboratory values.
Since we do not consider high-dimensional issue in the present paper, we conduct a dimension reduction using principal component analysis and choose top $10$ principal components as our state features, which explain around 97\% of the total variance.. 
Patients' data were coded as multidimensional discrete time series $X_{i,t} \in\RR^{10}$ for $1\le i \le N$ and $1\le t \le T_i$ with four-hour time steps.
The actions of interests are the total volume of intravenous (IV) fluids and maximum dose of vasopressors administrated over each four-hour period.
We discretize two action variables into three levels, respectively. Our low corresponds to 1 - 2, medium corresponds to 3 and high corresponds to 4 - 5 in \cite{komorowski2018artificial}.
The combination of the two drugs makes $M = 3 \times 3 = 9$ possible actions in total.
In the final processed dataset, we sampled \num{1000} unique adult ICU admissions, corresponding to unique trajectories to be fed into our algorithms.
The observation length $T_i$ varies across trajectories. %, with 12,987 records in total.
%We then perform dimension reduction by PCA, resulting in ten variables that explain around 97\% of the variance.

The reward signal $R_{i,t}$ is important and is crafted carefully in real applications.
For the final reward, we follow \cite{komorowski2018artificial} and use hospital mortality or 90-day mortality.
Specifically, when a patient survived, a positive reward was released at the end of each patient's trajectory; a negative reward was issued if the patient died.
For the intermediate rewards, we follow \cite{prasad2017reinforcement} and associate reward to the health measurement of a patient.
The detailed description of the data pre-processing is presented in Section \ref{appen:mimic-iii} of the supplemental material.
More information about the dataset can also be found in \cite{komorowski2018artificial} and \cite{prasad2017reinforcement}.

The final processed data contains 12,987 tuples $\braces{\paran{X_{i,t}, A_{i,t}, R_{i,t}, X_{i,t+1}}}$ with $i\in [1000]$ and $ t \in [T_i]$. 
We first evaluate the value of a random policy $\pi(a|\bx) = 1/9$ for $a\in[9]$ on the data set using ACPE and MVPE. 
Note both {\PE} and MVPV try to minimize the Bellman residual, which are the sample counterpart of the left hand side of \eqref{eqn:q-bellman-cons-2}, while {\PE} accounts for heterogeneity over different trajectories. 
We fit linear $Q$ functions and plot the histograms of the sample Bellman residuals of MVPE and {\PE}, respectively, in (a) and (b) of Figure~\ref{fig:res-plots}. 
As demonstrated in Figure~\ref{fig:res-plots}~(a), we see that the Bellman residual plot of MVPE is bi-modal.
This hints that there are at least two different coefficient groups in this population.
Figure~\ref{fig:res-plots}~(b) is the residual plot of ACPE.
We can see that the residual plot for our algorithm is well-shaped and performs better in terms of fitting error.
We perform a principal component analysis on our coefficient matrix (averaged across actions) as shown in Figure~\ref{fig:coef_pca}.
We can observe that there are two clear separate clusters within this coefficient matrix, which is a confirmation of the bi-modality of the residuals and a sign that simple $k$-means clustering would work well when we cluster the $\bbeta$'s into groups.

We further use {\PI} to obtain optimal policies $\pi(\hat \balpha^{(1) *}_{ACPI})$ and $\pi(\hat \balpha^{(2) *}_{ACPI})$, respectively for estimated group 1 and 2.  
We follow the same protocol to obtain a optimal policy  $\pi(\hat\balpha^{*}_{MVPI})$ via standard mean-value-based policy iteration (MVPI).
Table~\ref{tab:real_values} summarizes the estimated values of $\pi(\hat \balpha^{(1) *}_{ACPI})$, $\pi(\hat \balpha^{(2) *}_{ACPI})$, and $\pi(\hat\balpha^{*}_{MVPI})$ with different values of discount factor $\gamma$.
These policies are evaluated on two groups separately on the two sub-groups returned by {\PI}. 
We can see that the policy obtained by {\PI} outperforms the one returned by MVPI by around 30~\% on both group 1 and group 2. 

We then demonstrate two sample patients in Table~\ref{tab:sample_patients1} and ~\ref{tab:sample_patients2} to show the effectiveness of estimated policies.
Patient 1 eventually died in hospital, and this policy is evaluated at two time points before his/her death.
At this time point, many of his/her indices are extremely abnormal.
For example, his/her Glasgow Coma Scale (GCS) decreased sharply from 15 to 6.
This is a critical sign of urgent treatment.
Indeed, policy trained under heterogeneous estimation ACPI recommends high intravenous volume and high/medium vasopressors, while policy trained under homogeneous estimation MVPI still puts significant weights on medium intravenous volume and medium vasopressors. 
This conservative policy recommended by homogeneous estimation would be dangerous.
Patient 2 survives and all his/her indices are fairly normal at the evaluated time point, except for his/her relatively high respiratory rate (RR).
He/she is also one of the youngest patients in the sample.
Thus, the physician decides to inject little to no amounts of both intravenous and vasopressors. 
The heterogeneously estimated policy by ACPI recommends the same treatment as the physician does.
However, the homogeneously estimated policy by MVPI also suggests a medium volume of intravenous, which is unnecessary for such a patient.

\begin{figure}[ht!]
  \centering
  \begin{subfigure}[b]{0.45\textwidth}
  \includegraphics[width=1\textwidth]{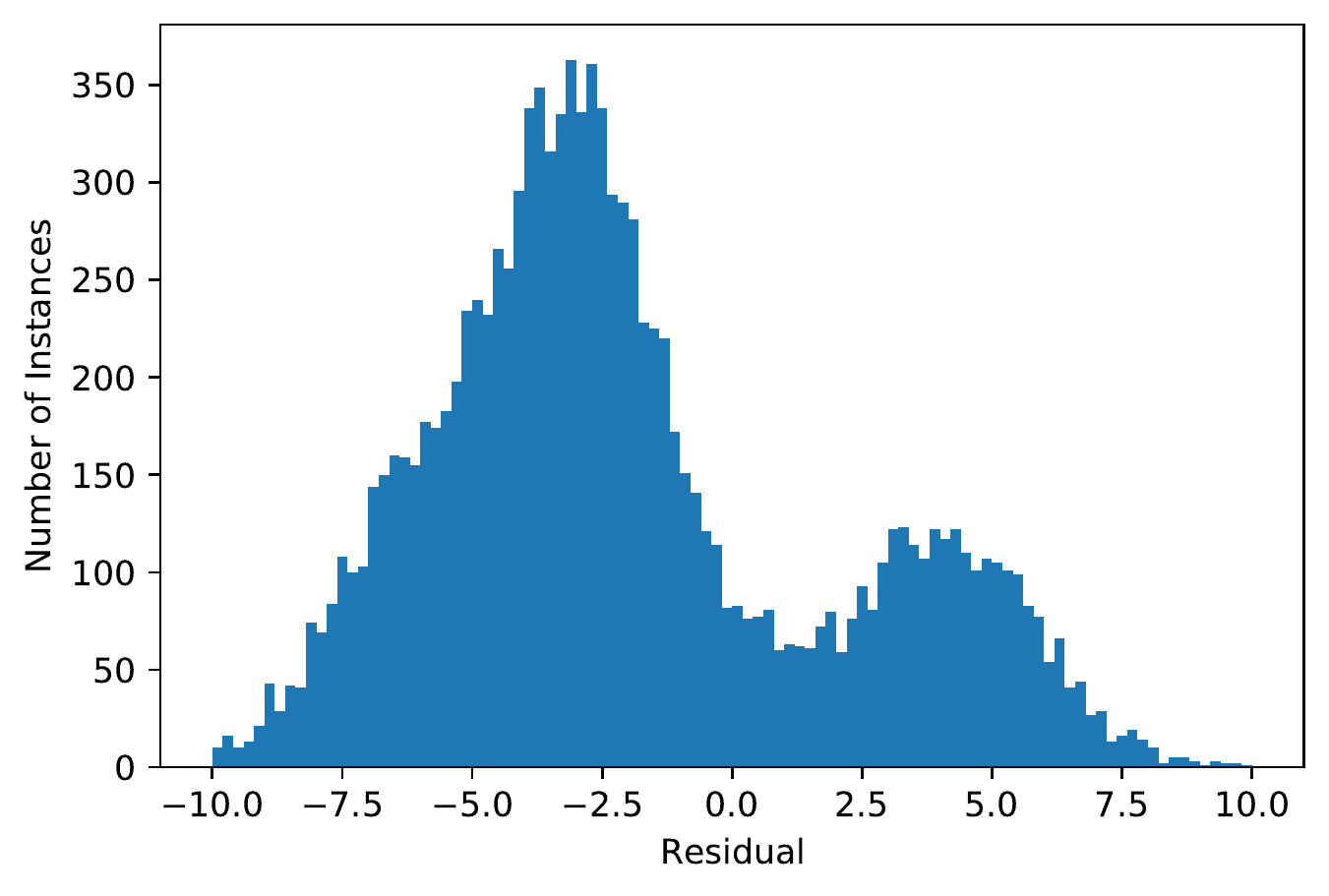}
  \caption{Residual distribution of simple linear regression is clearly bi-model, with two peaks at around -4 and 4.}
  %\label{fig:res_ls}
  \end{subfigure}
    \begin{subfigure}[b]{0.45\textwidth}
  \includegraphics[width=1\textwidth]{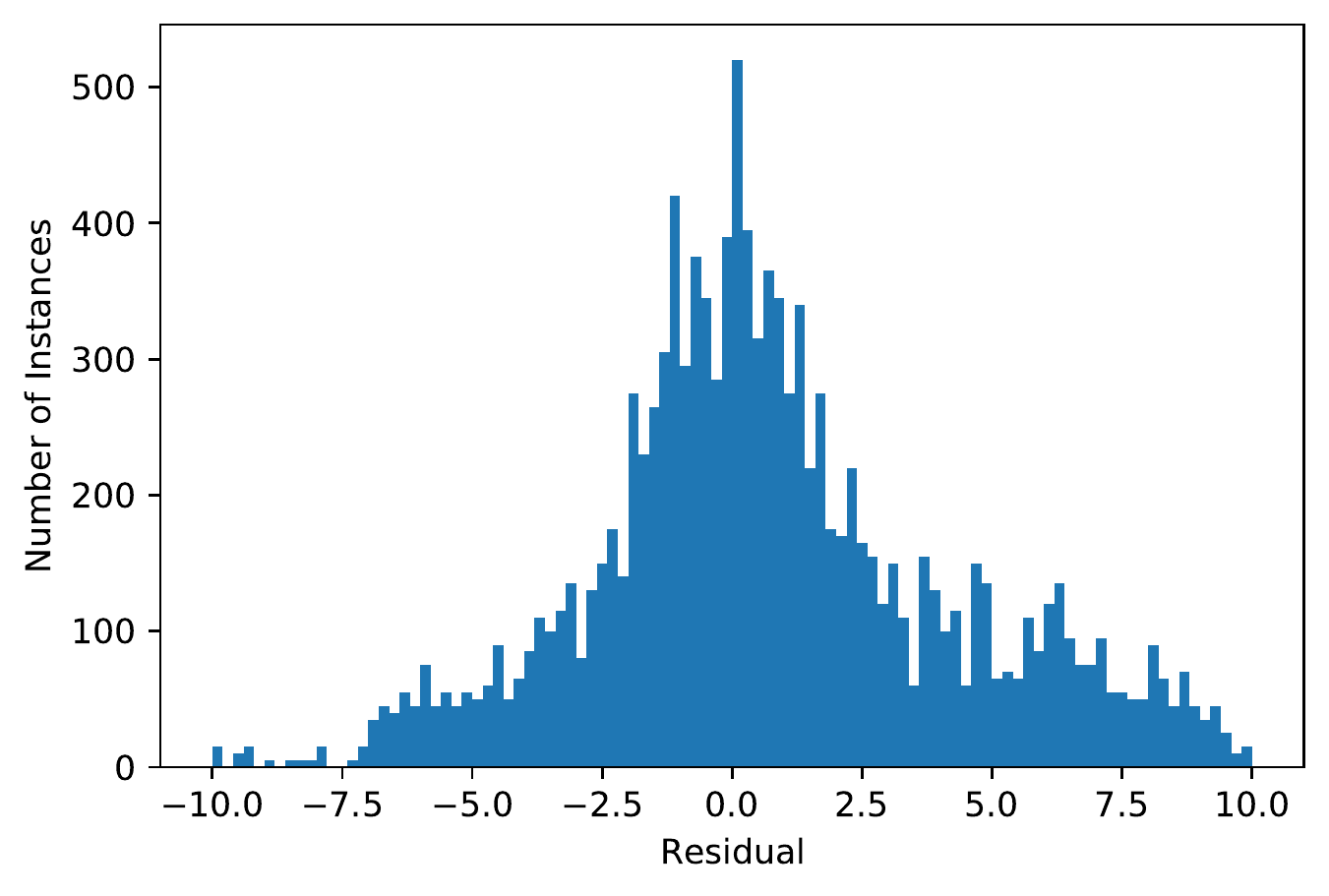}
  \caption{Residual of our algorithm is bell-shaped, with one peak at around 0. It also performs better in terms of fitting error.}
  %\label{fig:res_hetero}
  \end{subfigure}
  \caption{Bellman residuals of MVPE and ACPE of the random policy on the MIMIC-III data set.}
  \label{fig:res-plots}
\end{figure}

\begin{figure}[ht!]
    \centering
    %\begin{subfigure}[b]{0.6\textwidth}
    \includegraphics[width=.6\textwidth]{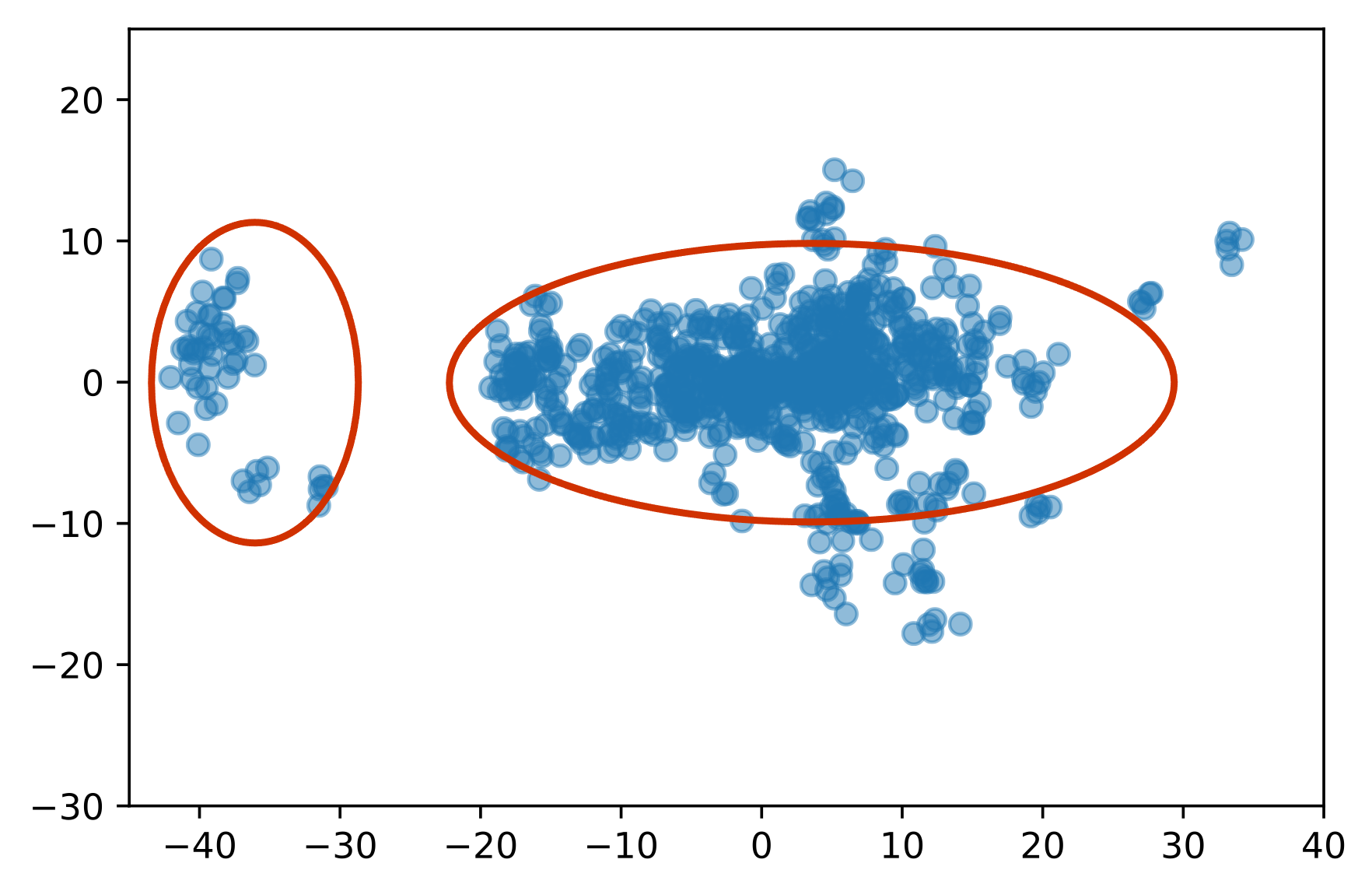}
    \caption{Plot of the first two principal components of the coefficient matrix (averaged across actions) produced by ACPE of the random policy. 
        We can see there is a clear separation between two groups.}
    \label{fig:coef_pca}
   %\end{subfigure}
\end{figure}

\begin{table}[ht!]
\centering
\begin{tabular}{|c|c|c|c|c|}
\hline
\multirow{2}{*}{Estimated optimal policy} &\multicolumn{2}{c|}{Estimated Group 1} & \multicolumn{2}{c|}{Estimated Group 2} \\
\cline{2-5}
 & $\gamma = 0.5$ & $\gamma = 0.7$ & $\gamma = 0.5$ & $\gamma = 0.7$  \\
\hline
$\pi\paran{\hat \balpha^{(1) *}_{ACPI}}$ & -2.2140 & -3.1851 & - & - \\
\hdashline
$\pi\paran{\hat \balpha^{(2) *}_{ACPI}}$ & - & - & -1.4433 & -1.7793 \\
\hline
$\pi\paran{\hat \balpha^{*}_{MVPI}}$   & -3.0964 & -4.5735 & -2.0634 & -2.6487 \\
\hline
ACPI Improvement  & 28.50\% & 30.36\% & 30.05\% & 32.82\% \\
\hline
\end{tabular}
\caption{Estimated values on two subgroups of the heterogeneously estimated policies by ACPI and the homogeneously estimated policy by MVPI. 
The groups are estimated by ACPI. }
\label{tab:real_values}
\end{table}

\begin{table}[ht!]
\centering
\begin{tabular}{|c|c|c|c|}
\hline
Optimal policy & ACPI & MVPI & Expert \\
\hline
iv low, vaso low & $< 0.0001$ & $<0.0001$ & 0\\
\hline
iv low, vaso med & $< 0.0001$ & 0.071 & 0\\
\hline
iv low, vaso high &  $< 0.0001$ & <0.0001& 0\\
\hline
iv med, vaso low & $<0.0001$ & $0.10$ & 0 \\
\hline
iv med, vaso med & $<0.0001$  & 0.14 & 0 \\
\hline
iv med, vaso high & $0.042$ & 0.21 & 0\\
\hline
iv high, vaso low & $0.23$ & $<0.0001$ & 0\\
\hline
iv high, vaso med & $0.36$ &  0.26 & 0\\
\hline
iv high, vaso high & 0.37 & 0.22  & 1\\
\hline
\end{tabular}
\caption{Recommended action by policy estimated by ACPI, MVPI and the expert for a sample patient from group 1.}
\label{tab:sample_patients1}
\end{table}

\begin{table}[ht!]
\centering
\begin{tabular}{|c|c|c|c|}
\hline
Optimal policy & ACPI & MVPI & Expert \\
\hline
iv low, vaso low & $0.72 $ & $0.36$ & 1\\
\hline
iv low, vaso med & $0.18 $ & 0.018 & 0\\
\hline
iv low, vaso high &  $< 0.0001$ & <0.0001 & 0\\
\hline
iv med, vaso low & $<0.0001$ & $0.43$ & 0 \\
\hline
iv med, vaso med & $<0.0001$  & <0.0001& 0 \\
\hline
iv med, vaso high & $<0.0001$ & 0.11 & 0\\
\hline
iv high, vaso low & $0.09$ & $<0.0001$ & 0\\
\hline
iv high, vaso med & $<0.001$ &  0.075 & 0\\
\hline
iv high, vaso high & <0.001 & <0.0001  & 0\\
\hline
\end{tabular}
\caption{Recommended action by policy estimated by ACPI, MVPI and the expert for a sample patient from group 2.}
\label{tab:sample_patients2}
\end{table}

%%
%% Examples of figures, algorithms and codes inputs
%%
% \input{exps/99-figure.tex}
% \input{exps/99-algo.tex}
% \input{exps/99-codes.tex}

\section{Summary} \label{sec:summ}

%Reinforcement learning (RL) has achieved tremendous theoretical and technical achievements in  recent years, leading its increasing applicability to real-life problems in economics, health care, autonomous driving, financial and business management.
Classical RL methods model and optimize the (action-)value function defined as the expectation of total return.
Direct applications of such mean-value based RL to large-scale datasets may generate misleading results because of data heterogeneity.
In this paper, we go beyond classical mean-value based RL to allow for heterogeneity which is characterized by different values across sub-populations.
We proposed {\PE} and {\PI} for both the policy evaluation and control.
We establish convergence rates and construct confidence intervals (CIs) for the estimators obtained by the {\PE} and {\PI}.
Our theoretical findings are validated on synthetic and real datasets.
Particularly, the experiments on the standard MIMIC-III dataset shows evidence of data heterogeneity and confirms the advantage of our new method.

Statistical analysis of policy evaluation and optimal control under the framework of RL and MDP have great potential to facilitate dynamic decision making in a variety of real applications.
We have demonstrated the importance of data heterogeneity when combining RL and large-scale datasets.
There are several interesting directions for future research in this area.
First, our method is based on the Bellman consistency equation to approximate the $Q^\pi$ function.
Its performance is guaranteed under the assumption of bounded state distribution shift, which is caused by the discrepancy between the behavior policy and the target policy.
However, this assumption may not hold generally in real applications.
It is of great interest to investigate various ways to relax such an assumption.
Second, the theoretical results in this paper are developed for offline batch estimation.
Developing online RL estimation with heterogeneity is an interesting topic for future research.
Lastly, we estimate the optimal policy by a variation of $Q$ learning.
It would also be worthwhile to investigate the ways to incorporate heterogeneity in other RL methods such as policy gradient.

%-------
%
%  bibliography
%

\clearpage
%\nocite{*} % include all bib's
\bibliographystyle{\mybibsty}
\bibliography{\mybib}

%-------
%
%  appendix
%
\clearpage
\setcounter{page}{1}
\begin{appendices}

\begin{center}
    {\Large Supplementary materials for ``Heterogeneous Reinforcement Learning with Offline Data: Estimation and Inference''}

    {Elynn Y. Chen, Rui Song, and Michael I. Jordan}
\end{center}

%\input{99-proof-A.tex}

%!TEX root = 0-main.tex

\section{Proofs: Oracle Estimators}  \label{appen:oracle}

The solution of \eqref{eqn:Q-beta-ora} can be obtained by solving the estimating equation:
\begin{equation} \label{eqn:empirical-scores-ora}
    \frac{1}{NT} \sum_{i=1}^{N} \sum_{t=1}^{T} \paran{(\tilde{\bLam}^i)^\top Z_{i,t} R_{i,t} - (\tilde{\bLam}^i)^\top Z_{i,t}\paran{ Z_{i,t} - \gamma U^{\pi}_{i,t+1}}^\top \tilde{\bLam}^i \btheta^{\pi} } = 0.
\end{equation}
The resulting oracle estimator of the group coefficients obtained by \eqref{eqn:empirical-scores-ora} has the following decomposition:
\begin{equation} \label{eqn:theta-decompose}
    \paran{ \tilde\btheta^{\pi}  - \mathring\btheta^{\pi}  }
    %=  \tilde\bSigma^{-1} \tilde\bG\paran{\mathring\btheta ^\pi}
    =  (\tilde\bSigma^{\pi})^{-1}   \frac{1}{N T} \sum_{i=1}^{N} \sum_{t=1}^{T}  \paran{(\tilde{\bLam}^i)^\top Z_{i,t} (\eps_{1,i,t}^{\pi}  + \eps_{2,i,t}^{\pi} )}
    = (\tilde\bSigma^{\pi})^{-1} \tilde\bzeta_1^{\pi} + (\tilde\bSigma^{\pi})^{-1} \tilde\bzeta_2^{\pi},
\end{equation}
where
\begin{align}
    \tilde\bSigma^{\pi}
    & = \frac{1}{N T} \sum_{i=1}^{N} \sum_{t=1}^{T} (\tilde{\bLam}^i)^\top  Z_{i,t}\paran{ Z_{i,t} - \gamma U^{\pi}_{i,t+1}}^\top \tilde{\bLam}^i,\\
    \tilde\bzeta_1^{\pi}
    & = \frac{1}{N T} \sum_{i=1}^{N} \sum_{t=1}^{T}  (\tilde{\bLam}^i)^\top Z_{i,t} \eps_{1,i,t}, \\
    \tilde\bzeta_2^{\pi}
    & = \sum_{i=1}^{N} \sum_{t=1}^{T}  (\tilde{\bLam}^i)^\top Z_{i,t} \eps_{2,i,t}, \\
    \eps_{1,i,t}^{\pi} 
    & = R_{i,t} + \gamma\cdot\sum_{a\in\calA} Q(\pi, X_{i,t+1}, a) \pi(a \mid X_{i,t+1}) -  Q(\pi, X_{i,t}, A_{i,t})   \label{eqn:eps-1}\\
    \eps_{2,i,t}^{\pi} 
    & = \gamma\cdot\sum_{a\in\calA} e(\pi, X_{i,t+1}, a) \pi(a \mid X_{i,t+1}) - \sum_{a\in\calA} e(\pi, X_{i,t}, a) \bbone(A_{i,t}=a),  \label{eqn:eps-2} \\
    e(\pi, X_{i,t}, a)
    & = Q(\pi, X_{i,t}, a) - \bphi(X_{i,t})^\top \tilde{\bLam}^i \mathring\btheta^{\pi}. % \mathring\bbeta^{\pi,i}_{a},
\end{align}
%and $\mathring\bbeta^\pi = [ (\mathring\bbeta^{\pi,1}_{1})^\top,
%\cdots,
%(\mathring\bbeta^{\pi,1}_{M})^\top,
%\cdots,
%(\mathring\bbeta^{\pi,N}_{1})^\top,
%\cdots,
%(\mathring\bbeta^{\pi,N}_{M})^\top ]^\top
%= (\bW\otimes\bI_{JM}) \mathring\btheta^\pi$.
In the proofs, we will omit the superscript $\pi$ in $\tilde\bSigma^{\pi}$, $\tilde\bzeta_1^{\pi}$, $ \tilde\bzeta_2^{\pi}$, $\eps_{1,i,t}^{\pi}$, $\eps_{2,i,t}^{\pi}$, $\tilde\bSigma^{\pi, (k)}$, $\tilde\bzeta_1^{\pi, (k)}$, and $\tilde\bzeta_2^{\pi, (k)}$ for brevity. 

\begin{proposition}[Oracle estimator $\ell_2$ convergence]        \label{thm:beta-l2-convg}
    Suppose Assumption \ref{assum:r-q-func} -- \ref{assum:geo-ergodic} hold.
    $N_{min} = \underset{1\le k \le K}{\min} N_k$ and $N_{max} = \underset{1\le k \le K}{\max} N_k$.
    If $K=\smlo{N_{\min}T}$,
    $J \ll \sqrt{N_{\min} T} / \log(N_{\min} T)$,
    $J^{\kappa / p} \gg \sqrt{N_{\max} T}$,
    as either $N_{\min} \rightarrow \infty$ or $T \rightarrow \infty$,
    \begin{equation*}
        \norm{ \tilde\btheta^{\pi}  - \mathring\btheta^{\pi}  }_2
        =  \Op{ J^{1/2} (T N_{\min})^{-1/2} }  +  \Op{J^{-\kappa / p}}.
    \end{equation*}
    If we have balanced groups, that is, $N_{\min} \asymp N_{\max} \asymp N$, then
    \begin{equation*}
        \norm{ \tilde\btheta^{\pi}  - \mathring\btheta^{\pi}  }_2
        =  \Op{ J^{1/2} (T N)^{-1/2} }  +  \Op{J^{-\kappa / p}}.
    \end{equation*}
\end{proposition}

\begin{proof}
    Recall that  $\tilde{\bLam}^i = \bLam^i (\bW \otimes \bI_{JM})$.
    Let $P^{(k)} = N^{(k)} / N$, $\tilde\bSigma^{(k)}$ and $\tilde\bzeta_{i}^{(k)}$ are defined in \eqref{eqn:tilde-Sigma-k} and \eqref{eqn:tilde-zeta-k}, respectively.
    We have
    \begin{equation} \label{eqn:hetero-homo-relation-mat}
         \tilde\bSigma = \bdiag\paran{ P^{(k)} \tilde\bSigma^{(k)} }_{1\le k\le K}, \quad\text{and}\quad
         \tilde\bzeta_{i}  = \begin{bmatrix}
             P^{(1)} (\tilde\bzeta_{i}^{(1)})^\top & \cdots & P^{(K)} (\tilde\bzeta_{i}^{(K)})^\top
         \end{bmatrix}^\top, \quad\text{for } i = 1,2.
    \end{equation}
    Then, by decomposition \eqref{eqn:theta-decompose} and Proposition \ref{thm:subhomo-l2}, we have
    \begin{equation*}
        \norm{ \tilde\btheta^{\pi}  - \mathring\btheta^{\pi}  }_2
        = \paran{\sum_{k=1}^K \norm{ (\tilde\bSigma^{(k)})^{-1} \tilde\bzeta_1^{(k)} + (\tilde\bSigma^{(k)})^{-1} \tilde\bzeta_2^{(k)} }_2^2}^{1/2}
        = \Op{ \sqrt{\frac{KJ}{TN_{min}}} }  +  \Op{K J^{-\kappa / p}}.
    \end{equation*}
    The desired result follows by noticing that $K$ is fixed and $P^{(k)}$ is bounded away from zero under our assumption.
\end{proof}

\begin{theorem}  [Asymptotic normality of oracle coefficient estimator]   \label{eqn:ora-asymp-normal}
    Suppose Assumption \ref{assum:r-q-func} -- \ref{assum:geo-ergodic} hold.
    Let $N_{min} = \underset{1\le k \le K}{\min} N_k$ and $N_{max} = \underset{1\le k \le K}{\max} N_k$.
    If $K=\smlo{N_{\min} T}$,
    $J = \smlo{ \sqrt{N_{\min} T} / \log(N_{\min} T) }$.
    % $J^{-\kappa / p} \ll (N_{\min} T)^{-1/2}$.
    For any $\bnu \in \RR^{JMK}$ satisfying $\norm{\bnu}_2 = \bigO{1}$,  $\norm{\bnu}_2^{-1} = \bigO{1}$ and $J^{-\kappa / p} \ll \paran{N_{\min} T \paran{1 + \norm{\bnu}_2^{-2}}}^{-1/2}$, we have, as either $N_{\min}\rightarrow\infty$ or $T\rightarrow \infty$,
    \begin{equation*}
        \sqrt{NT} \cdot
        \tilde \sigma_{\theta} \paran{\pi, \bnu}^{-1} \cdot
        \bnu^\top \paran{ \tilde\btheta^{\pi} - \mathring\btheta^{\pi} }
        \convdist \calN\paran{0, 1},
    \end{equation*}
    where $\tilde\sigma_{\theta}\paran{\pi, \bnu} = \sqrt{ \bnu^\top (\tilde\bSigma^{\pi})^{-1} \tilde\bOmega^\pi (\tilde\bSigma^{\pi \top})^{-1} \bnu }$,
    \begin{align}
        \tilde\bSigma^\pi
        & = \frac{1}{NT} \sum_{i=1}^{N} \sum_{t=1}^{T} (\tilde\bLam^{i})^\top  Z_{i,t}\paran{ Z_{i,t} - \gamma U_{i,t+1}^\pi}^\top \tilde \bLam^{i},  \label{eqn:tilde-sigma}\\
        \tilde\bOmega^\pi
        & = \frac{1}{NT}  \sum_{i=1}^{N} \sum_{t=1}^{T} (\tilde\bLam^{i})^\top Z_{i,t} Z_{i,t}^\top \tilde\bLam^{i} \paran{ R_{i,t} - \paran{Z_{i,t} - \gamma U_{i,t+1}^\pi}^\top \tilde\bLam^{i} \tilde\btheta^{\pi} }^2. \label{eqn:tilde-omega}
    \end{align}
\end{theorem}
%\noindent\textbf{Proof of Theorem \ref{eqn:ora-asymp-normal}}
\begin{proof}
    Suppose that $J^{- \kappa / p} \ll \paran{ N_{\min} T \paran{1 + \norm{\bnu}_2^{-2}} }^{-1/2}$.
    Combining Proposition \ref{thm:beta-l2-convg} and the fact that $J^{- \kappa / p} \ll \paran{ N_{\min} T }^{-1/2}$,
    we have, for any $\bnu = \brackets{\bnu^{(1)}, \cdots, \bnu^{(k)}} \in \RR^{JMK}$,
    $\bnu^{(k)}\in\RR^{JM}$ for $k\in[K]$, that 
    \begin{equation*}
        \begin{aligned}
            \sqrt{NT} \cdot \bnu^\top \paran{ \tilde\btheta^{\pi}  - \mathring\btheta^{\pi}  }
            & = \sqrt{NT} \cdot \bnu^\top \tilde\bSigma^{-1} \tilde\bzeta_1 + \op{1} \\
            & = \sum_{k\in[K]} \sqrt{N_k T / P^{(k)}} \cdot (\bnu^{(k)})^\top (\tilde\bSigma^{(k)})^{-1} \tilde\bzeta_1^{(k)}
            + \op{1} \\
            & \convdist \calN\paran{0, \sum_{k\in[K]}  (\bnu^{(k)})^\top (\tilde\bSigma^{(k)})^{-1} \tilde\bOmega^{(k)} (\tilde{\bSigma}^{(k)^\top})^{-1} \bnu^{(k)} / P^{(k)}},
        \end{aligned}
    \end{equation*}
    where the second equality follows from \eqref{eqn:hetero-homo-relation-mat}
    and the distribution is obtained by applying Theorem \ref{thm:subhomo-distn}.
    Note that $\tilde\bOmega^{\pi} =  \bdiag\paran{P^{(k)} \tilde\bOmega^{\pi, (k)}}$
    where $\tilde\bOmega^{\pi,(k)}$ is defined in \eqref{eqn:hat-Omega-k}.
    Combining this with \eqref{eqn:hetero-homo-relation-mat},
    we have for any $\bnu \in \RR^{JMK}$ satisfying $\norm{\bnu}_2 = \bigO{1}$ and $\norm{\bnu}_2^{-1} = \bigO{1}$, 
    \begin{equation*}
        \frac{ \sqrt{NT} \bnu^\top \paran{ \tilde\btheta^{\pi}  - \mathring\btheta^{\pi}  }}{ \sqrt{\bnu^\top \tilde\bSigma^{-1} \tilde\bOmega (\tilde\bSigma^\top)^{-1} \bnu} }
        \convdist \calN\paran{0, 1}.
    \end{equation*}

Under the condition in Theorem \ref{eqn:ora-asymp-normal}, we can show that $\tilde\sigma_{\theta}\paran{\pi, \bnu}$ converges almost surely to $\sigma_{\theta}\paran{\pi, \bnu} = \sqrt{ \bnu^\top (\bSigma^{\pi})^{-1} \bOmega^{\pi} (\bSigma^{\pi \, \top})^{-1} \bnu }$, where $\bOmega^\pi=\EE[\tilde\bOmega^\pi]$ and $\bSigma^\pi=\EE[\tilde\bSigma^\pi]$.
\end{proof}

Applying Theorem \ref{eqn:ora-asymp-normal} to the value function estimators defined in \eqref{eqn:Q-V-group-est}, we establish the asymptotic distribution for state and action value functions $V(\pi, \bx)$ and $Q(\pi, \bx, a)$ and the asymptotic distribution for integrated value $V_\calR(\pi)$ for a given distribution $\calR$ in the next two corollaries.

\begin{corollary}  [Asymptotic normality of oracle value estimator.]   \label{thm:val-asymp-normal}
    Suppose Assumption \ref{assum:r-q-func} -- \ref{assum:geo-ergodic} hold.
    Let $N_{min} = \underset{1\le k \le K}{\min} N_k$ and $N_{max} = \underset{1\le k \le K}{\max} N_k$.
    If $K=\smlo{N_{\min} T}$,
    $J = \smlo{\sqrt{N_{\min} T} / \log(N_{\min} T)}$,
    and $J^{-\kappa / p} \ll \paran{N_{\max} T \paran{1 + \norm{\bu(\pi, \bx)}_2^{-2}}}^{-1/2}$,
    we have, for any $\bx\in\calX$ and $a\in\calA$, that as either $N_{\min}\rightarrow\infty$ or $T\rightarrow \infty$,
    \begin{align*}
        \sqrt{NT} \, \tilde\sigma_{v}^{(k)}\paran{\pi, \bx}^{-1} \paran{ \tilde V^{(k)}\paran{\pi,\bx} - V^{(k)}\paran{\pi, \bx} }
        \convdist \calN\paran{0, 1}, \\
        \sqrt{NT} \, \tilde\sigma_{q}^{(k)}\paran{\pi, \bx, a}^{-1} \paran{ \tilde Q^{(k)}\paran{\pi,\bx, a} - Q^{(k)}\paran{\pi,\bx,a} }
        \convdist \calN\paran{0, 1},
    \end{align*}
    for any $k\in[K]$, where
    \begin{align*}
        \tilde\sigma_{v}^{(k)}\paran{\pi, \bx} & = \sqrt{ \bu\paran{\pi, \bx}^\top
            (\tilde\bSigma^{\pi,(k)})^{-1} \tilde\bOmega^{\pi,(k)} (\tilde\bSigma^{\pi,(k) \top})^{-1}
            \bu\paran{\pi, \bx} }, \quad \\
        \tilde\sigma_{q}^{(k)}\paran{\pi, \bx, a} &= \sqrt{ \bz\paran{\bx, a}^\top
           (\tilde\bSigma^{\pi,(k)})^{-1} \tilde\bOmega^{\pi,(k)} (\tilde\bSigma^{\pi,(k) \top})^{-1}
            \bz\paran{\bx, a} },
    \end{align*}
    where the matrices $\tilde\bSigma^{\pi,(k)}$ and $\tilde\bOmega^{\pi,(k)}$ are defined as
    \begin{align}
        \tilde\bSigma^{\pi,(k)}
        & = \frac{1}{NT} \sum_{i=1}^{N} \sum_{t=1}^{T}
        (\tilde\bLam^{i,(k)})^\top
        Z_{i,t}\paran{ Z_{i,t} - \gamma U_{i,t+1}^\pi}^\top
        \tilde \bLam^{i,(k)},  \label{eqn:tilde-sigma-k}\\
        \tilde\bOmega^{\pi,(k)}
        & = \frac{1}{NT}  \sum_{i=1}^{N} \sum_{t=1}^{T} (\tilde\bLam^{i,(k)})^\top Z_{i,t} Z_{i,t}^\top \tilde\bLam^{i,(k)} \paran{ R_{i,t} - \paran{Z_{i,t} - \gamma U_{i,t+1}^\pi}^\top \tilde\bLam^{i,(k)} \tilde\btheta^{\pi} }^2, \label{eqn:tilde-omega-k}
    \end{align}
    and $\tilde\bLam^{i,(k)} = \bLam^i(\bw_{\cdot k} \otimes \bI_{JM})$.
\end{corollary}
%\noindent\textbf{Proof of Corollary \ref{thm:val-asymp-normal}.}
\begin{proof}
   The proof uses the result in Corollary \ref{thm:subhomo-distn-value} and is similar to that of Theorem \ref{eqn:ora-asymp-normal}.
\end{proof}

\begin{corollary}[Asymptotic normality of oracle integrated value estimator] \label{thm:int-val-asymp-normal}
    Suppose Assumption \ref{assum:r-q-func} -- \ref{assum:geo-ergodic} hold.
    Let $N_{min} = \underset{1\le k \le K}{\min} N_k$ and $N_{max} = \underset{1\le k \le K}{\max} N_k$.
    If $K=\smlo{N_{\min} T}$, 
    $J = \smlo{\sqrt{N^{(k)} T} / \log(N^{(k)} T)}$, 
    $J^{-\kappa / p} \ll \paran{N_{\max} T \paran{1 + \norm{\int\bphi_J(\bx)\calR(d\bx)}_2^{-2}}}^{-1/2}$,
    as either $N_{\min} \rightarrow \infty$ or $T\rightarrow \infty$, we have for any $\bx\in\calX$, % and uniformly over $\pi\in\Pi$,
    \begin{equation*}
        \begin{aligned}
            \sqrt{N T} \cdot \tilde\sigma_{\calR}^{(k)}(\pi)^{-1}
            \paran{ \tilde V_{\calR}^{(k)}\paran{\pi} -  V_{\calR}^{(k)}\paran{\pi} }
            & \convdist \calN\paran{0, 1} \, ,
        \end{aligned}
    \end{equation*}
    for any $k\in[K]$, where
    $$\tilde\sigma_{\calR}^{(k)}\paran{\pi}^2
    = \paran{\int\bu\paran{\pi, \bx}\calR(d\bx)}^\top
    (\tilde\bSigma^{\pi,(k)})^{-1} \tilde \bOmega^{\pi,(k)} (\tilde\bSigma^{\pi,(k) \top})^{-1} \paran{\int\bu\paran{\pi, \bx}\calR(d\bx)},$$
    $\tilde\bSigma^{\pi,(k)}$ and $\tilde\bOmega^{\pi,(k)}$ are given in \eqref{eqn:tilde-sigma-k} and \eqref{eqn:tilde-omega-k}.
\end{corollary}
%\noindent\textbf{Proof of Corollary \ref{thm:int-val-asymp-normal}.}
\begin{proof}
    The proof uses the result in Corollary \ref{thm:subhomo-distn-value-intg} and is similar to that of Theorem \ref{eqn:ora-asymp-normal}.
\end{proof}

\begin{comment}
\begin{corollary}
    Under the conditions in Theorem \ref{eqn:ora-asymp-normal} and \ref{thm:feasible-in-prob},
    for any $\bnu \in \RR^{JMN}$ satisfying $J^{\kappa / p} \gg \sqrt{N_{\max} T \paran{1 + \norm{\bnu}_2^{-2}}}$, we have as either $N_{\min}\rightarrow\infty$ or $T\rightarrow \infty$,
    \begin{equation*}
        \sqrt{NT} \; \hat\sigma_{\beta_i}^{-1} \bnu^\top \paran{ \hat\bbeta^{\pi}_i- \mathring\bbeta^{\pi}_i }
        \convdist \calN\paran{0, 1},
    \end{equation*}
    where
    \begin{align*}
        \hat \sigma_{\beta_i}
        & = \bnu^\top (\hat\bw_{i\cdot}^\top \otimes\bI_{JM})  \hat\bSigma^{-1} \hat\bOmega (\hat\bSigma^\top)^{-1} (\hat\bw_{i\cdot}\otimes\bI_{JM})  \bnu \\
        \hat\bSigma
        & = \frac{1}{NT} \sum_{i=1}^{N} \sum_{t=1}^{T} \hat\bLam_i^\top  Z_{i,t}\paran{ Z_{i,t} - \gamma U_{i,t+1}^\pi}^\top \hat \bLam_i,\\
        \hat\bOmega
        & = \frac{1}{NT} \sum_{i=1}^{N} \sum_{t=1}^{T} \hat\bLam_i^\top Z_{i,t} Z_{i,t}^\top \hat\bLam_i \paran{ R_{i,t} - \paran{Z_{i,t} - \gamma U_{i,t+1}^\pi}^\top \hat\bLam_i \hat\btheta^{\pi} }^2.
    \end{align*}
\end{corollary}
\end{comment}

\subsection{Proof of Theorem \ref{eqn:ora-uniform-conv}}

\begin{proof}
    By decomposition \eqref{eqn:theta-decompose}, we have
    \begin{equation*}
        \norm{ \tilde\btheta^{\pi}  - \mathring\btheta^{\pi}  }_\infty \le \norm{\paran{ \tilde\bSigma }^{-1}}_\infty \norm{\tilde\bzeta_1}_\infty + \norm{\paran{ \tilde\bSigma }^{-1}}_\infty \norm{\tilde\bzeta_2}_\infty
    \end{equation*}
    We bound each term on the right hand side as follows.

    Using \eqref{eqn:hetero-homo-relation-mat} and Lemma \ref{thm:Sigma-k} (iii), we have
    \begin{equation*}
        \norm{(\tilde\bSigma)^{-1}}_\infty
        \le \underset{k}{\max} \norm{(P^{(k)} \tilde\bSigma^{(k)})^{-1}}_\infty
        = 6 C^{-1} \sqrt{J} (N/N_{\min}),
    \end{equation*}
    with probability at least $1 - \bigO{(N_{\min} T)^{-2}}$.

    By Lemma \ref{thm:zeta12} (iii), we have
    \begin{align*}
        \Pr\paran{ \norm{\tilde\bzeta_1^{(k)}}_\infty > c \sqrt{2 \log(N_kT) / (N_kT)}} \le 2 J M (N_k T)^{- 2},
    \end{align*}
    for some positive constant $c$.
    Lemma \ref{thm:zeta12}  (iv) shows that
    $\norm{\tilde\bzeta_2^{(k)}}_{\infty} \le c_1 J^{-\kappa/p} \ll c_1 \paran{N_k T}^{-1/2}$ almost surely.
    By a union bound and Lemma \ref{thm:zeta12} (iii), we have
    \begin{align*}
            \Pr\paran{ \norm{\tilde\bzeta_1}_\infty > c (N_{\max} / N) \sqrt{2\log(N_{\min}T)/N_{\min}T} }
            & \le \sum_{k=1}^{K} \Pr\paran{\norm{P^{(k)} \tilde\bzeta_1^{(k)}}_\infty > (N_{\max} / N) \sqrt{2\log(N_{\min} T)/N_{\min} T}}  \\
            & \le \sum_{k=1}^{K} \Pr\paran{\norm{P^{(k)} \tilde\bzeta_1^{(k)}}_\infty > P^{(k)} \sqrt{2\log(N_k T)/N_k T}}  \\
            & = \sum_{k=1}^{K} 2JM (N_{k} T)^{-2} \\
            & \le 2JMK (N_{\min} T)^{-2},
    \end{align*}
    for some positive constant $c$.
    Similarly, using Lemma \ref{thm:zeta12}  (iv), we have
    \begin{equation*}
        \norm{\tilde\bzeta_2}_\infty
        =  \underset{1\le k \le K}{\max}\; \norm{P^{(k)} \tilde\bzeta_2^{(k)}}_{\infty}
        \le c_1 J^{-\kappa/p} N_{max}/N
        \ll c_1  N_{max}/N (N_{\max} T)^{-1/2}, \quad\text{almost surely}.
    \end{equation*}
    Thus, by a union bound, we have that
    \begin{equation*}
        \norm{ \tilde\btheta^{\pi}  - \mathring\btheta^{\pi}  }_\infty
        \le 6 c C^{-1} (N_{\max}/N_{\min}) \sqrt{2 J \log(N_{\min}T)/N_{\min}T}
    \end{equation*}
    holds with probability at least $1 - 2JMK (N_{\min} T)^{-2} - \bigO{(N_{\min}T)^{-2}}$.
\end{proof}

%!TEX root = 0-main.tex

\section{Feasible estimator} \label{appen:feasible}

\subsection{Proof of Theorem \ref{thm:feasible-in-prob}}

\begin{proof}
Note that all coefficients are with respect to a given policy $\pi\in\Pi$, thus we drop superscript $\pi$ for brevity.
We use $\bB = \brackets{\bbeta_1, \cdots, \bbeta_N}$ and $\bbeta = \vect{\bB}$  (also $\bTheta = \brackets{\btheta_1, \cdots, \btheta^{(k)}}$ and $\btheta = \vect{\bTheta}$) interchangeably.
It is easy to see that $\bB = \bTheta \bW^\top$, $\bbeta^i= \bTheta \bw_{i\cdot} = \bLam_i \bbeta$, and $\bbeta = (\bW\otimes\bI_p) \btheta$.
Define
\begin{equation*} 
\begin{aligned}
L\paran{\bB} &
= \bG\paran{\pi, \bB}^\top \bG\paran{\pi, \bB}, \quad
P\paran{\bB, \lam}
= \frac{1}{N^2}\sum_{1\le i<j \le N} p\paran{(JM)^{-1/2}\norm{\bbeta^i- \bbeta^j}_2, \lam}, \\
\tilde L\paran{\bTheta}  &
= \tilde\bG(\pi,\bTheta)^\top \tilde\bG \paran{\pi, \bTheta}, \quad
\tilde P\paran{\bTheta, \lam}
= P\paran{\bTheta\bW^\top, \lam},
\end{aligned}
\end{equation*}
where
\begin{equation*}
    \begin{aligned}
        G\paran{\pi, \bB}
        = \frac{1}{NTJ} \sum_{i=1}^{N}  \sum_{t=1}^{T} \bLam_i^\top \bZ_{i,t} R_{i,t}
        - \bLam_i^\top \bZ_{i,t} \paran{ \bZ_{i,t} - \gamma \bU_{\pi,i,t+1}}^\top \bLam_i  \bbeta, \\
        \tilde G\paran{\pi, \bTheta}
        = \frac{1}{NTJ} \sum_{i=1}^{N}  \sum_{t=1}^{T} \tilde\bLam_i^\top \bZ_{i,t} R_{i,t}
        - \tilde\bLam_i^\top \bZ_{i,t} \paran{ \bZ_{i,t} - \gamma \bU_{\pi,i,t+1}}^\top \tilde\bLam_i  \btheta.
    \end{aligned}
\end{equation*}
Then the loss function in the optimization problem \eqref{eqn:Q-ls-beta} and \eqref{eqn:Q-beta-ora} can be rewritten, respectively, as,
\begin{equation}  \label{eqn:loss-func}
\calL\paran{ \bB } := L\paran{\bB}  + P\paran{\bB, \lam}, \text{and}\qquad \tilde\calL\paran{ \bTheta } := \tilde L\paran{\bTheta}  + \tilde P\paran{\bTheta, \lam}.
% \bG\paran{\pi,\bbeta}^\top \bG\paran{\pi,\bbeta} + \frac{1}{N^2}\sum_{1\le i<j \le N} p\paran{(JM)^{-1/2}\norm{\bbeta^i - \bbeta^j}_2, \lam }.
\end{equation}
Let $\calM_{\calG}$ be the subspace of $\RR^{JM\times N}$, defined as
\begin{equation*}
\calM_{\calG} = \braces{\bB \in \RR^{JM \times N}: \bbeta^i= \bbeta^j, \text{ for any } i,j \in \calG_k, 1\le k \le K}.
\end{equation*}
Recall that $\bW$ is the true $N\times K$ \textit{membership matrix}, then for each $\bB\in\calM_{\calG}$, it can be written as $\bB = \bTheta\bW^\top$ for some $\bTheta\in\RR^{JM \times K}$.
Also by matrix calculation, we have $\bW^\top \bW = \diag\braces{N_1, \cdots, N_K}$ where $N_k$ denotes the number of trajectories in $\calG_k$.

Let $\calT: \calM_G \rightarrow \RR^{JM\times K}$ be the mapping such that $\calT(\bB)$ is the $JM \times K$ matrix whose $k$-th column equals to the common value of $\bbeta^i$ for $i \in \calG_k$.
Let $\bar{\calT}: \RR^{JM\times N} \rightarrow  \RR^{JM\times K}$ be the mapping such that $\bar{\calT}(\bB) = \bB\bW(\bW^\top\bW)^{-1}$.
Clearly, when $\bB\in\calM_\calG$, $\calT(\bB) = \bar{\calT}(\bB)$.
By calculation, we have $P\paran{\bB, \lam} = \tilde P\paran{\calT(\bB), \lam} = 0$ for every $\bB \in \calM_{\calG}$ and $P\paran{\calT^{-1}(\bTheta), \lam} = \tilde P\paran{\bTheta, \lam}$ for every $\bTheta \in \RR^{JM\times K}$.

Consider the neighborhood of $\mathring\bB$:
\begin{equation*}
\mathring\calN = \braces{\bB\in\RR^{JM\times N}, \norm{ \vect{\bB - \mathring\bB} }_\infty \le \phi_{NT} }.
\end{equation*}
According to Theorem \ref{eqn:ora-uniform-conv}, there is an event $E_1$ such that on the event $E_1$,
\begin{equation*}
\norm{ \vect{\tilde\bB - \mathring\bB}}_\infty \le \phi_{NT},
\end{equation*}
and $\Pr( E_1^C ) \le 2JMK (N_{min} T)^{-2} + \bigO{(N_{\min}T)^{-2}}$.
Hence, $\tilde\bB \in \mathring\calN$ on the event $E_1$.

%We show that $\tilde\bB$ is a strictly local minimizer of the objective function \eqref{eqn:Q-ls-beta} with probability approaching 1 through the following two steps:
For any $\bB\in\RR^{JM\times N}$, let $\bar{\bTheta} = \bar{\calT}(\bB)$, $\bar{\bB} = \calT^{-1}\paran{\bar{\bTheta}}$.
Lemma \ref{thm:E1-ora-B-min-loss} shows that, on event $E_1$, $\calL\paran{ \tilde\bB } < \calL\paran{\bar{\bB}}$ for any $\bB$ whose $\bar{\bB} \ne \tilde\bB$.
Lemma \ref{thm:E2-B*-min} shows that there is an event $E_2$ such that $\Pr\paran{E_2^C} \le 2 ((NT)^{-1})$.
On $E_1 \cap E_2$, there is a neighborhood of $\tilde\bB$, denote by $\tilde\calN$, such that $\calL\paran{ \bar{\bB} } < \calL\paran{\bB}$ for any $\bB \in \mathring\calN\cap\tilde\calN$ for sufficient large $N$ or $T$.
Therefore, we have $\calL\paran{ \vect{\tilde\bB} } < \calL\paran{\vect{\bB}}$
for any $\bB\in \mathring\calN\cap \tilde\calN$ and $\bB \ne \tilde\bB$,
so that $\tilde\bB$ is a strict local minimizer of $\calL(\vect{\bB})$ given in
\eqref{eqn:Q-ls-beta} on the event $E_1\cap E_2$ with $\Pr\paran{E_1 \cap E_2} \ge 1 - 2 (K + JM + 1)(NT)^{-1}$ for sufficient large $N$ or $T$.
\end{proof}

\begin{lemma}  \label{thm:E1-ora-B-min-loss}
    Suppose the assumptions in Theorem \ref{thm:feasible-in-prob} hold. 
    For any $\bB\in\RR^{JM\times N}$, let $\bar{\bB} = \calT^{-1}\paran{\bar\calT(\bB)}$.
    On event $E_1$, $\calL\paran{ \tilde\bB} < \calL\paran{\bar{\bB}}$ for any $\bB\in\RR^{JM\times N}$ whose $\bar{\bB} \ne \tilde\bB$.
\end{lemma}
\begin{proof}
    By definition \eqref{eqn:loss-func}, we have $\calL(\cdot) = L(\cdot) + P(\cdot)$ and consider each term on the RHS.
    First, we have $L\paran{\tilde\bB} < L\paran{\bar{\bB}}$ by the following argument.
    Since $\tilde\bTheta$ is the unique global minimizer of $\tilde L\paran{\bTheta}$,
    then $\tilde L\paran{\tilde\bTheta} < \tilde L\paran{\bar{\calT}(\bB)}$ for all $\bar{\calT}(\bB)  \ne \tilde\bTheta$.
    By definition we have $\tilde L\paran{\tilde\bTheta} = L\paran{\tilde\bB}$ and $\tilde L\paran{\bar{\calT}(\bB)} = L\paran{\bar{\bB}}$ and thus the result.

    Now we study the penalty term $P\paran{\tilde\bB}$ and $P\paran{\bar{\bB}}$.
    The group-wise coefficients satisfy
    \begin{equation} \label{eqn:group-indiv-relation}
        \norm{\btheta - \mathring\btheta}_{\infty} = \underset{k}{\sup} \norm{ N_k^{-1}\sum_{i\in\calG} \paran{\bbeta^i- \mathring\bbeta^i} }_{\infty} \le \norm{\bbeta - \mathring\bbeta}_{\infty}.
    \end{equation}
    Then, by Assumption \ref{assum:signal-difference}, we have for any pair of groups $k \ne l$
    \begin{equation*}
        (JM)^{-1/2} \norm{\btheta^{(k)} - \btheta^{(l)}}_2 \ge  (JM)^{-1/2} \norm{\mathring\btheta^{(k)} - \mathring\btheta^{(l)}}_2 - 2 \norm{\btheta - \mathring\btheta}_{\infty} \ge d_{NT} - 2 \phi_{NT} > c \lam.
    \end{equation*}
    By Assumption \ref{assum:penalty}, we have $\tilde P\paran{\tilde\bTheta} = \tilde P\paran{\bar{\bTheta}} = C_N$.
    Further we have $P\paran{\tilde\bB} = P\paran{\bar{\bB}} = C_N$ since by definition we have $P\paran{\tilde\bB} = \tilde P\paran{\tilde\bTheta}$ and $P\paran{\bar{\bB}} = \tilde P\paran{\bar{\bTheta}}$ where $\bar{\bTheta} = \bar{\calT}(\bB)$.
    As a result,
    \begin{equation*}
        \begin{aligned}
            \calL\paran{\tilde\bB} = L\paran{\tilde\bB} + P\paran{\tilde\bB}
            < L\paran{\bar{\bB}} + P\paran{\bar{\bB}} = \calL\paran{\bar{\bB}}.
        \end{aligned}
    \end{equation*}
\end{proof}

\begin{lemma} \label{thm:E2-B*-min}
    Suppose the assumptions in Theorem \ref{thm:feasible-in-prob} hold. 
    For any $\bB\in\RR^{JM\times N}$, 
    let $\bar{\bTheta} = \bar{\calT}(\bB)$ and $\bar{\bB} = \calT^{-1}\paran{\bar{\bTheta}}$.
    There is an event $E_2$ such that $\Pr\paran{E_2^C} \le 2 ((NT)^{-1})$.
    On $E_1 \cap E_2$, there is a neighborhood $\tilde\calN$ of $\tilde\bB$
    such that $\calL\paran{ \bar{\bB} } < \calL\paran{\bB}$ for any $\bB \in \mathring\calN\cap\tilde\calN$
    for sufficiently large $N$ or $T$.
\end{lemma}

\begin{proof}
    For a positive sequence $\delta_{NT}$, let $\tilde\calN = \braces{\bbeta: \norm{ \bbeta - \tilde\bbeta }_2 \le \delta_{NT}}$ be a neighborhood of $\tilde\bbeta$.
    Similar to \eqref{eqn:group-indiv-relation}, we have, for any $\bbeta \in \tilde\calN \cap \mathring\calN$,
    \begin{equation*}
        \begin{aligned}
            \norm{\bar\btheta  - \mathring\btheta}_\infty
            & \le \norm{\bar\bbeta  - \mathring\bbeta}_\infty \le \phi_{NT}, \\
            \norm{\bar\btheta  - \tilde\btheta}_\infty
            & \le \norm{\bar\bbeta  - \tilde\bbeta}_\infty \le \delta_{NT},
        \end{aligned}
    \end{equation*}
    thus $\bar\bbeta \in \tilde\calN \cap \mathring\calN$.
    Let $\bbeta^{(\iota)}=  \iota \bbeta + (1-\iota) \bar\bbeta$ for some $\iota\in(0,1)$, we have
    \begin{equation*}
        \norm{\bbeta^{(\iota)} - \mathring\bbeta}_\infty %\le \norm{\bbeta  - \mathring\bbeta}_\infty 
        \le \phi_{NT},
    \end{equation*}
        By a Taylor expansion, we have
    \begin{equation*}
        \begin{aligned}
            \calL\paran{ \bbeta } - \calL\paran{ \bar\bbeta } & = \bG\paran{\pi, \bbeta}^\top \bG\paran{\pi, \bbeta} - \bG\paran{\pi, \bar\bbeta}^\top \bG\paran{\pi, \bar\bbeta} \\
            & ~~~ + P\paran{\bbeta, \lam} - P\paran{\bar\bbeta, \lam} \\
            % & = \bG\paran{\bbeta^m}^\top \bSigma ^\top \paran{\bbeta - \bar\bbeta } + \sum_{1\le i \le N} \paran{\frac{\partial P(\bbeta)}{\partial\bbeta^{(\iota)}} \bigg|_{\bbeta=\bbeta^m} }^\top \paran{\bbeta^i- \bar\bbeta^j } \\
            % & = \mathring\bG \bSigma ^\top \paran{\bbeta - \bar\bbeta } \\
            % & ~~~ + \frac{1}{2}  \paran{\bbeta - \bar\bbeta }^\top \bSigma \bSigma^\top\paran{\bbeta - \bar\bbeta } + \lam_{NT} P\paran{\bar\bbeta, \btheta^*} \\
            % & = (\bzeta_1 + \bzeta_2)^\top \bSigma ^\top \bnu
            % + \frac{1}{2}  \bnu^\top \bSigma \bSigma^\top \bnu + \lam_{NT} P\paran{\bar\bbeta, \btheta^*} \\
            % & = \frac{1}{N^2} \paran{ I_1 + I_2 },
            & =  I_1 + I_2,
        \end{aligned}
    \end{equation*}
    where, for some $\bbeta^{(\iota)}=  \iota \bbeta + (1-\iota) \bar\bbeta$ with $\iota\in(0,1)$,
    \begin{equation*}
        I_1 = 2 \bG\paran{\pi, \bbeta^{(\iota)}}^\top \bSigma \paran{\bbeta - \bar\bbeta }, \quad\text{and}\quad
        I_2 =  \frac{1}{N^2} \sum_{1\le i \le N} \paran{\frac{\partial P\paran{\bbeta, \lam}}{\partial\bbeta^{(\iota)}} \bigg|_{\bbeta=\bbeta^{(\iota)}} }^\top \paran{\bbeta^i- \bar\bbeta^i}.
    \end{equation*}

     Firstly, we consider $I_2$. 
     Recall that by Assumption \ref{assum:penalty}, 
     Term $I_2$ can be rewritten as
    \begin{equation*}
        \begin{aligned}
            N^2 I_2 & =  \sum_{1\le i \le N} \paran{\frac{ \partial P\paran{\bbeta, \lam} }{\partial\bbeta^{(\iota)}} \bigg|_{\bbeta=\bbeta^{(\iota)}} }^\top \paran{\bbeta^i- \bar\bbeta^i} \\
            %% correct, comment for brevity. 
            %& = \lam \sum_{i < j} \rho'\paran{ (JM)^{-1/2} \norm{\bbeta^{i,\iota} - \bbeta^{j,\iota}}_2 } (JM)^{-1/2} \norm{\bbeta^{i,\iota} - \bbeta^{j,\iota}}_2^{-1} \paran{\bbeta^{i,\iota} - \bbeta^{j,\iota}}^\top \paran{\bbeta^i- \bar\bbeta^i}  \\
            %& ~~~ + \lam \sum_{i > j} \rho'\paran{ (JM)^{-1/2} \norm{\bbeta^{i,\iota} - \bbeta^{j,\iota}}_2 } (JM)^{-1/2} \norm{\bbeta^{i,\iota} - \bbeta^{j,\iota}}_2^{-1} \paran{\bbeta^{i,\iota} - \bbeta^{j,\iota}}^\top \paran{\bbeta^i- \bar\bbeta^i} \\
            & = \lam \sum_{i < j} \rho'\paran{ (JM)^{-1/2} \norm{\bbeta^{i,\iota} - \bbeta^{j,\iota}}_2 } (JM)^{-1/2} \norm{\bbeta^{i,\iota} - \bbeta^{j,\iota}}_2^{-1} \paran{\bbeta^{i,\iota} - \bbeta^{j,\iota}}^\top \paran{\bbeta^i- \bar\bbeta^i}  \\
            & ~~~ + \lam \sum_{j > i} \rho'\paran{ (JM)^{-1/2} \norm{\bbeta^{j,\iota} - \bbeta^{i,\iota}}_2 } (JM)^{-1/2} \norm{\bbeta^{j,\iota} - \bbeta^{i,\iota}}_2^{-1} \paran{\bbeta^{j,\iota} - \bbeta^{i,\iota}}^\top \paran{\bbeta^j - \bar\bbeta^j}  \\
            & = \lam \sum_{i < j} \rho'\paran{ (JM)^{-1/2} \norm{\bbeta^{i,\iota} - \bbeta^{j,\iota}}_2 } (JM)^{-1/2} \norm{\bbeta^{i,\iota} - \bbeta^{j,\iota}}_2^{-1} \paran{\bbeta^{i,\iota} - \bbeta^{j,\iota}}^\top \paran{ \paran{\bbeta^i- \bar\bbeta^i} - \paran{\bbeta^j - \bar\bbeta^j} }.
        \end{aligned}
    \end{equation*}
    Note that, for any $i,j \in \calG_k$, $\bar\bbeta^i= \bar\bbeta^j$ and $\bbeta^{i,\iota} - \bbeta^{j,\iota}=\iota\paran{\bbeta^i- \bbeta^j}$,
    while, for any $i \in \calG_k$, $j \in \calG_l$, $k \ne l$,
    \begin{equation*}
        (JM)^{-1/2} \norm{\bbeta^{i,\iota}  - \bbeta^{j,\iota}}_2
        \ge \underset{k\ne l}{\min}\;(JM)^{-1/2}\norm{\mathring\btheta^{(k)}  - \mathring\btheta^{(l)}}_2
        - 2 \norm{\bbeta  - \mathring\bbeta}_{\infty}
        \ge d_{NT} - 2\phi_{NT} > a \lam,
    \end{equation*}
    and thus $\rho'\paran{(JM)^{-1/2} \norm{\bbeta^{i,\iota}  - \bbeta^{j,\iota}}_2}$ = 0.
    As a result, we have
    \begin{equation*}
        I_2 = \frac{1}{N^2} \cdot \lam \sum_{k=1}^{K} \sum_{i, j \in \calG_k, i < j} \rho'\paran{ (JM)^{-1/2} \iota \norm{\bbeta^i- \bbeta^j}_2 } (JM)^{-1/2} \norm{ \bbeta^i- \bbeta^j }_2.
    \end{equation*}
    Note that, since $\bar\bbeta \in \tilde\calN \cap \mathring\calN$, we have
    \begin{equation*}
        (JM)^{-1} \norm{\bbeta^i- \bbeta^j}_2^2
        \le 2 \norm{\bbeta - \bar\bbeta }_\infty^2
        \le 2 \norm{\bbeta - \tilde\bbeta }_\infty^2 + 2 \norm{\bar\bbeta - \tilde\bbeta }_\infty^2
        \le 4 \delta_{NT}.
    \end{equation*}
    Hence, by concavity of $\rho(\cdot)$ and $0 < \iota < 1$, we have
    \begin{equation} \label{eqn:I2}
        I_2 \ge \frac{1}{N^2} \cdot \lam \sum_{k=1}^{K} \sum_{i, j \in \calG_k, i < j} \rho'\paran{ 4 \delta_{NT} } (JM)^{-1/2} \norm{ \bbeta^i- \bbeta^j }_2.
    \end{equation}
    Now we consider $I_1$, which can be rewritten as
    \begin{equation*}
        \begin{aligned}
           - I_1 & = \bG\paran{\pi, \bbeta^{(\iota)}}^\top \bSigma \paran{\bbeta - \bar\bbeta } \\
            % & = \paran{ \sum_{i=1}^{N} \sum_{t=1}^{T} \paran{\bLam_i^\top \bZ_{i,t} R_{i,t}
            %        - \bLam_i^\top \bZ_{i,t} \paran{ \bZ_{i,t} - \gamma \bU_{\pi,i,t+1}}^\top \bLam_i  \bbeta^{(\iota)}} }^\top \bSigma \paran{\bbeta - \bar\bbeta } \\
            & = \frac{1}{N^2} \sum_{i=1}^{N} (\bG_i^\iota)^\top \bSigma_i \paran{\bbeta^i- \bar\bbeta^i} \\
            & =  \frac{1}{N^2} \sum_{1\le k \le K} \sum_{i,j\in\calG_k}  (N_k)^{-1} \paran{ \bSigma_i^\top \bG_i^\iota }^\top \paran{\bbeta^i- \bbeta^j }  \\
            & =  \frac{1}{N^2} \sum_{1\le k \le K} \sum_{i,j\in\calG_k, i < j}  (N_k)^{-1} \paran{ \bSigma_i^\top \bG_i^\iota - \bSigma_j^\top \bG_j^\iota }^\top \paran{\bbeta^i- \bbeta^j },
        \end{aligned}
    \end{equation*}
    where the third equation follows from the fact that $\sum_{j \in \calG_k} \paran{\bbeta^j - \bar\bbeta^j} = 0$ and
    \begin{align*}
            &\bG_i^\iota = (TJ)^{-1} \sum_{t=1}^{T} Z_{i,t}  \paran{ R_{i,t} - (Z_{i,t} - \gamma U_{i,t+1}^\pi) \bbeta^{i,\iota}} = \bzeta_{i,1} + \bzeta_{i,2} + \bSigma_i  (\bbeta^{i,\iota} - \mathring\bbeta^{i} )  \\
            %\bG_i^\iota = \bzeta_{i,1}^\iota + \bzeta_{i,2}^\iota, \quad
           &\bSigma_i  = (TJ)^{-1} \sum_{t=1}^{T} Z_{i,t} \paran{ Z_{i,t} - \gamma U_{i,t+1}^\pi}^\top, \quad 
           \bzeta_{i,1} = (TJ)^{-1} \sum_{t=1}^{T} Z_{i,t} \eps_{1,i,t}, \quad 
            \bzeta_{i,2} = (TJ)^{-1} \sum_{t=1}^{T} Z_{i,t} \eps_{2,i,t},
    \end{align*}
   and $\eps_{1,i,t}^\iota$ and $\eps_{2,i,t}^\iota$ are defined in \eqref{eqn:eps-1} and \eqref{eqn:eps-2}. 
    Hence, we have
   \begin{equation} \label{eqn:I1}
       \begin{aligned}
           \abs{I_1}
           & = \abs{\sum_{1\le k \le K} \sum_{i,j\in\calG_k, i < j}  (N_k)^{-1} \paran{ \bSigma_i \bG_i^\iota - \bSigma_j \bG_j^\iota }^\top \paran{\bbeta^i- \bbeta^j } }\\
           % & \le \sum_{1\le k \le K} (N_k)^{-1} \max_{i,j} \norm{ \bSigma_i \bG_i^\iota - \bSigma_j \bG_j^\iota }_2  \sum_{i,j\in\calG_k, i < j} \norm{\bbeta^i- \bbeta^j }_2 \\
           & \le (N_{\min})^{-1} \cdot \max_{i,j} \norm{ \bSigma_i \bG_i^\iota - \bSigma_j \bG_j^\iota }_2
           \cdot \sum_{1\le k \le K}  \sum_{i,j\in\calG_k, i < j} \norm{\bbeta^i- \bbeta^j }_2.
       \end{aligned}
   \end{equation}
    By Lemma \ref{lemma:sieve-basis} and \ref{thm:Z-U}, we have that, for some positive constant $c_1$, 
    \begin{equation*}
        \max_{1 \le i \le N} \norm{\bSigma_i}_2^2
        \le \max_{1 \le i \le N} (TJ)^{-2} T \cdot \sum_{t=1}^{T} \paran{ \norm{ Z_{i,t}Z_{i,t}^\top }_2^2 + \gamma \norm{ U_{\pi,i,t+1} Z_{i,t}^\top }_2^2  }
        \le 2 c_1^4 J^{-1}.
    \end{equation*}
    By Lemma  \ref{thm:Z-U} and \ref{thm:noise}, we have that, 
     \begin{align*}
        \max_{1 \le i \le N} \norm{\bzeta_{i,1} + \bzeta_{i,2}}_2^2 \le \max_{1 \le i \le N} (TJ)^{-2} T \cdot \sum_{t=1}^{T} \norm{Z_{i,t}}_2^2 \cdot \paran{ \norm{ \eps_{1,i,t}^\iota }_2^2 + \norm{\eps_{2,i,t}^\iota}_2^2  } \le J^{-1}\paran{ c_2^2 + (2 c_3 J^{-\kappa/p})^2) }.
    \end{align*}
     \begin{equation*}
        \max_{1 \le i \le N} \norm{ \bSigma_i  (\bbeta^{i,\iota} - \mathring\bbeta^{i} ) }_2^2
        \le \max_{1 \le i \le N} \norm{ \bSigma_i }_2^2  \max_{1 \le i \le N} \norm{ \bbeta^{i,\iota} - \mathring\bbeta }_2^2
        \le 2 c_1^4 \phi_{NT}^2.
    \end{equation*}
    Thus, we have
    \begin{equation} \label{eqn:SigmaG}
        \begin{aligned}
            \max_{i,j} \norm{\bSigma_i \bG_i^\iota - \bSigma_j \bG_j^\iota}_2
            & \le 2 \max_{1 \le i \le N} \norm{\bSigma_i \bG_i^\iota}_2 
            %& \le 2 \max_{1 \le i \le N} \norm{\bSigma_i}_2
            %\max_{1 \le i \le N} \paran{ \norm{ (TJ)^{-1} \sum_{t=1}^{T} \bZ_{i,t} \eps_{1,i,t} }_2  + \norm{  (TJ)^{-1} \sum_{t=1}^{T} \bZ_{i,t} \eps^{\iota}_{2,i,t} }_2 } \\
            % & \le 2 J^{-2} \max_{1 \le i \le N} \norm{\bSigma_i}_2 \max_{1 \le i \le N, 1\le t \le T} \norm{\bZ_{i,t}}_2 \max_{1 \le i \le N, 1\le t \le T} \paran{ \norm{\eps_{1,i,t} + \eps_{2,i,t}}_2  + \norm{ \eps_{2,i,t}^\iota - \eps_{2,i,t}}_2 } \\
            & \le 4 c_1^2 \cdot J^{-1} \cdot (c_2 + 2 c_3 J^{-\kappa/p}) + 4 c_1^2 J^{-1/2} \phi_{NT} ,
        \end{aligned}
    \end{equation}
    for some positive constant $c_j$, $j=1,2,3$.
    % $\phi_{NT} = \norm{ \tilde\btheta^{\pi} - \mathring\btheta^{\pi} }_\infty \le 6 c C^{-1} \frac{N_{\max}}{N_{\min}} \sqrt{2 J \frac{\log(N_{\max}T)}{N_{\max}T} }.$
    By the condition that $\lam \gg \phi_{NT}$, $J \ll \sqrt{N_{\min} T} / \log(N_{\min} T)$, and $J^{-\kappa / p} \ll (N_{\max} T)^{-1/2}$,
    we have
    \begin{equation*}
        \lam \gg N_{\min}^{-1}J^{-\kappa / p}, \quad\text{and}\quad \lam \gg N_{\min}^{-1} \phi_{NT}.
    \end{equation*}
    Let $\delta_{NT} = \smlo{1}$, then $\rho'\paran{2 \iota \delta_{NT}} \rightarrow 1$.
    Therefore, by equation \eqref{eqn:I2}, \eqref{eqn:SigmaG}, and \eqref{eqn:I1}, we have
    \begin{equation*}
        \begin{aligned}
            \calL\paran{ \bbeta } - \calL\paran{ \bar\bbeta }
            %& \ge \frac{1}{N^2} \cdot \paran{ \lam \rho'\paran{ 4 \delta_{NT} } (JM)^{-1/2} - (N_{\min})^{-1} \cdot J^{-1/2} \cdot \paran{4 c_1^2 \paran{ (c_2 + 2 c_3 J^{-\kappa/p}) + c_4 J \phi_{NT}  }} } \\
            %& ~~~ \cdot \sum_{1\le k \le K} \sum_{i,j\in\calG_k, i < j} \norm{\bbeta^i- \bbeta^j }_2 \\
            & \ge 0,
        \end{aligned}
    \end{equation*}
    for sufficiently large $N_{\min}$.
\end{proof}

%\noindent\textbf{Proof of Theorem \ref{thm:feasible-beta-asym}.}
%\begin{proof}
%    Let $\bu = (\hat\bw_{i\cdot}^\top \otimes\bI_{JM}) \bnu$
%    \begin{equation*}
%        \hat \sigma_{\beta_i}^{-1} \bnu^\top \paran{ \hat\bbeta^{\pi}_i- \mathring\bbeta^{\pi}_i }
%        = \frac{\bu^\top \paran{ \btheta^{\pi} - \mathring\btheta^{\pi} }}{\bu^\top \hat\bSigma^{-1} \hat\bOmega (\hat\bSigma^\top)^{-1} \bu}
%        \convdist \calN\paran{0, 1}.
%    \end{equation*}
%\end{proof}

%!TEX root = 0-main.tex

\section{Proof for the convergence of optimal policy}

\noindent\textbf{Proof of Theorem \ref{thm:opt-policy-conv}.}
\begin{proof}
 \begin{enumerate}[label=(\roman*)]
     \item
     By definition, we have
     $$\hat V_{\calR}(\pi(\balpha)) = \int \hat V(\pi(\balpha), \bx) d\calR(\bx) = \int \bu(\pi(\balpha), \bx)^\top\hat\bbeta^{\pi(\balpha)}  d\calR(\bx),$$
     where $\bu(\pi, \bx)$ is defined in \eqref{eqn:u}.
     %$\bU\paran{\pi, \bx} = \brackets{ \bphi(\bx)^\top \pi(1|\bx), \cdots, \bphi(\bx)^\top  \pi(M|\bx) }^\top \in \RR^{JM}$.
     Together with parametric policy class $\Pi$ defined in \eqref{eqn:policy-param}, we can write explicitly that
     \begin{equation}
         \begin{aligned}
             \hat V\paran{\pi(\balpha), \bx} - V\paran{\pi(\balpha), \bx}
             & = \sum_{j=1}^{M-1} \frac{\exp(\bx^\top\balpha_j)}{1 + \sum_{j=1}^{M-1} \exp(\bx^\top\balpha_j)} \bphi(\bx)^\top \paran{\hat\bbeta^{\pi(\balpha)}_j - \mathring{\bbeta}^{\pi(\balpha)}_j} \\
             & + \frac{1}{1 + \sum_{j=1}^{M-1} \exp(\bx^\top\balpha_j)} \bphi(\bx)^\top \paran{\hat\bbeta^{\pi(\balpha)}_M - \mathring{\bbeta}^{\pi(\balpha)}_M}
         \end{aligned}
     \end{equation}

     Following Theorem \ref{thm:feasible-in-prob} and Proposition \ref{thm:beta-l2-convg}, we have
    \begin{equation} \label{eqn:v-uniform-converge}
        \Pr\paran{\underset{\balpha}{\sup} \abs{\hat V_\calR\paran{\pi(\balpha)} - V_\calR\paran{\pi(\balpha)} } > \delta} \rightarrow 0.
    \end{equation}

     Combining \eqref{eqn:v-uniform-converge} with the unique and well-separated maximum condition in Assumption \ref{assum:policy-to-value},
     and continuity of $V_\calR\paran{\pi\paran{\balpha}}$ in $\alpha$, we have
     \begin{equation} \label{eqn:alpha-converge}
         \norm{\hat\balpha - \balpha^*}_2 \convprob 0.
     \end{equation}

     \item
     By Proposition \ref{thm:equicont},
     %\attn{Part 1 of Theorem 4.2},
     equation \eqref{eqn:alpha-converge}, and the continuous mapping theorem, we have
     \begin{equation*}
         \abs{V_\calR\paran{\pi(\hat\balpha)} - V_\calR\paran{\pi(\balpha^*)}} \convprob 0.
     \end{equation*}
 \end{enumerate}
\end{proof}

%!TEX root = 0-main.tex

\section{Sub-homogeneous MDP} \label{appen:homo}

With the knowledge of the true groups, the solution $\btheta^{\pi,(k)}$ for group $\calG^{(k)}$ of \eqref{eqn:Q-beta-ora} is equivalent to that of the quasi-likelihood estimating equation:
\begin{equation} \label{eqn:empirical-scores-k}
    \bG^{(k)}\paran{\pi, \btheta^{\pi,(k)}}
    = \sum_{i\in\calG^{(k)}} \sum_{t\in[T]} Z_{i,t} \paran{R_{i,t} - \paran{Z_{i,t} - \gamma U_{i,t+1}^{\pi}}^\top  \btheta^{\pi,(k)}} = 0 \,.
\end{equation}
We have $\E{\bG^{(k)}\paran{\pi, \mathring\btheta^{\pi,(k)}}} = 0$.
%We denote $\bOmega^{\pi,(k)}= \cov\paran{ \bG^{(k)}\paran{\btheta^{\pi,(k)}} }$ and
Let
\begin{align} \label{eqn:tilde-Sigma-k}
    \tilde\bSigma^{\pi,(k)} % & = - (N^{(k)} T)^{-1} \frac{\partial \bG^{(k)}\paran{\btheta^{\pi,(k)}}}{\partial \btheta^{\pi,(k)}}
    =  \frac{1}{N^{(k)} T} \sum_{i\in\calG^{(k)}} \sum_{t\in[T]}  Z_{i,t}\paran{ Z_{i,t} - \gamma U_{i,t+1}^{\pi}}^\top \,.
    % & \bH_k = \cov\paran{ T^{-1} \sum_{i\in\calG^{(k)}} \sum_{t\in[T]} Z_{i,t} \eps_{1,i,t}^\pi }.
\end{align}
The oracle estimator $\tilde\btheta^{\pi,(k)}$ for each group $\calG^{(k)}$ have the following decomposition:
\begin{equation} \label{eqn:theta-decompose-k}
    \paran{ \tilde\btheta^{\pi,(k)} - \mathring\btheta^{\pi,(k)} }
    %=  (\tilde\bSigma^{\pi,(k)})^{-1}  \bG^{(k)}\paran{\btheta^{\pi,(k)}}
    =  (\tilde\bSigma^{\pi,(k)})^{-1}  \paran{ \sum_{i\in\calG^{(k)}} \sum_{t\in[T]} \paran{Z_{i,t} (\eps_{1,i,t}^\pi + \eps_{2,i,t}^\pi)} }
    = (\tilde\bSigma^{\pi,(k)})^{-1} \tilde\bzeta^{\pi,(k)}_1 + (\tilde\bSigma^{\pi,(k)})^{-1} \tilde\bzeta^{\pi,(k)}_2,
\end{equation}
where
\begin{align}
        & \tilde\bzeta^{\pi,(k)}_1 = \frac{1}{N^{(k)} T} \sum_{i\in\calG^{(k)}} \sum_{t\in[T]} Z_{i,t} \eps_{1,i,t}^\pi,
           \qquad  \tilde\bzeta^{\pi,(k)}_2 = \frac{1}{N^{(k)} T} \sum_{i\in\calG^{(k)}} \sum_{t\in[T]} Z_{i,t} \eps_{2,i,t}^\pi, \label{eqn:tilde-zeta-k}\\
        & \eps_{1,i,t}^\pi = R_{i,t} + \gamma\cdot\sum_{a\in\calA} Q(\pi, X_{i,t+1}, a) \pi(a \mid X_{i,t+1}) -  Q(\pi, X_{i,t}, A_{i,t}), \\
        & \eps_{2,i,t}^\pi %=  \gamma\cdot\sum_{a\in\calA} \paran{Q(\pi, X_{i,t+1}, a) - \bphi(X_{i,t+1})^\top {\bbeta}^{\pi}_{a,i}} \pi(a \mid X_{i,t+1}) - \sum_{a\in\calA} \paran{Q(\pi, X_{i,t}, a)-\bphi(X_{i,t})^\top {\bbeta}^{\pi}_{a,i}}\bbone(A_{i,t}=a), \\
        = \gamma\cdot\sum_{a\in\calA} e(\pi, X_{i,t+1}, a) \pi(a \mid X_{i,t+1}) - \sum_{a\in\calA} e(\pi, X_{i,t}, a) \bbone(A_{i,t}=a),\\
        & e(\pi, \bx, a) = Q(\pi, \bx, a) - \bphi(\bx)^\top \mathring\btheta^{\pi,(k)}_{a}.
\end{align}
%The one step TD-error is the term $\eps_{1,i,t}^\pi + \eps_{2,i,t}^\pi$, of which $\eps_{1,i,t}^\pi$ is the random noise of the immediate rewards and $\eps_{2,i,t}^\pi$ is the bias from approximating $Q(\pi, \bx, a)$.
%Within each subgroup $k$ for $1 \le k \le K$, the MDP is homogeneous and our analysis for $\ell_2$ convergence of the policy evaluation is similar to that in \cite{shi2020statistical}.
%However, for the consistency of feasible estimators, we need stronger $\ell_\infty$ convergence result.

To derive the convergence rate of $\paran{ \tilde\btheta^{\pi,(k)} - \mathring\btheta^{\pi,(k)} }$, we study the properties of $\tilde\bSigma^{\pi,(k)}$, $\tilde\bzeta^{\pi,(k)}_1$, and $\tilde\bzeta^{\pi,(k)}_2$, respectively.
Since data is not independent over $T$, we specifically consider two settings as follows.
\begin{enumerate}[label=(\Roman*)]
    \item $T$ is fixed and $N^{(k)}$ goes to infinity.
    When $T$ is fixed and $N^{(k)}$ goes to infinity, we consider the concentration property of a random matrix $T^{-1} \sum_{t=0}^T  Z_{i,t}\paran{ Z_{i,t} - \gamma U_{i,t+1}^{\pi}}^\top$.
    \item $T$ goes to infinity and $N^{(k)}$ can be either fixed or go to infinity.
    When $T$ goes to infinity, we consider $Z_{i,t}\paran{ Z_{i,t} - \gamma U_{i,t+1}^{\pi}}^\top$ and apply the martingale concentration inequality under Assumption \ref{assum:geo-ergodic}.
\end{enumerate}
Although the techniques of proofs are different for these two settings, we obtain the same $\ell_2$ and $\ell_\infty$ bounds, which are summarized in Proposition \ref{thm:subhomo-l2}, Theorem \ref{thm:subhomo-distn}, Corollary \ref{thm:subhomo-distn-value} and Theorem \ref{eqn:sub-uniform-conv}.

Proofs of the results are presented in subsequent sections.
In the homogeneous setting, most of the results with respect to $\ell_2$ bounds are similar to those in \cite{shi2020statistical}.
Readers are referred to \cite{shi2020statistical} whenever we use their intermediate results.
%To simply notation, we suppress the subscript $k$ in the proof of this section.
%For brevity, we drop the superscript $\pi$. %and index trajectories in the $k$-th group by $1\le i \le N^{(k)}$ instead of by $i\in\calG^{(k)}$.
%We denote $\bOmega^{\pi,(k)}= \cov\paran{ \bG^{(k)}\paran{\btheta^{\pi,(k)}} }$ and

\begin{proposition}[Existence and Equicontinuity] \label{thm:equicont}
    Suppose Assumptions Assumption \ref{assum:r-q-func} -- \ref{assum:policy-smooth} hold.
    For all $\pi \in \Pi$, there exists a $\mathring{\btheta}^{\pi, (k)} \in \RR^{JM}$ such that $\EE[\bG^{(k)}(\pi, \btheta^{\pi, (k)})]$ has a zero at $\mathring{\btheta}^{\pi, (k)}$.
    Moreover, $\sup_{\pi\in\Pi} \norm{\mathring\btheta^{\pi, (k)}}_2 < \infty$ and
    $\sup_{\norm{\balpha_1 - \balpha_2}_2\rightarrow \delta} \norm{ \mathring{\btheta}^{\pi(\alpha_1), (k)}  - \mathring{\btheta}^{\pi(\alpha_1), (k)} }_2 \rightarrow 0$ as $\delta \downarrow 0$.
\end{proposition}
\begin{proof}
    The proof is similar to Theorem 4.2 Part 1 of \cite{luckett2019estimating} and thus is omitted here.
\end{proof}

\begin{proposition}[$\ell_2$ convergence of $\tilde\btheta^{\pi,(k)}$]  \label{thm:subhomo-l2}
    Suppose Assumption \ref{assum:r-q-func} -- \ref{assum:geo-ergodic} hold.
    If $J^{-\kappa / p} \ll (N^{(k)} T)^{-1/2}$ and $J \ll \sqrt{N^{(k)} T} / \log(N^{(k)} T)$,
    we have, as either $N\rightarrow \infty$ or $T\rightarrow \infty$,
    \begin{equation*}
        \norm{{\tilde\btheta^{\pi,(k)}-\mathring\btheta^{\pi,(k)}}}_2 = \Op{J^{- \kappa / p}} + \Op{J^{1/2} \paran{N^{(k)} T}^{-1/2}}
    \end{equation*}
    holds uniformly on the class of policies $\Pi$ defined by \eqref{eqn:policy-param}.
\end{proposition}

\begin{proof}
    Applying decomposition \eqref{eqn:theta-decompose-k}, Lemma \ref{thm:Sigma-k} and \ref{thm:zeta12}, we have
    \begin{align*}
        \norm{ \tilde\btheta^{\pi,(k)} - \mathring\btheta^{\pi,(k)} }_2
        & = \norm{ (\tilde\bSigma^{\pi,(k)})^{-1} }_2
        \cdot \norm{ \tilde\bzeta^{\pi,(k)}_1 }_2
        + \norm{ (\tilde\bSigma^{\pi,(k)})^{-1} }_2
        \cdot \norm{ \tilde\bzeta^{\pi,(k)}_2 }_2 \\
        & = \Op{ J^{1/2} (N^{(k)} T)^{-1/2} } + \Op{ J^{-\kappa / p} }.
    \end{align*}
\end{proof}

\begin{theorem}  [Uniform convergence of $\tilde\btheta^{\pi, (k)}$]   \label{eqn:sub-uniform-conv}
    Suppose Assumption \ref{assum:r-q-func} -- \ref{assum:geo-ergodic} hold.
    If $J \ll \sqrt{N^{(k)} T} / \log(N^{(k)} T)$, $J^{-\kappa / p} \ll (N^{(k)} T)^{-1/2}$,
    There exists some positive constant $C$ and $C_1$ that
    \begin{equation*}
        \norm{ \tilde\btheta^{\pi,(k)} - \mathring\btheta^{\pi,(k)} }_\infty
        \le C_1 C^{-1} \sqrt{J \cdot \frac{\log(N^{(k)}T)}{N^{(k)}T}},
    \end{equation*}
    holds with probability at least $1 - 2JMK (N^{(k)} T)^{-2}$.
    Also, the result holds uniformly on the class of policy $\Pi$ defined by \eqref{eqn:policy-param}.
\end{theorem}

\begin{proof}
    By decomposition \eqref{eqn:theta-decompose-k}, we have
    \begin{equation*}
        \norm{ \tilde\btheta^{\pi,(k)} - \mathring\btheta^{\pi,(k)} }_\infty
        \le \norm{\paran{ \tilde\bSigma^{\pi,(k)}}^{-1}}_\infty \norm{\tilde\bzeta^{\pi,(k)}_1}_\infty
        + \norm{\paran{ \tilde\bSigma^{\pi,(k)} }^{-1}}_\infty \norm{\tilde\bzeta^{\pi,(k)}_2}_\infty.
    \end{equation*}
    We bound each term in the right hand side as follows.
    By Lemma \ref{thm:Sigma-k} (iii), we have
    with probability at least $1 - \bigO{(N^{(k)} T)^{-2}}$,
    \begin{equation*}
        \norm{(\tilde\bSigma^{\pi,(k)})^{-1}}_\infty
        \le \sqrt{J} \norm{(\tilde\bSigma^{\pi,(k)})^{-1}}_2
        \le 6 C^{-1} \sqrt{J}.
    \end{equation*}
    By Lemma \ref{thm:zeta12} (iii), we have
    \begin{align*}
        \Pr\paran{ \norm{\bzeta_1^{(k)}}_\infty > c \sqrt{2 \log(N^{(k)}T) / (N^{(k)}T)}} \le 2 J M (N^{(k)} T)^{- 2},
    \end{align*}
    for some positive constant $c$.
    Lemma \ref{thm:zeta12}  (iv) shows that
    $\norm{\bzeta_2^{(k)}}_{\infty} \le c_1 J^{-\kappa/p} \ll c_1 \paran{N^{(k)} T}^{-1/2}$ almost surely.

    Thus, by a union bound, we have
    \begin{equation*}
        \norm{ \tilde\btheta^{\pi,(k)} - \mathring\btheta^{\pi,(k)} }_\infty
        \le C_1 C^{-1} \sqrt{J \cdot \frac{\log(N^{(k)}T)}{N^{(k)}T}}
    \end{equation*}
    holds with probability at least $1 - 2JMK (N^{(k)} T)^{-2}$.
\end{proof}

\begin{theorem}[Bidirectional asymptotic of $\tilde\btheta^{\pi,(k)}$]  \label{thm:subhomo-distn}
    Suppose Assumption \ref{assum:r-q-func} -- \ref{assum:geo-ergodic} hold.
    If $J \ll \sqrt{N^{(k)} T} / \log(N^{(k)} T)$,
    $J^{-\kappa / p} \ll (N^{(k)} T)^{-1/2}$,
    as either $N^{(k)} \rightarrow \infty$ or $T\rightarrow \infty$, we have for any vector $\bnu\in\RR^{JM}$,
    \begin{equation*}
        \begin{aligned}
            \sqrt{N^{(k)} T} \cdot (\tilde\sigma^{(k)})^{-1}(\pi, \bnu) \cdot \bnu^\top\paran{\tilde\btheta^{\pi,(k)} - \mathring\btheta^{\pi,(k)}}
            & \convdist \calN(0, 1),
        \end{aligned}
    \end{equation*}
    where $\tilde\sigma^{(k)}(\pi, \bnu)^2 = \bnu^\top (\tilde\bSigma^{\pi,(k)})^{-1} \tilde \bOmega^{\pi, (k)} (\tilde\bSigma^{\pi,(k)\; \top})^{-1} \bnu$,
    matrix $\tilde\bSigma^{\pi,(k)}$ is given in \eqref{eqn:tilde-Sigma-k} and
    \begin{align} \label{eqn:hat-Omega-k}
        \tilde \bOmega^{\pi, (k)}
        & =  (N^{(k)}T)^{-1} \sum_{i\in\calG^{(k)}} \sum_{t\in[T]} Z_{i,t} Z_{i,t}^\top \paran{R_{i,t} - \paran{Z_{i,t} - \gamma U_{i,t+1}^{\pi}}^\top  \tilde\btheta^{\pi,(k)}}^2  \,.
    \end{align}
\end{theorem}

\begin{proof}~
    \begin{enumerate}[label = \textsc{Step} \Roman*., wide=0pt, leftmargin=15pt, labelwidth=15pt, align=left]
        \item
        Letting $\bnu \in \RR^{JM}$ be any vector with $\norm{\bnu}_2$ bounded away from zero, we have
        \begin{equation*}
            \abs{ \bnu^\top \paran{ \tilde\btheta^{\pi,(k)} - \mathring\btheta^{\pi,(k)} -  (\tilde\bSigma^{\pi,(k)})^{-1} \tilde\bzeta^{\pi,(k)}_1 } }
            \le  \norm{ \bnu }_2 \norm{ \tilde\btheta^{\pi,(k)} - \mathring\btheta^{\pi,(k)} -  (\tilde\bSigma^{\pi,(k)})^{-1} \tilde\bzeta^{\pi,(k)}_1 }_2.
        \end{equation*}
        Let $\sigma^{(k)}\paran{\pi, \bnu}^2 = \bnu^\top  (\bSigma^{\pi,(k)})^{-1}  \bOmega^{\pi,(k)} {(\bSigma^{\pi,(k)})^{\top}}^{-1} \bnu$.
        By Lemma \ref{thm:Sigma-k} and \ref{thm:zeta12}, we have
        \begin{equation*}
            \begin{aligned}
                \sigma^{(k)}\paran{\pi, \bnu}^2 \ge 3^{-1} c_0^{-1} c_1 C \norm{ \bnu }_2^2.
            \end{aligned}
        \end{equation*}
        Further, it holds that
        \begin{align*}
            &\sigma^{(k)}\paran{\pi, \bnu}^{-1}\abs{ \bnu^\top\paran{\tilde\btheta^{\pi,(k)} - \mathring\btheta^{\pi,(k)} } - \bnu_k^\top (\tilde\bSigma^{\pi,(k)})^{-1} \tilde\bzeta^{\pi,(k)}_1 } \\
            & \le  \sqrt{3 c_0 c_1^{-1} C^{-1}} \norm{ \tilde\btheta^{\pi,(k)} - \mathring\btheta^{\pi,(k)} - (\tilde\bSigma^{\pi,(k)})^{-1} \tilde\bzeta^{\pi,(k)}_1 }_2 \\
            & = \Op{J^{-\kappa / p}} + \Op{J (N^{(k)} T)^{-1}\log(N^{(k)} T)}.
        \end{align*}
        Under the condition that $J \ll \sqrt{N^{(k)} T} \log(N^{(k)} T)^{-1}$ and $J^{\kappa / p} \gg \sqrt{N^{(k)} T}$, we have
        \begin{equation}
            \sqrt{N^{(k)} T} \cdot \sigma^{(k)}\paran{\pi, \bnu}^{-1}\bnu^\top\paran{\tilde\btheta^{\pi,(k)} - \mathring\btheta^{\pi,(k)}}
            = \sqrt{N^{(k)} T} \cdot \sigma^{(k)}(\pi,\bx)^{-1} \bnu^\top(\bSigma^{\pi,(k)})^{-1} \tilde\bzeta^{\pi,(k)}_1
            + \op{1}.
        \end{equation}

        \item
        We have
        \begin{align}
            \sqrt{N^{(k)} T}  \cdot \sigma^{(k)}\paran{\pi, \bnu}^{-1}\bnu^\top(\bSigma^{\pi,(k)})^{-1} \tilde\bzeta^{\pi,(k)}_1
            & = \sum_{i \in \calG^{(k)}} \sum_{t\in[T]} \frac{\bnu^\top(\bSigma^{\pi,(k)})^{-1} Z_{i,t}\eps_{1,i,t}^\pi}{\sqrt{N^{(k)} T} \sigma^{(k)}\paran{\pi, \bnu} } \convdist \calN(0, 1).
        \end{align}
        This can be shown by applying a martingale central limit theory for triangular arrays (Corollary 2.8 of McLeish, 1974).
        The construction of the martingale is the same as that in \cite{shi2020statistical}.
        % \attn{This is a point-wise CLT indexed by $\pi$. }
        The following two conditions
        \begin{enumerate}[label=(\alph*)]
            \item $\underset{i\in \calG^{(k)}, t\in [T]}{\max} \abs{ \frac{\bnu^\top(\bSigma^{\pi,(k)})^{-1} Z_{i,t}\eps_{1,i,t}^\pi}{\sqrt{N^{(k)} T} \sigma^{(k)}\paran{\pi, \bnu} } } \convprob 0$.
            \item $\frac{\sigma^{(k)}\paran{\pi, \bnu}^{-2}} {\sqrt{N^{(k)} T} } \sum_{i \in \calG^{(k)}} \sum_{t\in[T]} \bnu^\top(\bSigma^{\pi,(k)})^{-1} Z_{i,t}\eps_{1,i,t}^\pi  \convprob 1$.
        \end{enumerate}
        can be checked by applying Lemma \ref{thm:Sigma-k} and \ref{thm:cov-approx}, the argument is thus omitted here.

        \item Finally, % by Lemma \attn{???}, 
        it can be shown that $\hat \sigma^{(k)}(\pi,\nu) / \sigma^{(k)}(\pi,\nu) \convprob 1$.
        The desired result follows.
    \end{enumerate}
\end{proof}

\begin{corollary}[Bidirectional asymptotic of $\tilde V^{(k)}\paran{\pi, \bx}$]   \label{thm:subhomo-distn-value}
    Suppose Assumption \ref{assum:r-q-func} -- \ref{assum:geo-ergodic} hold.
    If $J \ll \sqrt{N^{(k)} T} / \log(N^{(k)} T)$,
    $ J^{-\kappa / p} \ll \paran{N^{(k)} T \paran{1 + \norm{\bu\paran{\pi, \bx}}_2^{-2}}}^{-1/2}$,
    as either $N\rightarrow \infty$ or $T\rightarrow \infty$, we have for any $\bx\in\calX$, % and uniformly over $\pi\in\Pi$,
    \begin{equation*}
        \begin{aligned}
            \sqrt{N^{(k)} T} \cdot \tilde\sigma^{(k)}(\pi,\bx)^{-1}
            \paran{ \tilde V^{(k)}\paran{\pi, \bx} -  V^{(k)}\paran{\pi, \bx} }
            & \convdist \calN\paran{0, 1} \,,
        \end{aligned}
    \end{equation*}
    where $\tilde\sigma^{(k)}\paran{\pi, \bx}^2
    = \bu\paran{\pi, \bx}^\top (\tilde\bSigma^{\pi,(k)})^{-1} \tilde \bOmega^{\pi, (k)} (\tilde\bSigma^{\pi,(k)})^{\top\;-1} \bu\paran{\pi, \bx}$,
    matrices $\tilde\bSigma^{\pi,(k)}$ and $\tilde \bOmega^{\pi, (k)}$ are given in \eqref{eqn:tilde-Sigma-k} and \eqref{eqn:hat-Omega-k}.
\end{corollary}

\begin{proof}
%    By Theorem \attn{???}, we have
%    \begin{equation*}
%        \norm{\tilde\btheta^{\pi,(k)} - \mathring\btheta^{\pi,(k)}}_2
%        \le (\tilde\bSigma^{\pi,(k)})^{-1} \tilde\bzeta^{\pi,(k)}_1
%        + \Op{J^{-\kappa/p} } + \Op{J (N^{(k)} T)^{-1}\log(N^{(k)}T)}.
%    \end{equation*}
%    Under the condition that $J \ll \sqrt{N^{(k)} T} / \log(N^{(k)}T)$, we have
%    \begin{equation*}
%        \norm{\tilde\btheta^{\pi,(k)} - \mathring\btheta^{\pi,(k)}}_2
%        \le \Op{J^{-\kappa/p} } + \Op{J^{1/2} (N^{(k)} T)^{-1/2}}.
%    \end{equation*}
    For any fixed $\pi$, we have $\tilde V^{(k)}\paran{\pi, \bx} = \bu\paran{\pi, \bx}^\top \tilde\btheta^{\pi,(k)}$.
    Thus, we obtain that
    \begin{align*}
        & \abs{
            \tilde V^{(k)}\paran{\pi, \bx}
            -  V^{(k)}\paran{\pi, \bx}
            - \bu\paran{\pi, \bx}^\top (\tilde\bSigma^{\pi,(k)})^{-1} \tilde\bzeta^{\pi,(k)}_1 } \\
        & \le  \abs{ \bu\paran{\pi, \bx}^\top \paran{ \tilde\btheta^{\pi,(k)} - \mathring\btheta^{\pi,(k)} -  (\tilde\bSigma^{\pi,(k)})^{-1} \tilde\bzeta^{\pi,(k)}_1 } }
                + \abs{ \bu(\pi, \bx)^\top \mathring\btheta^{\pi,(k)} - V^{(k)}(\pi, \bx) } \\
         & \le  \norm{ \bu(\pi, \bx)^\top }_2 \norm{ \tilde\btheta^{\pi,(k)} - \btheta^{\pi,(k)} -  (\tilde\bSigma^{\pi,(k)})^{-1} \tilde\bzeta^{\pi,(k)}_1 }_2 + C J^{-\kappa / p}.
    \end{align*}
    %where we used Lemma \attn{???} in the last inequality.

    Under the condition that $J \ll \sqrt{N^{(k)} T} \log(N^{(k)} T)^{-1}$, $J^{- \kappa /  p} \ll \sqrt{N^{(k)} T \paran{1 +\norm{ \bu(\pi, \bx) }_2^2}}$, we obtain
    \begin{equation*}
        \begin{aligned}
            \sqrt{N^{(k)} T} \cdot \sigma^{(k)}(\pi,\bx)^{-1} \paran{ \tilde V^{(k)}(\pi, \bx) -  V^{(k)}(\pi, \bx)}
            & = \sqrt{N^{(k)} T} \cdot \sigma^{(k)}(\pi,\bx)^{-1} \bu(\pi, \bx)^\top(\bSigma^{\pi,(k)})^{-1} \tilde\bzeta^{\pi,(k)}_1
            + \op{1}.
        \end{aligned}
    \end{equation*}
    Applying Theorem \ref{thm:subhomo-distn} by setting $\bnu = \bu\paran{\pi, \bx}$, we obtain the desired result.
\end{proof}

\begin{corollary}[Bidirectional asymptotic of $\tilde V^{(k)}_\calR\paran{\pi}$]   \label{thm:subhomo-distn-value-intg}
    Suppose Assumption \ref{assum:r-q-func} -- \ref{assum:geo-ergodic} hold.
    If $J \ll \sqrt{N^{(k)} T} / \log(N^{(k)} T)$,
    $ J^{-\kappa / p} \ll \paran{N^{(k)} T \paran{1 + \norm{\int\bphi_J(\bx)\calR(d\bx)}_2^{-2}}}^{-1/2}$,
    as either $N\rightarrow \infty$ or $T\rightarrow \infty$, we have for any $\bx\in\calX$, % and uniformly over $\pi\in\Pi$,
    \begin{equation*}
        \begin{aligned}
            \sqrt{N^{(k)} T} \cdot \tilde\sigma^{(k)}(\pi,\calR)^{-1}
            \paran{ \tilde V^{(k)}_{\calR}\paran{\pi} -  V^{(k)}_{\calR}\paran{\pi} }
            & \convdist \calN\paran{0, 1} \, ,
        \end{aligned}
    \end{equation*}
    where
    $$\tilde\sigma^{(k)}\paran{\pi, \calR}^2
    = \paran{\int\bu\paran{\pi, \bx}\calR(d\bx)}^\top (\tilde\bSigma^{\pi,(k)})^{-1} \tilde \bOmega^{\pi, (k)} (\tilde\bSigma^{\pi,(k)})^{\top\;-1} \paran{\int\bu\paran{\pi, \bx}\calR(d\bx)},$$
    matrices $\tilde\bSigma^{\pi,(k)}$ and $\tilde\bSigma^{\pi,(k)}$ are given in \eqref{eqn:tilde-Sigma-k} and \eqref{eqn:hat-Omega-k}.
\end{corollary}
\begin{proof}
    Note that $\tilde V^{(k)}_\calR\paran{\pi} = \paran{\int_\calR\bu\paran{\pi, \bx}\calR(d\bx)}^\top \tilde\btheta^{\pi,(k)}$.
    Similar to the arguments in Corollary \ref{thm:subhomo-distn-value}, we have
    \begin{align*}
        & \abs{
            \tilde V^{(k)}_\calR\paran{\pi}
            -  V^{(k)}_\calR\paran{\pi}
            - \paran{\int_\calR\bu\paran{\pi, \bx}\calR(d\bx)}^\top (\tilde\bSigma^{\pi,(k)})^{-1} \tilde\bzeta^{\pi,(k)}_1 } \\
        & \le  \norm{ \int\bu\paran{\pi, \bx}\calR(d\bx) }_2
        \norm{ \tilde\btheta^{\pi,(k)} - \btheta^{\pi,(k)} -  (\tilde\bSigma^{\pi,(k)})^{-1} \tilde\bzeta^{\pi,(k)}_1 }_2
        + C J^{-\kappa / p}.
    \end{align*}
    Under the condition that $J \ll \sqrt{N^{(k)} T} \log(N^{(k)} T)^{-1}$, $J^{- \kappa /  p} \ll \sqrt{N^{(k)} T \paran{1 +\norm{ \int\bu(\pi, \bx) \calR(d\bx) }_2^2}}$, we obtain
    \begin{equation*}
        \begin{aligned}
            & \sqrt{N^{(k)} T} \cdot \sigma^{(k)}(\pi,\bx)^{-1} \paran{ \tilde V^{(k)}(\pi, \bx) -  V^{(k)}(\pi, \bx)} \\
            & = \sqrt{N^{(k)} T} \cdot \sigma^{(k)}(\pi,\bx)^{-1}
            \paran{\int\bu\paran{\pi, \bx}\calR(d\bx)}^{\top}
            (\bSigma^{\pi,(k)})^{-1} \tilde\bzeta^{\pi,(k)}_1
            + \op{1}.
        \end{aligned}
    \end{equation*}
    Applying Theorem \ref{thm:subhomo-distn} by setting $\bnu = \int\bu\paran{\pi, \bx}\calR(d\bx)$, we obtain the desired result.

\end{proof}

\subsection{Useful lemmas}

\subsubsection{ Properties of $\tilde\bSigma^{\pi,(k)}$, $\tilde\bzeta^{\pi,(k)}_1$, and $\tilde\bzeta^{\pi,(k)}_2$}

In this section, we study the properties of $\tilde\bSigma^{\pi,(k)}$, $\tilde\bzeta^{\pi,(k)}_1$, and $\tilde\bzeta^{\pi,(k)}_2$.
The results in this section applies to the setting as either $N\rightarrow\infty$ or $T\rightarrow\infty$.
We restrict our attention to two particular type of Sieve basis functions satisfying Assumption \ref{assum:basis}. 

\begin{assumption} \label{assum:basis}
    Let $BSpline(J, r)$ denote a tensor-product B-spline basis of dimension $J$ and of degree $r$ on $[0,1]^p$ and
    $Wav(J, r)$ denote a tensor-product Wavelet basis of regularity $r$ and dimension $J$ on $[0,1]^p$.
    The sieve $\bphi_J$ is either $BSpline(J, r)$ or $Wav(J, r)$ with $r > \max\paran{\kappa, 1}$.
\end{assumption}

%\begin{remark}
%    \elynn{Remark on basis functions. Gaussian basis function offers the highest degree of flexibility.}
%\end{remark}

\begin{lemma}  \label{thm:smooth-Q}
    User Assumption \ref{assum:r-q-func},
    there exist some $c > 0$ such that, for any given $\pi$ and $a\in\calA$, $Q(\pi; \bx, a)$ as a function of $\bx$ belongs to $\calH(\kappa, c)$.
\end{lemma}
\begin{proof}
    See the proof of Lemma 1 in \cite{shi2020statistical}.
\end{proof}

\begin{lemma} \label{lemma:sieve-basis}
Suppose the basis functions in $\bphi_J(\bx)$ satisfy Assumption \ref{assum:basis}, there exists some constant $c \ge 1$ such that
\begin{equation*}
    c^{-1}
    \le \lam_{\min}\paran{\int_{\bx\in\calX} \bphi_J(\bx)\bphi_J(\bx)^\top d\bx}
    \le \lam_{\max}\paran{\int_{\bx\in\calX} \bphi_J(\bx)\bphi_J(\bx)^\top d\bx}
    \le c,
\end{equation*}
and $\underset{\bx\in\calX}{\sup}\;\norm{\phi_J(\bx)}_2\le c\sqrt{J}$.
\end{lemma}
\begin{proof}
    See \cite{shi2020statistical} Lemma 2.
    For the B-spline basis, the first assertion follows from the arguments used in the proof of Theorem 3.3 of \cite{burman1989nonparametric}.
    For wavelet basis, the first assertion follows from the arguments used in the proof of Theorem 5.1 of \cite{chen2015optimal}.

    For either B-spline or wavelet sieve and any $J\ge 1$, $\bx\in\calX$, the number of nonzero elements in the vector $\bphi(\bx)$ is bounded by some constant.
    Moreover, each of the basis function is uniformly bounded by $\bigO{\sqrt{J}}$.
    This proves the second assertion.
\end{proof}

\begin{lemma}~ \label{thm:Z-U}
    Suppose the basis functions in $\bphi_J(\bx)$ satisfy Assumption \ref{assum:basis}.  
    Under Assumption \ref{assum:density}, we have
    \begin{enumerate}[label=(\roman*)]
        \item $\underset{i\in [N], t\in [T]}{\max}\; \norm{Z_{i,t}}_2 \le c \sqrt{J}$
        \item $\underset{\pi\in\Pi}{\sup} \; \underset{i\in [N], t\in [T]}{\max}\; \norm{U_{i,t}^{\pi}}_2 \le c \sqrt{J}$
        \item $\norm{ \E{Z_{i,t} Z_{i,t}^\top} }_2^2 = \bigO{1}$ and $\norm{ \E{Z_{i,t}^\top Z_{i,t}} }_2^2 = \bigO{1}$
    \end{enumerate}
\end{lemma}
\begin{proof}
    Recall that
    \begin{equation*}
    \begin{aligned}
    Z_{i,t} & = \brackets{ \bphi(X_{i,t})^\top \bbone(A_{i,t}=1), \cdots, \bphi(X_{i,t})^\top \bbone(A_{i,t}=M) }^\top \in \RR^{JM}, \\
    U_{i,t+1}^{\pi} & = \brackets{ \bphi(X_{i,t+1})^\top \pi(1|X_{i,t+1}), \cdots, \bphi(X_{i,t+1})^\top  \pi(M|X_{i,t+1}) }^\top \in \RR^{JM}.
    \end{aligned}
    \end{equation*}
   Then, we have
    \begin{enumerate}[label=(\roman*)]
        \item $\underset{i\in [N], t\in [T]}{\max}\; \norm{Z_{i,t}}_2 \le \underset{\bx}{\sup}\; \norm{\bphi_J(\bx)}_2 \le c \sqrt{J}$
        \item $\underset{\pi\in\Pi}{\sup} \; \underset{i\in [N], t\in [T]}{\max}\; \norm{U_{i,t}^{\pi}}_2 \le \underset{\bx}{\sup}\; \norm{\bphi_J(\bx)}_2 \le c \sqrt{J}$
        \item
        Let $\bnu=\brackets{\bnu_1^\top, \cdots, \bnu_M^\top}^\top$ where sub-vector $\bnu_m \in \RR^J$ for $1\le m \le M$.

        \begin{equation*}
        \begin{aligned}
        \bnu^\top \E{ Z_{i,t} Z_{i,t}^\top } \bnu
        & = \sum_{m=1}^M \bnu_m^\top \E{\bphi(X_{i,t})\bphi(X_{i,t})^\top \bbone(A_{i,t}=1) } \bnu_m \\
        & = \sum_{m=1}^M \bnu_m^\top \E{\bphi(X_{i,t})\bphi(X_{i,t})^\top} \bnu_m b(m|X_{i,t}) \\
        & \le \lam_{\max}\paran{ \int_{\bx\in\calX} \bphi(\bx)\bphi(\bx)^\top \mu_0(\bx) d\bx } \norm{\bnu}_2^2
        \end{aligned}
        \end{equation*}
        By Assumption \ref{assum:density} and Lemma \ref{lemma:sieve-basis}, we obtain
        \begin{equation*}
            \lam_{\max}\paran{ \int_{\bx\in\calX} \bphi(\bx)\bphi(\bx)^\top \mu_0(\bx) d\bx }
            \le \underset{\bx\in\calX}{\sup} \mu_0(\bx) \; \lam_{\max}\paran{ \int_{\bx\in\calX} \bphi(\bx)\bphi(\bx)^\top d\bx }  = \bigO{1}.
        \end{equation*}
        Thus, we have
        \begin{equation*}
            \begin{aligned}
                \norm{ \EE Z_{i,t} Z_{i,t}^\top }_2^2 = \bigO{1}
            \end{aligned}
        \end{equation*}
        Using similar arguments, we have $\norm{ \EE Z_{i,t}^\top Z_{i,t} }_2^2 = \bigO{1}$.

    \end{enumerate}
\end{proof}

\begin{lemma}   \label{thm:noise}
    Under Assumption \ref{assum:r-q-func} and \ref{assum:basis}, there exists some positive constants $c_1$, $c_2$, $c_3$, such that
    \begin{enumerate}[label = (\roman*)]
        \item $\underset{\pi\in\Pi}{\sup} \; \underset{i\in\calG^{(k)}, 1\le t \le T}{\max}\; \abs{\eps_{1,i,t}^\pi} \le c_1 + 2 c_2$
         \item $ \underset{\pi\in\Pi}{\sup} \; \underset{i\in\calG^{(k)}, 1\le t \le T}{\max}\; \abs{\eps_{2,i,t}^\pi} \le 2 c_3 J^{-\kappa/p}$
    \end{enumerate}
\end{lemma}

\begin{proof}
    Recall that we define, for $\forall i\in\calG^{(k)}, \; 1\le k \le K$,
    \begin{equation*}
    \begin{aligned}
    & \eps_{1,i,t}^\pi = R_{i,t} + \gamma\cdot\sum_{a\in\calA} Q_k(\pi, X_{i,t+1}, a) \pi(a \mid X_{i,t+1}) -  Q_k(\pi, X_{i,t}, A_{i,t}) \\
    & \eps_{2,i,t}^\pi = \gamma\cdot\sum_{a\in\calA} e_k(\pi, X_{i,t+1}, a) \pi(a \mid X_{i,t+1}) - \sum_{a\in\calA} e_k(\pi, X_{i,t}, a) \bbone(A_{i,t}=a),\\
    & e_k(\pi, \bx, a) = Q_k(\pi, \bx, a) - \bphi(\bx)^\top \mathring\btheta^{\pi, (k)}_{a}.
    \end{aligned}
    \end{equation*}

    \begin{enumerate}[label = (\roman*)]
        \item
        {The condition that $\Pr\paran{\underset{0\le t \le T}{\max}\; \abs{R_{i,t}} \le c_1} = 1$} implies that $\abs{R_{i,t}}\le c_1$ for $\forall i, t$, almost surely.
        By Lemma \ref{thm:smooth-Q} and the definition of $k$-smooth function, we obtain that $\abs{Q(\pi, \bx, a)} \le c_2$ for any $\pi, \bx, a$.
        Thus, we have
        \begin{equation}   \label{eqn:eps-it-max}
            \underset{i\in\calG^{(k)}, 1\le t \le T}{\max}\; \abs{\eps_{1,i,t}^\pi}
            \le  c_1 + (1+\gamma) c_2
            \le c_1 + 2 c_2 \, , \quad\text{almost surely.}
        \end{equation}

        \item
        By Lemma \ref{thm:smooth-Q} and Assumption \ref{assum:r-q-func}, there exist a set of vector $\btheta^{\pi,(k)}_{a}$ that satisfy (see Section 2.2 of \cite{huang1998projection} for details)
        \begin{equation*}
            \underset{\bx\in\calX, a\in\calA}{\sup}\; \abs{e_k(\pi, \bx, a)} = \underset{\bx\in\calX, a\in\calA}{\sup}\; \abs{Q_k(\pi, \bx, a) - \bphi(\bx)^\top \mathring\btheta^{\pi,(k)}_{a}} \le c_3 J^{-\kappa/p} ,
        \end{equation*}
        for some positive constant $c_3$.
        Thus, we have
        \begin{equation*}
        \begin{aligned}
        &  \underset{i \in \calG^{(k)}, 1\le t \le T}{\sup}\; \abs{\eps_{2,i,t}^\pi} =  \underset{1\le i \le N, 1\le t \le T}{\sup}\; \abs{\gamma\cdot\sum_{a\in\calA} e_k(\pi, X_{i,t+1}, a) \pi(a \mid X_{i,t+1}) - \sum_{a\in\calA} e_k(\pi, X_{i,t}, a) \bbone(A_{i,t}=a)},\\
        &  \qquad\qquad \le (1+\gamma) \underset{\bx\in\calX, a\in\calA}{\sup}\; \abs{e_k(\pi, \bx, a)}
        = 2 c_3 J^{-\kappa/p}.
        \end{aligned}
        \end{equation*}

    \end{enumerate}
\end{proof}

Define $\tilde e(\pi, \bx, a) = Q(\pi, \bx, a) - \bphi(\bx)^\top \tilde\bbeta^{\pi,i}_{a}$ and
\begin{equation*}
    \tilde \eps_{2,i,t}^\pi
    = \gamma\cdot\sum_{a\in\calA} \tilde e(\pi, X_{i,t+1}, a) \pi(a \mid X_{i,t+1}) - \sum_{a\in\calA} \tilde e(\pi, X_{i,t}, a) \bbone(A_{i,t}=a).
\end{equation*}
The following lemma is needed for the proof of asymptotic normality, that is, Theorem \ref{thm:subhomo-distn} and Corollary \ref{thm:subhomo-distn-value}.
Its proof uses results in Proposition \ref{thm:subhomo-l2}.
\begin{lemma}   \label{thm:noise-approx}
    Under Assumption \ref{assum:basis}, we have
    \begin{enumerate}[label = (\roman*)]
        \item  $\underset{i\in\calG^{(k)}, 1\le t \le T}{\max}\; \abs{2\eps_{1,i,t}^\pi + \tilde\eps_{2,i,t}^\pi} = \Op{1}$
        \item $\underset{i\in\calG^{(k)}, 1\le t \le T}{\max}\; \abs{\tilde\eps_{2,i,t}^\pi} = \op{1}$
        \item  $\underset{i\in\calG^{(k)}, 1\le t \le T}{\max}\; \abs{(\eps_{1,i,t}^\pi)^2 - (\eps_{1,i,t}^\pi + \tilde\eps_{2,i,t}^\pi)^2} = \op{1}$
    \end{enumerate}
\end{lemma}

\begin{proof}
    \begin{enumerate}[label = (\roman*)]
        \item The result follows trivially from Lemma \ref{thm:noise}.
        \item First, we have
        \begin{align*}
                \underset{\bx\in\calX, a\in\calA}{\sup}\; \abs{ \tilde e(\pi, \bx, a) }
                & \le  \underset{\bx\in\calX, a\in\calA}{\sup}\; \abs{ Q(\pi, \bx, a) - \bphi(\bx)^\top \tilde\bbeta^{\pi,i}_{a} } \\
                & \le  \underset{\bx\in\calX, a\in\calA}{\sup}\; \abs{ Q(\pi, \bx, a) - \bphi(\bx)^\top \mathring\bbeta^{\pi,i}_{a} }
                +  \underset{\bx\in\calX, a\in\calA}{\sup}\; \abs{ \bphi(\bx)^\top \mathring\bbeta^{\pi,i}_{a} - \bphi(\bx)^\top \tilde\bbeta^{\pi,i}_{a} } \\
                & \le C J^{-\kappa / p} + \underset{\bx\in\calX, a\in\calA}{\sup}\; \norm{ \bphi(\bx) }_2 \underset{a\in\calA}{\sup}\; \norm{ \mathring\bbeta^{\pi,i}_{a} - \tilde\bbeta^{\pi,i}_{a}}_2 \\
                & = \Op{ J^{1/2 -\kappa / p} }  +  \Op{ J (N^{(k)} T)^{-1/2} }.
        \end{align*}
        Thus, we obtain that
        \begin{equation*}
            \begin{aligned}
                \underset{i\in\calG^{(k)}, 1\le t \le T}{\max}\; \abs{\tilde \eps_{2,i,t}^\pi}
                & \le \underset{i\in\calG^{(k)}, 1\le t \le T}{\max}\; \abs{\gamma\cdot\sum_{a\in\calA} \tilde e(\pi, X_{i,t+1}, a) \pi(a \mid X_{i,t+1}) - \sum_{a\in\calA} \tilde e(\pi, X_{i,t}, a) \bbone(A_{i,t}=a)}  \\
                & \le (1+\gamma) \underset{\bx\in\calX, a\in\calA}{\sup}\; \abs{ \tilde e(\pi, \bx, a) } \\
                & = \Op{ J^{1/2 -\kappa / p} }  +  \Op{ J (N^{(k)} T)^{-1/2} }.
            \end{aligned}
        \end{equation*}
        Under the condition that $J^{-\kappa / p} \ll (N^{(k)} T)^{-1/2}$ and $J \ll (N^{(k)} T)^{1/2} \log(N^{(k)} T)^{-1}$, we have $\Op{ J^{1/2 -\kappa / p} }  = \op{1}$ and $\Op{ J (N^{(k)} T)^{-1/2} } = \op{1}$.
        Therefore, we have $\underset{i\in\calG^{(k)}, 1\le t \le T}{\max}\; \abs{\tilde \eps_{2,i,t}^\pi} = \op{1}$.
        \item $\underset{i\in\calG^{(k)}, 1\le t \le T}{\max}\; \abs{(\eps_{1,i,t}^\pi)^2 - (\eps_{1,i,t}^\pi + \tilde \eps_{2,i,t}^\pi)^2} \le \underset{i\in\calG^{(k)}, 1\le t \le T}{\max}\; \abs{ 2\eps_{1,i,t}^\pi + \tilde \eps_{2,i,t}^\pi} \underset{i\in\calG^{(k)}, 1\le t \le T}{\max}\; \abs{\tilde \eps_{2,i,t}^\pi} = \op{1}$
    \end{enumerate}
\end{proof}

\begin{lemma} [Lemma E.2 of \cite{shi2020statistical}.]   \label{thm:Sigma-k}
    Suppose the conditions in Proposition \ref{thm:subhomo-l2} hold.
    We have as either $T\rightarrow\infty$ or $N^{(k)}\rightarrow\infty$,
    \begin{enumerate}[label=(\roman*)]
        \item  $\norm{\E{\tilde\bSigma^{\pi,(k)}}} _2 \ge 2^{-1} C_1$ and $\norm{\E{\tilde\bSigma^{\pi,(k)}}^{-1}} _2 \le 2 C_1^{-1}$,
        \item $\norm{\tilde\bSigma^{\pi,(k)} - \E{\tilde\bSigma^{\pi,(k)}}}_2 \le c_2 \sqrt{J (N^{(k)} T)^{-1} \log(N^{(k)}T)}$ with probability at least $1 - \bigO{(N^{(k)}T)^{-2}}$,
        \item $\norm{\tilde\bSigma^{\pi,(k)}}_2  \le 6 C_1^{-1}$ with probability at least $1 - \bigO{(N^{(k)}T)^{-2}}$,
        \item $\norm{(\tilde\bSigma^{\pi,(k)})^{-1} - \E{\tilde\bSigma^{\pi,(k)}}^{-1}} = \Op{ \sqrt{J (N^{(k)} T)^{-1} \log(N^{(k)}T)} }$,
    \end{enumerate}
    where $C_1$ is specified in Assumption \ref{assum:min-eigen-of-E-Sigma} and $c_1$ and $c_2$ are some positive constant.
\end{lemma}
\begin{proof}
    See Section E.7 of \cite{shi2020statistical} for a complete proof.
\end{proof}

Next, we have some properties of $\tilde\bzeta_1$ and $\tilde\bzeta_2$.
\begin{lemma}[Lemma E.3 of \cite{shi2020statistical}.] \label{thm:Z-U-eig}
    Suppose the conditions in Proposition \ref{thm:subhomo-l2} hold.
    As either $N \rightarrow \infty$ or $T\rightarrow \infty$, we have
    \begin{enumerate}[label=(\roman*)]
        \item $\lam_{\max} \paran{ T^{-1} \sum_{t=0}^{T-1} Z_{i,t} Z_{i,t}^\top } = \Op{1}$
        \item $\lam_{\max} \paran{ (N^{(k)} T)^{-1} \sum_{i\in\calG^{(k)}} \sum_{t=0}^{T-1} Z_{i,t} Z_{i,t}^\top } = \Op{1}$
        \item $\lam_{\min} \paran{ T^{-1} \sum_{t=0}^{T-1} Z_{i,t} Z_{i,t}^\top } \ge c / 2$ with probability approaching one.
        \item $\lam_{\min} \paran{ (N^{(k)} T)^{-1} \sum_{i\in\calG^{(k)}} \sum_{t=0}^{T-1} Z_{i,t} Z_{i,t}^\top } \ge c / 3$ with probability approaching one.
    \end{enumerate}
\end{lemma}
%\noindent\textbf{Proof of Lemma \ref{thm:Z-U-eig}.}
\begin{proof}
    %The proof is similar to that of Lemma \ref{thm:Sigma-k} and thus omitted here.
    See Section E.8 of \cite{shi2020statistical} for detailed proofs.
\end{proof}

\begin{lemma}  \label{thm:zeta12}
    Suppose the conditions in Proposition \ref{thm:subhomo-l2} hold.
    There exists some positive constant $c$, $c_1$ such that, as either $N \rightarrow \infty$ or $T\rightarrow \infty$,
    \begin{enumerate}[label=(\roman*)]
        \item $\E{ \tilde\bzeta_1^{(k)} } = 0$, $\E{  \norm{\tilde\bzeta_1^{(k)}}_2^2 } \le c_3 J (N^{(k)} T)^{-1}$
        and $\tilde\bzeta_1^{(k)} = \Op{\sqrt{J/(N^{(k)} T)}}$.
        \item $\norm{ \tilde \bzeta_2^{(k)}}_2 = \Op{ J^{-\kappa/p}  }$
        \item $\Pr\paran{ \norm{\tilde \bzeta_1^{(k)}}_\infty > c \sqrt{2 \log(N^{(k)}T) / (N^{(k)}T)}} \le 2 J M (N^{(k)} T)^{- 2}$.
        \item $\norm{\tilde \bzeta_2^{(k)}}_{\infty} \le c_1 J^{-\kappa/p} \ll c_1 \paran{N^{(k)} T}^{-1/2}$ almost surely.
    \end{enumerate}
\end{lemma}

\begin{proof}
    For brevity, we drop $\;\tilde{}\;$ in this proof. 
    Recall that
    \begin{equation*}
        \bzeta_1^{(k)} = (N^{(k)} T)^{-1}\sum_{i\in\calG^{(k)}} \sum_{t=1}^{T} Z_{i,t} \eps_{1,i,t}^\pi, \quad\text{ and }\quad
        \bzeta_2^{(k)} = (N^{(k)} T)^{-1} \sum_{i\in\calG^{(k)}} \sum_{t=1}^{T} Z_{i,t} \eps_{2,i,t}^\pi.
    \end{equation*}

    \begin{enumerate}[label=(\roman*)]
        \item  By the Bellman first moment condition \eqref{eqn:q-bellman-cons-2}, MA and CMIA, we have $\E{ \eps_{1,i,t}^\pi \mid \calF_{i,t} } = 0$.
        Thus
        \begin{equation*}
            \E{ \bzeta_1^{(k)} } =  (N^{(k)} T)^{-1} \sum_{i\in\calG^{(k)}} \sum_{t=1}^{T} \E{ \E{ Z_{i,t} \eps_{1,i,t}^\pi \mid \calF_{i,t}} } = 0,
        \end{equation*}

        Also by the Bellman first moment condition \eqref{eqn:q-bellman-cons-2}, MA and CMIA, we have, for any $0 \le t_1 < t_2 \le T-1$,
        \begin{equation*}
            \E{\eps_{i,t_1}\eps_{i,t_2}Z_{i,t_1}Z_{i,t_2}^\top} = \E{ \eps_{i,t_1}Z_{i,t_1}Z_{i,t_2}^\top \E{\eps_{i,t_2} \mid \calF_{t_2}}} = 0,
        \end{equation*}
        since $Z_{i,t}$ is a function of $X_{i,t}$ and $A_{i,t}$ only.
        By the independence assumption, we have, for any $0 \le t_1 \le t_2 \le T-1$ and $i_1 < i_2 \in \calG^{(k)}$, that $\E{\eps_{i_1,t_1}\eps_{i_2,t_2}Z_{i_1,t_1}Z_{i_2,t_2}^\top} = 0$.
        Applying Lemma \ref{thm:noise} (i), we have
        \begin{equation*}
            \begin{aligned}
                \E{  \norm{\bzeta_1^{(k)}}_2^2 }
                & = (N^{(k)} T)^{-2} \sum_{i=\calG^{(k)}} \sum_{t=1}^{T} \E{ (\eps_{1,i,t}^\pi)^2 Z_{i,t}^\top Z_{i,t} }
                \le (c_1 + 2 c_2)^2 (N^{(k)} T)^{-1} \E{ Z_{i,t}^\top Z_{i,t} }  \\
                & \le (N^{(k)} T)^{-1} (c_1 + 2 c_2)^2 \underset{\bx\in\calX}{\sup}\;\norm{\bphi(\bx)}_2^2 \\
                & \le c_3 J (N^{(k)} T)^{-1},
            \end{aligned}
        \end{equation*}
        for some constant $c_3$.
        By Markov's inequality, we obtain
        $$\Pr\paran{ \norm{\bzeta_1^{(k)}}_2 > \delta  } \le  \delta^{-2} \E{ \norm{\bzeta_1^{(k)}}_2^2 } \le \delta^{-2} c_3 J (N^{(k)} T)^{-1}.$$
        The result follows.

        \item
        By Lemma \ref{thm:noise} and the Cauchy-Schwarz inequality, we have, for any $\bnu\in\RR^{JM}$,
        \begin{equation*}
            \begin{aligned}
                \abs{\bnu^\top  \bzeta_2^{(k)} \bzeta_2^{(k)\;\top} \bnu}
                & = \abs{\bnu^\top  \bzeta_2^{(k)}}^2
                \le \paran{ (N^{(k)} T)^{-1}  \sum_{i\in\calG^{(k)}} \sum_{t=1}^{T} \abs{ \bnu^\top Z_{i,t} } \abs{ \eps_{2,i,t}^\pi } }^2  \\
                & \le \underset{i \in \calG^{(k)}, 1\le t \le T}{\sup}\; \abs{\eps_{2,i,t}^\pi}^2
                \cdot \paran{ (N^{(k)} T)^{-1}  \sum_{i\in\calG^{(k)}} \sum_{t=1}^{T} \abs{ \bnu^\top Z_{i,t} } }^2  \\
                & \le (c_4 J^{-\kappa/p})^2  \paran{ (N^{(k)} T)^{-1}  \sum_{i\in\calG^{(k)}} \sum_{t=1}^{T} \abs{ \bnu^\top Z_{i,t} Z_{i,t}^\top\bnu} }
            \end{aligned}
        \end{equation*}
        Thus we have,
        \begin{equation*}
            \norm{\bzeta_2^{(k)}}_2 = \underset{\norm{\bnu}=1}{\max} \paran{ \abs{\bnu^\top  \bzeta_2^{(k)} \bzeta_2^{(k)\;\top} \bnu} }^{1/2}
            \le c_4 J^{-\kappa/p} \cdot \lam_{\max}^{1/2}\paran{(N^{(k)} T)^{-1}  \sum_{i\in\calG^{(k)}} \sum_{t=1}^{T} Z_{i,t} Z_{i,t}^\top}
        \end{equation*}
        Applying Lemma \ref{thm:Z-U-eig}, we have,
        \begin{equation*}
            \norm{\bzeta_2^{(k)}}_2 = \Op{ J^{-\kappa/p}  }
        \end{equation*}

        \item By a union bound and the Bernstein inequality for a.s. bounded $\eps_{1,i,t}^\pi$ \eqref{eqn:eps-it-max}, we have
        \begin{equation*}
            \begin{aligned}
                \Pr\paran{\norm{ \bzeta_1^{(k)} }_\infty > \delta}
                & \le \sum_{1\le j\le J, 1\le m \le M} \Pr\paran{ \abs{\phi_j(X_{i,t}) \bone(A_{i,t}=m) (N^{(k)}T)^{-1}\sum_{i\in\calG^{(k)}} \sum_{t=1}^{T} \eps_{1,i,t}^\pi} > \delta} \\
                & \le  JM \Pr\paran{ \underset{\bx\in\calX}{\sup}\;\abs{\phi_j(\bx)} \cdot \abs{ (N^{(k)}T)^{-1}\sum_{i\in\calG^{(k)}} \sum_{t=1}^{T} \eps_{1,i,t}^\pi}  > \delta} \\
                & \le 2JM \exp\paran{ - c^{-2} N^{(k)} T \delta^2 }.
                % & \le 2JM \exp\paran{ - \frac{\delta^2}{ c_1^2 (N^{(k)} T) } }.
            \end{aligned}
        \end{equation*}
        Setting $\delta =  c \sqrt{ 2 \log(N^{(k)}T) / (N^{(k)}T)}$, we have
        \begin{equation*}
            \Pr\paran{ \norm{\bzeta_1^{(k)}}_\infty > c \sqrt{2 \log(N^{(k)}T) / (N^{(k)}T)}}
            \le 2 J M (N^{(k)} T)^{- 2}.
        \end{equation*}

        \item Under the condition that $J^{\kappa / p} \gg \sqrt{N^{(k)}T}$ and $J \ll \sqrt{N^{(k)}T} / \log(N^{(k)}T)$, we have,
        \begin{equation*}
            \norm{\bzeta_2^{(k)}}_{\infty}
            \le \underset{\bx\in\calX}{\sup}\;\abs{\phi_j(\bx)} \cdot  \abs{ (N^{(k)}T)^{-1} \sum_{i\in\calG^{(k)}} \sum_{t=1}^{T} \eps_{2,i,t}^\pi}
            \le c_1 J^{-\kappa/p} \ll c_1 \paran{N^{(k)} T}^{-1/2} \, , \quad\text{almost surely}.
        \end{equation*}
    \end{enumerate}
\end{proof}

\subsubsection{ Properties of variance estimator}

Define
\begin{align*}
    \tilde\bOmega_k
    & =  (N^{(k)} T)^{-1} \sum_{i\in\calG^{(k)}} \sum_{t\in[T]} Z_{i,t}Z_{i,t}^\top (\eps_{1,i,t}^\pi + \eps_{2,i,t}^\pi)^2 \\
    \bOmega_k
    & = (N^{(k)} T)^{-1} \E{ \sum_{i\in\calG^{(k)}} \sum_{t\in[T]} Z_{i,t}Z_{i,t}^\top (\eps_{1,i,t}^\pi + \eps_{2,i,t}^\pi)^2 } \\
    \tilde\bH_k
    & =  (N^{(k)} T)^{-1} \sum_{i\in\calG^{(k)}} \sum_{t\in[T]} Z_{i,t}Z_{i,t}^\top (\eps_{1,i,t}^\pi)^2 \\
    \bH_k
    & = \E{  (N^{(k)} T)^{-1} \sum_{i\in\calG^{(k)}} \sum_{t\in[T]} Z_{i,t}Z_{i,t}^\top (\eps_{1,i,t}^\pi)^2 }.
\end{align*}
For any $\bx \in \calX$, $a \in \calA$, define
\begin{equation*}
    w^\pi \paran{\bx, a} = \E{ (\eps_{1,i,t}^\pi)^2 \bigg| X_{i,t}=\bx, A_{i,t}=a }.
    % = \E{ \paran{R_{i,t} + \gamma \sum_{a\in\calA} \pi(a | X_{i, t+1}) Q(\pi, X_{i,t+1, a}) - Q(\pi, X_{i,t}, A_{i,t})}^2 \bigg| X_{i,t}=\bx, A_{i,t}=a }.
\end{equation*}
Then $\bH_k
=  \E{ T^{-1} \sum_{t\in[T]} w^\pi \paran{X_{i,t}, A_{i,t}}  Z_{i,t} Z_{i,t}^\top}$.

\begin{lemma}   \label{thm:cov-approx}
    Under the conditions in Theorem \ref{thm:subhomo-distn}, we have
    \begin{enumerate}[label=(\roman*)]
        \item $\lam_{\min}\paran{\bH_k} \ge 3^{-1} c_0^{-1} c$
        \item $\lam_{\max}\paran{\bH_k} = \bigO{1}$ or $\norm{\bH_k}_2 = \bigO{1}$
        \item $\norm{\tilde \bOmega^{\pi, (k)}- \tilde\bH_k}_2 = \op{1}$
        \item $\norm{\tilde\bH_k - \bH_k}_2 = \op{1}$
    \end{enumerate}
\end{lemma}
\begin{proof}
    \begin{enumerate}[label=(\roman*)]
        \item By Lemma \ref{thm:Sigma-k} (i) and the fact that $\inf\, w^\pi \paran{\bx, a} \ge c_0^{-1}$, we have
        \begin{equation*}
            \lam_{\min} \paran{ \bH_k }
            \ge c_0^{-1}  \lam_{\min}\paran{(N^{(k)} T)^{-1}  \sum_{i\in\calG^{(k)}} \sum_{t\in[T]} \E{Z_{i,t}Z_{i,t}^\top} }
            \ge 3^{-1} c_0^{-1} c.
        \end{equation*}
        \item By Lemma \ref{thm:noise} (i) and \ref{thm:Sigma-k} (i), we have
        \begin{equation*}
            \lam_{\max}\paran{\bH_k}
            \le (c_1 + 2 c_2)^2  \lam_{\max}\paran{(N^{(k)} T)^{-1}  \sum_{i\in\calG^{(k)}} \sum_{t\in[T]} \E{Z_{i,t}Z_{i,t}^\top} }
            = \bigO{1}.
        \end{equation*}
        \item Applying Lemma \ref{thm:noise-approx} and \ref{thm:Sigma-k}, we have
        \begin{align*}
            \norm{\tilde \bOmega^{\pi, (k)}- \tilde\bH_k}_2
            & = (N^{(k)} T)^{-1} \norm{ \sum_{i\in\calG^{(k)}} \sum_{t\in[T]} Z_{i,t}^\top Z_{i,t} \paran{(\eps_{1,i,t}^\pi + \tilde \eps_{2,i,t}^\pi)^2 - (\eps_{1,i,t}^\pi)^2}}_2  \\
            & \le  \underset{i, t}{\sup} \abs{(\eps_{1,i,t}^\pi + \tilde \eps_{2,i,t}^\pi)^2 - (\eps_{1,i,t}^\pi)^2} \cdot \norm{ (N^{(k)} T)^{-1}  \sum_{i\in\calG^{(k)}} \sum_{t\in[T]} Z_{i,t}^\top Z_{i,t}}_2 \\
            & = \op{1}.
        \end{align*}
        \item The proof uses similar arguments in bounding $\norm{\tilde \bSigma_k - \E{\tilde\bSigma^{\pi,(k)}}}_2$ and is omitted here.
    \end{enumerate}
\end{proof}

\begin{lemma}
    Under the conditions in Theorem \ref{thm:subhomo-distn}, we have
    \begin{align*}
        \norm{(\tilde\bSigma^{\pi,(k)})^{-1}\tilde \bOmega^{\pi, (k)}(\tilde\bSigma^{\pi,(k)})^{\top\;-1} - (\bSigma^{\pi,(k)})^{-1} \bH_k (\bSigma_k^\top)^{-1}}_2 = \op{1}.
    \end{align*}
\end{lemma}
\begin{proof}
    We have
    \begin{align*}
        & \norm{(\tilde\bSigma^{\pi,(k)})^{-1}\tilde \bOmega^{\pi, (k)}(\tilde\bSigma^{\pi,(k)})^{\top\;-1} - \E{\tilde\bSigma^{\pi,(k)}}^{-1} \bH_k (\E{\tilde\bSigma^{\pi,(k)}}^\top)^{-1}}_2  \\
        & \le
        \norm{(\tilde\bSigma^{\pi,(k)})^{-1} \tilde \bOmega^{\pi, (k)}(\tilde\bSigma^{\pi,(k)})^{\top\;-1} - (\tilde\bSigma^{\pi,(k)})^{-1} \bH_k (\tilde\bSigma^{\pi,(k)})^{\top\;-1}}_2
        + \norm{(\tilde\bSigma^{\pi,(k)})^{-1} \bH_k (\tilde\bSigma^{\pi,(k)})^{\top\;-1} - \E{\tilde\bSigma^{\pi,(k)}}^{-1} \bH_k (\E{\tilde\bSigma^{\pi,(k)}}^\top)^{-1}}_2 \\
        &\le  \norm{(\tilde\bSigma^{\pi,(k)})^{-1}}_2^2 \norm{\tilde \bOmega^{\pi, (k)}- \bH_k}_2  \\
        & ~~~ + \norm{ \paran{(\tilde\bSigma^{\pi,(k)})^{-1} - \E{\tilde\bSigma^{\pi,(k)}}^{-1}} \bH_k (\tilde\bSigma^{\pi,(k)})^{\top\;-1}}
        + \norm{ \E{\tilde\bSigma^{\pi,(k)}}^{-1} \bH_k \paran{(\tilde\bSigma^{\pi,(k)})^{-1} - \E{\tilde\bSigma^{\pi,(k)}}^{-1}}^\top }_2\\
        & = \op{1},
    \end{align*}
    where we use the results in Lemma \ref{thm:Sigma-k} and \ref{thm:cov-approx}.
\end{proof}

%!TEX root = 0-main.tex

\section{Description of MIMIC-III Dataset}  \label{appen:mimic-iii}

\subsection{Intensive Care Unit Data}

The data we use is the Medical Information Mart for Intensive Care version III (MIMIC-III) Database \citep{johnson2016mimic}, which is a freely available source of de-identified critical care data from 53,423 adult admissions and 7,870 neonates from 2001 -- 2012 in six ICUs at a Boston teaching hospital.
The database contain high-resolution patient data, including demographics, time-stamped measurements from bedside monitoring of vital signs, laboratory tests, illness severity scores, medications and procedures, fluid intakes and outputs, clinician notes and diagnostic coding.

We extract a cohort of sepsis patients, following the same data processing procedure as in \cite{komorowski2018artificial}.
Specifically, the adult patients included in the analysis satisfy the international consensus sepsis-3 criterion.
The data includes 17,083 unique ICU admissions from five separate ICUs in one tertiary teaching hospital.
Patient demographics and clinical characteristics are shown in Table 1 and Supplementary Table 1 of \cite{komorowski2018artificial}.

Each patient in the cohort is characterized by a set of 47 variables, including demographics, Elixhauser premorbid status, vital signs, and laboratory values.
Demographic information includes age, gender, weight. Vital signs include heart rate, systolic/diastolic blood pressure, respiratory rate et al.
Laboratory values include glucose, total bilirubin, (partial) thromboplastin time et al.
Patients' data were coded as multidimensional discrete time series with 4-hour time steps.
The actions of interests are the total volume of intravenous (IV) fluids and maximum dose of vasopressors administrated over each 4-hour period.

All features were checked for outliers and errors using a frequency histogram method and uni-variate statistical approaches (Tukey's method).
Errors and missing values are corrected when possible.
For example, conversion of temperature from Fahrenheit to Celsius degrees and capping variables to clinically plausible values.

In the final processed data set, we have \num{17621} unique ICU admissions, corresponding to unique trajectories fed into our algorithms.

\subsubsection{Irregular Observational Time Series Data}

For each ICU admission, we code patient's data as multivariate discrete time series with a four hour time step.
Each trajectory covers from up to 24h preceding until 48h following the estimated onset of sepsis, in order to capture the early phase of its management, including initial resuscitation.
The medical treatments of interest are the total volume of intravenous fluids and maximum dose of vasopressors administered over each four hour period.
We use a time-limited parameter specific sample-and-hold approach to address the problem of missing or irregularly sampled data.
The remaining missing data were interpolated in MIMIC-III using multivariate nearest-neighbor imputation.
After processing, we have in total \num{278598} sampled data points for the entire sepsis cohort.

\subsubsection{State and Action Space Characterization}

The state $\bX_{i,t}$ is a 47-dimensional feature vector including fixed demographic information (age, weight, gender, admit type, ethnicity et al), vitals signs (heart rate, systolic/diastolic blood pressure, respiratory rate et al), and laboratory values (glucose, Creatinine, total bilirubin, partial thromboplastin time, $paO_2$, $paCO_2$ et al.).

For action space, we discretize two variables into three actions respectively according to Table in \cite{komorowski2018artificial}.
The combination of the two drugs makes $3 \times 3 = 9$ possible actions in total.
The action $A_t$ is a two-dimensional vector, of which the first entry $a_t[0]$ specifies the dosages of IV fluids and the second $a_t[1]$ indicates the dosages of IV fluids and vasopressors, to be administrated over the next 4h interval.

\subsection{Reward design}

The reward signal is important and need crafted carefully in real applications.  \cite{komorowski2018artificial} uses hospital mortality or 90-day mortality as the sole defining factor for the penalty and reward.
Specifically, when a patient survived, a positive reward was released at the end of each patient's trajectory (a reward of \num{+ 100}); while a negative reward (a penalty of \num{-100}) was issued if the patient died.
However, this reward design is sparse and provide little information at each step. Also, mortality may correlated with respect to the health statues of a patient. So it is reasonable to associate reward to the health measurement of a patient after an action is taken.

In this application, we build our reward signal based on physiological stability.
Specifically, in our design, physiological stability is measured by vitals and laboratory values $v_t$ with desired ranges $[ v_{\min}, v_{\max} ]$.
Important variables related to sepsis include heart rate (HR), systolic blood pressure (SysBP), mean blood pressure (MeanBP), diastolic blood pressure (DiaBP), respiratory rate (RR), peripheral capillary oxygen saturation (SpO2), arterial lactate, creatinine, total bilirubin, glucose, white blood cell count, platelets count, (partial) thromboplastin time (PTT), and International Normalized Ratio (INR).
We encode a penalty for exceeding desired ranges at each time step by a truncated Sigmoid function, as well as a penalty for sharp changes in consecutive measurements.

%\begin{equation}
%    r_{t+1} = \underset{v}{\sum} C_1 \left[ \frac{1}{1 + e^{-(v_t - v_{min})}} -  \frac{1}{1 + e^{-(v_t - v_{max})}} + 0.5  \right] - C_2 \left[ \max\left( 0, \frac{|v_{t+1} - v_t|}{v_t} - 0.2 \right) \right], \nonumber
%\end{equation}

Here, values $v_t$ are the measurements of those vitals $v$ believed to be indicative of physiological stability at time $t$, with desired ranges $[v_{min}, v_{max}]$. The penalty for exceeding these ranges at each time step is given by a truncated sigmoid function. The system also receives negative feedback when consecutive measurements see a sharp change.

\begin{remark}
    There are definitely improvements in shaping the reward space. For example, in medical situation, the definition of the normal range of a variable sometime depends demographic characterization. Also, sharp changes in a favorable direction should be rewarded.
\end{remark}
%\listoffigures
%\listoftables
\end{appendices}

\end{document}